\pdfoutput=1
\documentclass{article}
\usepackage[preprint]{neurips_2026}

\usepackage{amsmath,amsfonts,bm}

\def\eqref#1{equation~\ref{#1}}

\def\1{\bm{1}}

\DeclareMathAlphabet{\mathsfit}{\encodingdefault}{\sfdefault}{m}{sl}
\SetMathAlphabet{\mathsfit}{bold}{\encodingdefault}{\sfdefault}{bx}{n}

\usepackage{url}
\usepackage{amsmath,amssymb,mathtools}
\usepackage{amsthm}
\newtheorem{hypothesis}{Hypothesis}
\usepackage{booktabs}
\usepackage{graphicx}
\usepackage{float}
\usepackage{placeins}
\usepackage{microtype}
\usepackage{enumitem}
\usepackage{xcolor}
\usepackage{amsmath,amssymb}
\usepackage{tikz}
\usetikzlibrary{arrows.meta,positioning,calc,backgrounds,fit}
\usepackage{hyperref}
\hypersetup{hidelinks}
 
\definecolor{latentred}{RGB}{190,30,45}
\definecolor{attnorange}{RGB}{230,145,25}
\definecolor{contrgreen}{RGB}{20,105,70}
\definecolor{panelgray}{RGB}{120,120,120}

\title{Retrieval Capacity of Self-Attention Under Competition}
\author{%
  Timur Mudarisov$^{1}$ \quad
  Mikhail Burtsev$^{2}$ \quad
  Radu State$^{1}$ \\[0.6em]
  \normalfont $^{1}$University of Luxembourg, Luxembourg \\
  \normalfont $^{2}$London Institute for Mathematical Sciences, London, UK
}

\newcommand{\NLL}{\operatorname{NLL}}
\newcommand{\Inv}{\operatorname{Inv}}

\newcommand{\Sel}{\mathcal{S}}
\newcommand{\Shat}{\widehat{\mathcal{S}}}
\newcommand{\Prb}{\mathbb{P}}

\newcommand{\GInv}{\operatorname{GInv}}
\newcommand{\Iset}{\mathcal{I}}

\newcommand{\Tail}{\mathcal{T}}

\theoremstyle{plain}
\newtheorem{proposition}{Proposition}
\theoremstyle{definition}
\newtheorem{definition}{Definition}

\begin{document}
\maketitle

\begin{abstract}
How many tokens from its context does a language model actually use, and what determines that number? We study this question through self-attention. Without retraining, we retain only the tokens with the highest attention weights at each head, layer, and query, keeping their original weights unchanged. By varying the selected set size and measuring the increase in negative log-likelihood (NLL), we estimate the effective attention set size needed to stay within a chosen loss tolerance. Relatively small selected sets can keep NLL close to the full-attention baseline, although the required size varies across models. Attention-based selection substantially outperforms random selection. Selected sets exhibit geometric structure, although geometric separation alone does not establish that model loss is preserved. Extending context while evaluating the same prediction targets increases the required set size, while its fraction of context decreases over the tested range. Experiments with a fixed supporting fact show that additional background pushes its tokens down the attention ranking and reduces their attention mass. Renormalizing the retained weights can substantially reduce the required set size, showing that it also depends on how selected representations are combined. Conditional theoretical models explain how competition and attention-mass retention can produce growing set sizes without more distinct information to retrieve. These results provide a way to measure effective attention set size in language models and investigate its dependence on context, competition, and aggregation.
\end{abstract}

\section{Introduction}
\label{sec:intro}

\emph{How many tokens from its context does a language model actually use,
and what determines that number?} Self-attention provides a natural setting
for studying this question because it explicitly scores context tokens and
combines their value vectors. Previous work by
\citet{mudarisov2026geometric} found that sets selected by the largest
attention weights exhibit stronger geometric separation than random sets
of the same size in the space of weighted value vectors. This observation
suggests a structured selection process and motivates examining whether
the selected tokens are sufficient to preserve model performance.

We develop this perspective through a useful-token hypothesis. For each
local context $C=(X,\ell,h,t)$, we posit an unknown useful set whose size
$K(C)$ can vary with the sequence, layer, head, and query position.
Attention provides an imperfect ranking of this set, with useful tokens
assumed to rank above other tokens with relatively few errors on a
reference distribution. Recovering the useful set therefore depends on
both its size and the quality of the ranking. Even when the useful set
remains fixed, ranking errors can require retaining additional tokens.
This framework guides our analysis of observable selected sets while
leaving their relationship to the unknown useful set to be investigated.

We first examine the geometric separation of selected and unselected
tokens under Euclidean and cosine distances, using random sets of the
same size as controls. We then test the selected sets through their
effect on language-modeling loss. At every query, head, and layer, we
retain up to $N$ tokens ranked by attention weight or contribution
magnitude, keeping their original attention weights unchanged. Selection
is recomputed throughout the intervened forward pass without retraining.
By varying $N$ and measuring the increase in average negative
log-likelihood (NLL), we estimate the \emph{effective attention set size}
$N_\tau^*$ needed to remain within a chosen loss tolerance $\tau$.
The same size limit applies throughout the model, while each attention
operation selects its own tokens. Comparing geometric separation with
NLL degradation then tests whether the observed structure indicates
functional sufficiency.

We study what affects the required set size through two
complementary experiments. In language modeling, we extend the available
context while evaluating the same prediction targets. The required set
size increases with context length, although its fraction of the context
decreases over the tested range. Additional natural context also improves the full model's predictions, so this experiment can reflect changes in
both useful information and competition. We therefore complement it with
BABILong experiments that add background text around a fixed annotated
supporting fact. These experiments track changes in support ranking,
attention mass, and recall alongside the selected-set size needed to
preserve answer loss.

The treatment of retained attention weights provides a further explanation of the
measured set sizes. Renormalizing these weights can substantially reduce
the number of tokens required, showing that functional sufficiency also
depends on how selected representations are combined. We develop
conditional theoretical models of ranking competition and attention-mass
retention to explain how required set sizes can grow with context,
including settings with a fixed required set or identical value vectors.
These results connect the useful-token framework to the functional
measurements while accounting for the effect of the intervention on the
attention output.

Our contributions are:\\
1. \textbf{A useful-token framework for attention selection.}
We distinguish the unknown useful set from observable selected sets and
derive recovery bounds connecting ranking errors to the number of tokens
that must be retained.\\
2. \textbf{Geometric and functional analysis of selected sets.}
We compare geometric separation with loss preservation and estimate
effective attention set sizes across nine checkpoints, using attention,
contribution, and random selection.\\
3. \textbf{Evidence for context-dependent set sizes.}
We measure the effect of context extension with fixed prediction targets
and examine competition through background addition at fixed annotated
support.\\
4. \textbf{Aggregation controls and conditional explanations.}
We compare deletion with renormalization and develop theoretical models
showing how ranking competition and preservation of attention mass can
make required set sizes grow with context.

\section{Attention selection and effective set size}
\label{sec:setup}
\label{sec:scores}
\label{sec:latent-hypothesis}
\label{sec:binary-relevance}
\label{sec:noisy-ranking}
\label{sec:topn}
\label{sec:geometry}
\label{sec:truncation}
\label{sec:functional-metrics}
\label{sec:width}
\label{sec:intervention-dependence}

For a sequence $X$, consider a causal attention operation at layer $\ell$,
head $h$, and query position $t$. The visible tokens have indices
$\Iset_t=\{0,\ldots,t\}$, and the attention weights and output are
\begin{equation}
\alpha_i=\frac{\exp(q_t^\top k_i/\sqrt{d_k})}
{\sum_{j\in\Iset_t}\exp(q_t^\top k_j/\sqrt{d_k})},
\qquad
z_t=\sum_{i\in\Iset_t}\alpha_i v_i.
\label{eq:index-set}
\end{equation}
Writing $y_i=\alpha_i v_i$ for the contribution of token $i$, we compare
two observable scores:
\begin{align}
r_i^{\mathrm{attn}}=\alpha_i, \quad 
r_i^{\mathrm{contr}}=\|y_i\|_2=\alpha_i\|v_i\|_2.
\label{eq:contr-score}
\end{align}
Attention ranking uses the assigned probabilities, while contribution
ranking also accounts for the magnitudes of the value vectors. Both
describe selection within a head before the output projection.

To relate these rankings to retrieval function of attention, we posit an unknown useful set for
each context $C=(X,\ell,h,t)$. Let $Z_i(C)\in\{0,1\}$ indicate whether
token $i$ belongs to this set, and define
\begin{equation}
\Sel^*(C)=\{i\in\Iset_t:Z_i(C)=1\},
\qquad
K(C)=|\Sel^*(C)|=\sum_{i\in\Iset_t}Z_i(C).
\label{eq:latent-set}
\end{equation}
These labels represent hypothesized relevance to downstream prediction.
Their values are unobserved, and membership in the useful set can vary
across sequences, layers, heads, and queries. We formulate attention as
an imperfect ranking of this set. For a ranking $r$, an inversion occurs
when a token outside the useful set has at least as high a score as a
useful token. Let $\Inv_r(C)$ count these pairs.

\begin{hypothesis}[Useful-token ranking on a reference distribution]
\label{hyp:latent-ranking}
Fix a module $(\ell,h)$ and a reference distribution $\mathcal D_0$ over
$(X,t)$, writing $C=(X,\ell,h,t)$. For at least one
$r\in\{\mathrm{attn},\mathrm{contr}\}$, there are
$\varepsilon,\delta\in[0,1)$ such that
\begin{equation}
\Prb_{(X,t)\sim\mathcal D_0}
\bigl[\Inv_r(C)\le\varepsilon K(C)^2\bigr]\ge1-\delta,
\label{eq:main-hypothesis}
\end{equation}
where $1\le K(C)<|\Iset_t|$ and $K(C)$ is nondegenerate under
$\mathcal D_0$.
\end{hypothesis}

The hypothesis describes ranking quality within a reference distribution (see Fig.~\ref{fig:latent-useful}).
Increasing the amount of background can change this quality, so the
assumption is not imposed uniformly across context lengths. Because the
useful labels are unobserved, our experiments examine selected sets and
their consequences without directly testing the hypothesis.

For an integer $N\ge1$, define the observable selected set
\begin{equation}
\Shat_N^{(r)}(C)=
\operatorname{TopN}\{r_i(C):i\in\Iset_t\},
\qquad r\in\{\mathrm{attn},\mathrm{contr}\}.
\label{eq:topn-candidate-set}
\end{equation}
Here Top-$N$ returns the indices of the
$\min\{N,|\Iset_t|\}$ highest-scoring tokens, breaking ties by increasing
index. The ranking determines which tokens enter the set, and $N$
determines how many are retained. Recovering the useful set depends on
both its size and the ranking errors. In particular, additional tokens
may need to be retained when they appear above useful tokens in the
ranking. Appendix~\ref{app:topk-recovery} makes this relationship precise
through recovery bounds at $N=K(C)$ and at larger selected-set sizes.

\begin{figure}[t]
\centering
\includegraphics[width=0.99\textwidth]{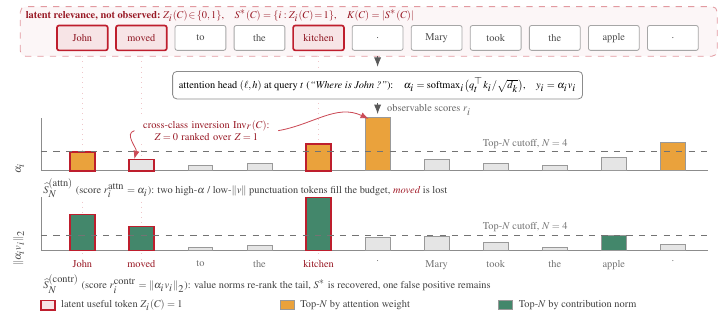}
\caption{
\textbf{Useful-token hypothesis.}
Each attention operation has an unknown useful set whose size $K(C)$
depends on context. Attention and contribution scores give rankings from
which sets of different sizes are selected. Ranking inversions occur when
tokens outside the useful set score at least as highly as useful tokens,
so recovering the useful set may require retaining more than $K(C)$ tokens.
}
\label{fig:latent-useful}
\end{figure}

Following \citet{mudarisov2026geometric}, we first examine whether selected
tokens form a geometrically distinct set in the space of their weighted
value vectors. For $1\le N<|\Iset_t|$, define
$s_N^{(r)}=\sum_{i\in\Shat_N^{(r)}}y_i$, the contribution of the selected
tokens to the head output. We retain the magnitude of this sum and compare
Euclidean distance,
$d_{\mathrm E}(x,y)=\|x-y\|_2$, and cosine distance,
$d_{\mathrm C}(x,y)=1-\langle x,y\rangle/
(\|x\|_2\|y\|_2+\eta)$, using $\eta=10^{-12}$.
Euclidean distance reflects both magnitude and direction, while cosine
distance provides a complementary directional comparison. Write
$D_i=d(y_i,s_N^{(r)})$ and suppress the ranking and distance indices below.

\begin{definition}[Geometric separability metrics]
\label{def:geo-metrics}
Let $S=\Shat_N^{(r)}$,
$\rho_{\min}=\min_{j\in\Iset_t\setminus S}D_j$, and
$\rho_{\max}=\max_{i\in S}D_i$. Define
\begin{equation}
P_N=
\frac{N}
{N+|\{j\in\Iset_t\setminus S:D_j\le\rho_{\max}\}|},
\qquad
R_N=
\frac{|\{i\in S:D_i<\rho_{\min}\}|}{N}.
\label{eq:geo-metrics}
\end{equation}
The extremal separability score is
$F_N=2P_NR_N/(P_N+R_N)$, with $P_N>0$ on this domain.
\end{definition}

Geometric precision measures contamination by unselected tokens within
the radius containing the selected set. Geometric recall measures the
fraction of selected tokens closer to the aggregate than every unselected
token. Their harmonic summary combines two different radii, so its
interpretation differs from a classification F-score evaluated at a common
threshold. Random controls use the same construction around their own
selected aggregates. At $N=1$, Euclidean $F_1=1$ in the absence of duplicate
vectors, including for random selection. Stabilized cosine distance need
not share this exact boundary. Appendix~\ref{app:geometry-details}
describes these boundary and scale effects.

To determine whether a selected set is sufficient for prediction, we
restrict attention to that set and measure the resulting change in model
loss. With mask 
$m_i^{(r,N)}=\mathbf 1\{i\in\Shat_N^{(r)}\}$, the primary intervention
leaves retained weights unchanged and sets the others to zero:
\begin{equation}
\widetilde\alpha_i^{(r,N)}=\alpha_i m_i^{(r,N)},
\qquad
\widetilde z_t^{(r,N)}
=\sum_{i\in\Iset_t}\widetilde\alpha_i^{(r,N)}v_i.
\label{eq:no-renorm}
\end{equation}
We apply the same size limit $N$ throughout the model, recomputing the
selected set from the current activations at every causal query, head,
and layer without retraining. Queries with at most $N$ visible tokens
retain them all. At fixed incoming activations, the local output change
is the discarded sum
$e_t^{(r,N)}=\sum_{i\notin\Shat_N^{(r)}}\alpha_i v_i$.
Appendix~\ref{app:local-bounds} relates its norm to the discarded
contribution norms and attention mass. The effect on downstream
prediction is evaluated through the full intervened model.

For the full-attention model $M$ and intervened model $M_N^{(r)}$, define
relative language-modeling degradation as
\begin{equation}
\Delta\NLL_{\mathrm{rel}}^{(r)}(N)=
\frac{\NLL(M_N^{(r)})-\NLL(M)}{\NLL(M)}.
\label{eq:relative-nll}
\end{equation}
For question answering tasks, we use candidate-normalized answer loss. Given a candidate set
$\mathcal A(x)$, candidate score $s_M(a\mid x)$, and correct answer $y$,
\begin{equation}
\ell_{\mathrm{ans}}(M;x)=-\log
\frac{\exp\!\bigl(s_M(y\mid x)\bigr)}
{\sum_{a\in\mathcal A(x)}
\exp\!\bigl(s_M(a\mid x)\bigr)}.
\label{eq:answer-loss}
\end{equation}
In the BABILong experiments, $s_M(a\mid x)$ is the
first-continuation-token logit for candidate $a$ among six answers.
We measure the mean loss increase on matched examples and candidate sets
and report candidate accuracy as a complementary measure.
Appendix~\ref{app:qa-scoring} provides the complete scoring protocol.
Annotated supporting facts supply a partial reference for relevance:
a fact can span several tokens, and other tokens may also contribute to
the computation.

Let $\Delta\mathcal L^{(r)}(N)$ denote the relative NLL increase or the
mean answer-loss increase, according to the task. For a tolerance
$\tau\ge0$, define the \emph{effective attention set size} as
\begin{equation}
N_\tau^{*,(r)}
=\min\{N\ge1:\Delta\mathcal L^{(r)}(N)\le\tau\},
\label{eq:required-width}
\end{equation}
where $N$ ranges over positive integers. This quantity describes the
common selected-set size needed to satisfy an average-loss criterion
across the model. Its relationship to the local useful-set size $K(C)$
depends on ranking errors, the values being combined, and the sensitivity
of subsequent computation.

We additionally evaluate renormalized Top-$N$, which divides each
retained weight by the total retained attention mass. This preserves
their relative weights and restores their sum to one. Comparing the two
interventions tests how the required set size depends on rescaling the
selected sum. The empirical comparison is reported in
Appendix~\ref{app:renorm-details}, with the corresponding local output
identities in Appendix~\ref{app:local-bounds}.

\section{Experiments}
\label{sec:experiments}
\label{sec:experimental-setup}

We evaluate nine decoder-only checkpoints:
Qwen-2.5-1.5B/7B~\citep{yang2024qwen25},
Gemma-7B~\citep{mesnard2024gemma},
Gemma-2-9B~\citep{riviere2024gemma2},
Llama-2-7B~\citep{touvron2023llama2},
Llama-3-8B~\citep{grattafiori2024llama3},
Llama-3.2-1B~\citep{meta2024llama32},
Mistral-7B-v0.3~\citep{jiang2023mistral}, and
Mistral-Small-24B-Base-2501~\citep{mistralai2025mistralsmall24b}.
The 24B checkpoint uses 8-bit weights, so comparisons involving this
model include differences in both architecture and numerical precision.

Language-model evaluation uses 50 documents per model from each of
OpenWebText~\citep{gokaslan2019openwebtext} and
WikiText-103~\citep{merity2017pointer}. Geometric measurements cover
Qwen-2.5-7B, Gemma-7B, Llama-3-8B, and Mistral-7B-v0.3 at the final query
across layers and heads. Functional interventions apply at every causal
query throughout the model. We also evaluate controlled retrieval on
BABILong~\citep{kuratov2024babilong} \texttt{qa1}, holding one annotated
supporting fact fixed while varying background text across
0K, 1K, 2K, and 4K conditions. Each condition contains 100 QA examples,
with 86 examples per background passing the tokenizer-span validation
used for support measurements. Appendix~\ref{app:protocol} provides the
detailed protocols.

\subsection{Geometric structure and effective attention set size}
\label{sec:exp-geometry}
\label{sec:geometry-result}
\label{sec:exp-lm-width}
\label{sec:lm-results}

We begin by comparing sets selected through attention and contribution
rankings with random sets of the same size. Figure~\ref{fig:geometry-main}
shows stronger geometric separation for score-based selection in all
four models.There is a high advantage in the mean Euclidean $F_N$ of
attention selection over random selection over models on OpenWebText. 
The corresponding cosine advantages
are smaller (see Appendix~\ref{app:additional-width-results}). Contribution ranking produces the same
qualitative pattern. The larger Euclidean differences indicate that
the observed separation includes a substantial magnitude-related
component, alongside the directional structure captured by cosine
distance.

\begin{figure}[t]
\centering
\includegraphics[width=0.85\linewidth]
{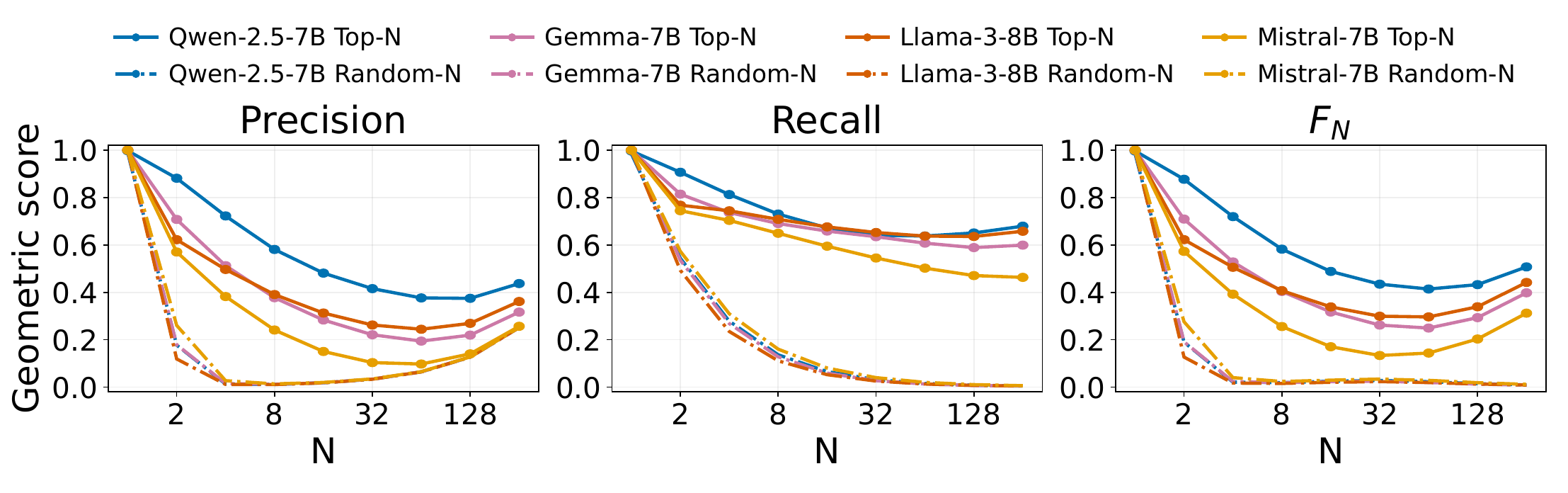}
\caption{
\textbf{Attention-selected sets exhibit stronger Euclidean separation
than random sets across all four models.}
Geometric precision (left), recall (middle), and extremal $F_N$ (right)
on OpenWebText. Solid curves show selection by attention weight,
and dash-dotted curves show random sets of the same size, each
evaluated relative to its own aggregate. Measurements use the final
query of 1024-token contexts and are averaged over heads, layers,
and documents. 
}
\label{fig:geometry-main}
\end{figure}

\begin{figure*}[b]
\centering
\includegraphics[width=0.9\textwidth]
{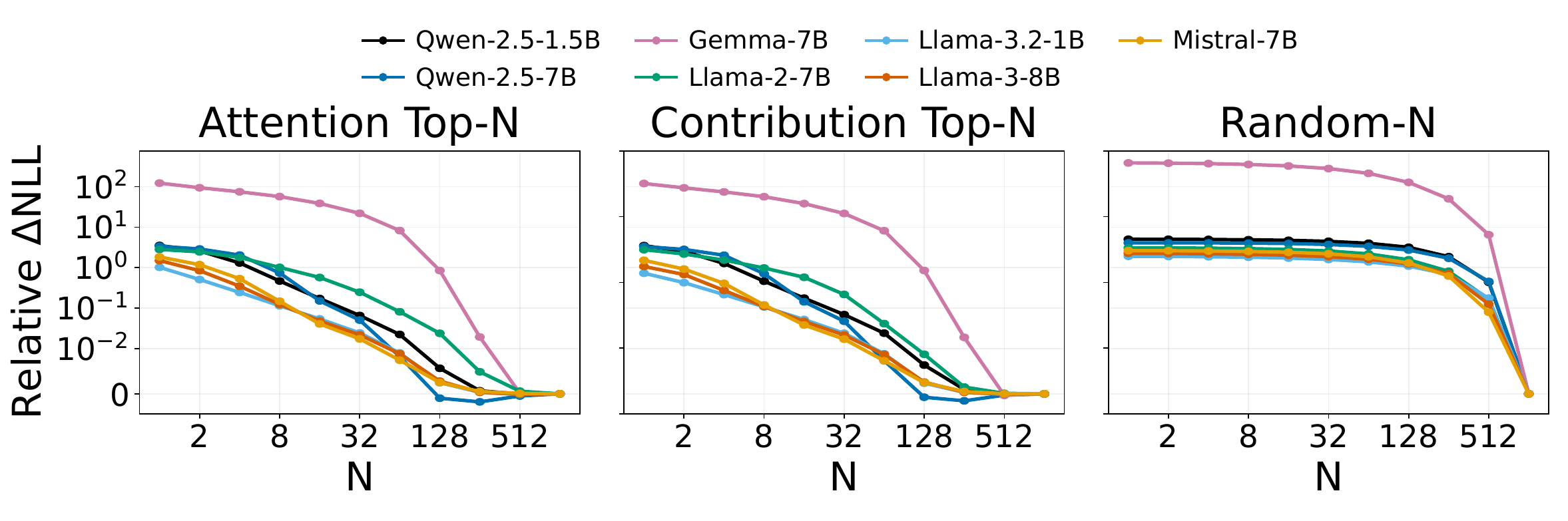}
\caption{
\textbf{Attention and contribution selection produce lower NLL
degradation than random selection at the same attention set sizes.}
Relative increase in mean NLL over the full model on OpenWebText
with 1k-token contexts. Panels compare selection by attention weight
(left), contribution magnitude (middle), and random selection (right).
For a chosen loss tolerance, the crossing of a curve with that
degradation level gives an approximate effective attention set size $N^*$.
}
\label{fig:nll-main}
\end{figure*}

These geometric differences motivate testing whether the selected sets
preserve predictive performance. Figure~\ref{fig:nll-main} shows how
relative NLL degradation changes as more highly scored tokens are
retained with their original attention weights. For a chosen loss
tolerance, the crossing of each curve provides an estimate of the
effective attention set size. We evaluate powers of two and refine
observed crossings to integer sizes, without exhaustively testing
every smaller integer. Detailed estimates for OpenWebText and
WikiText-103 are reported separately in
Appendix~\ref{app:nll-capacity-by-dataset}.

NLL degradation decreases rapidly as more highly scored tokens are
retained, while random selection produces substantially greater
degradation at the same set sizes. This advantage is consistent with
the useful-token hypothesis, suggesting that the rankings concentrate
tokens important for prediction near the top. The required set size
nevertheless varies substantially across models. Contribution ranking
reduces the required size most clearly for Llama-2-7B, showing that
value magnitude can help identify a sufficient set. For most other
models, the two rankings yield similar estimates.

These measurements characterize how many tokens must be retained
under a common selection rule to keep average loss within the chosen
tolerance. Within the hypothesis, differences in this estimate can
reflect both the size of the useful set and how accurately the scores
rank its members. Sensitivity to the discarded contributions also
matters. Localized interventions reveal substantial variation across
layers and heads
(Appendix~\ref{app:localization-extra}), so the common set size does
not describe a typical individual attention operation.

We next relate functional performance to the geometric separability.
Figure~\ref{fig:geometry-nll-scatter} shows positive rank correlations
along the plotted trajectories, so stronger separation often accompanies
greater NLL degradation. 
 Additional contribution-ranking trajectories and
an exploratory prediction analysis appear in
Appendices~\ref{app:geometry-function-scatter} and \ref{app:predictors}.

\begin{figure*}[t]
\centering
\begin{minipage}{0.495\textwidth}
\centering
\includegraphics[width=\linewidth]
{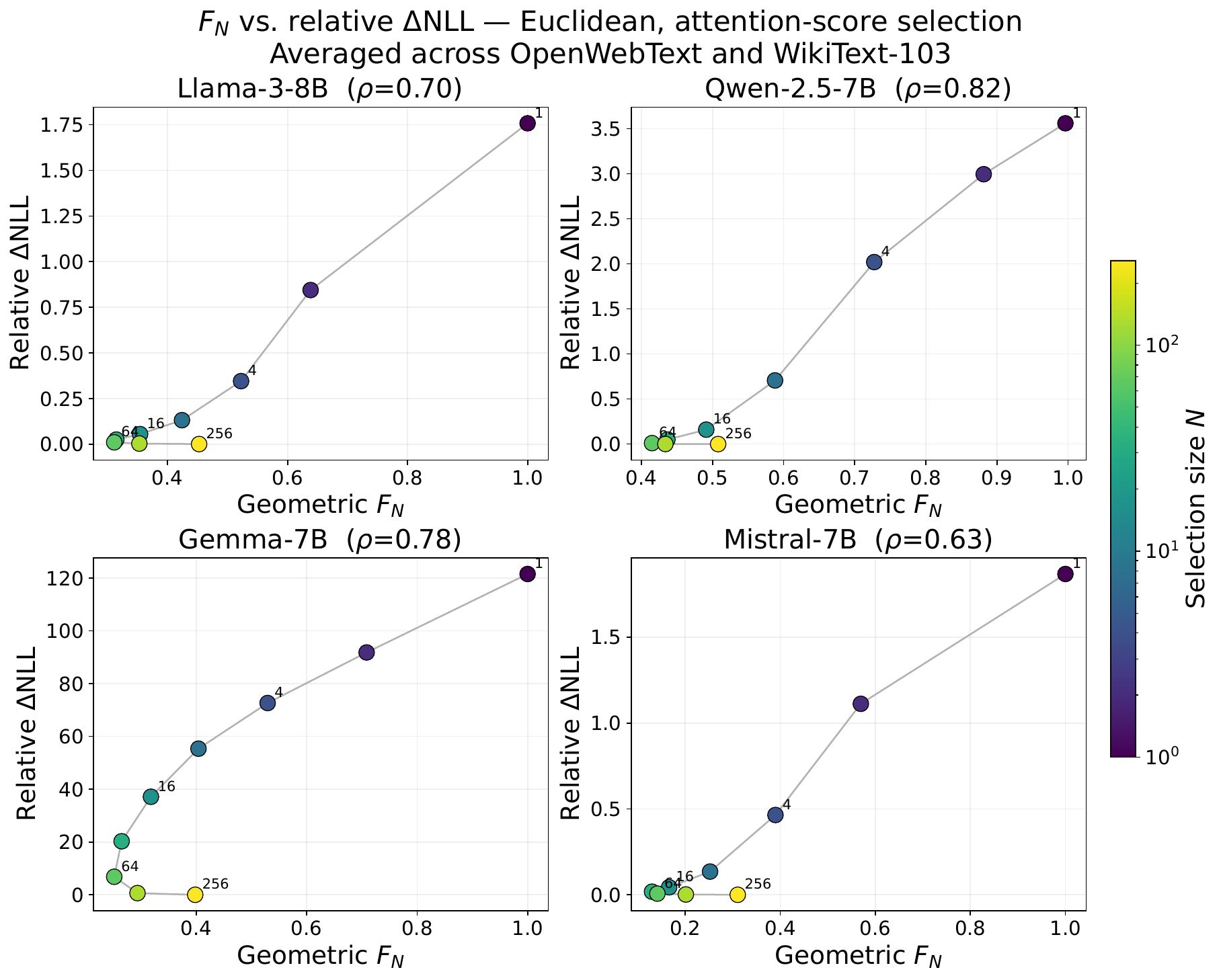}\\[-2pt]
{\small (a) Euclidean geometry}
\end{minipage}\hfill
\begin{minipage}{0.495\textwidth}
\centering
\includegraphics[width=\linewidth]
{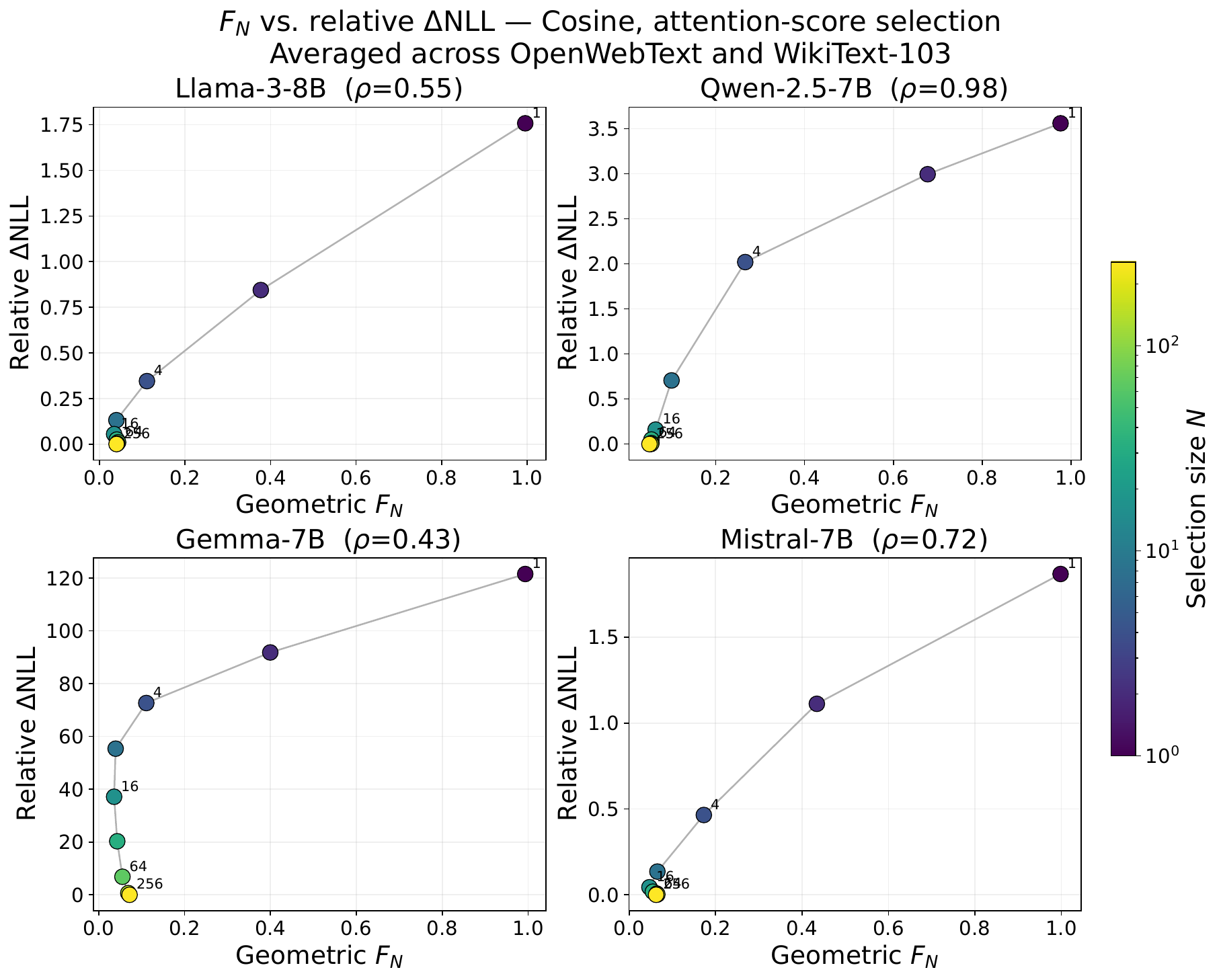}\\[-2pt]
{\small (b) Cosine geometry}
\end{minipage}
\caption{
\textbf{Stronger geometric separation often accompanies greater NLL
degradation.}
Geometric $F_N$ and relative NLL degradation under attention-weight
selection, using (a) Euclidean and (b) cosine geometry.
Each point corresponds to one set size $N$, indicated by color,
with both quantities averaged across OpenWebText and WikiText-103.
Lines connect successive set sizes, and $\rho$ denotes the Spearman
correlation along each trajectory. 
}
\label{fig:geometry-nll-scatter}
\end{figure*}

\subsection{Dependence on context length}
\label{sec:context-length}

To examine whether the required set size remains stable as more context
becomes available, we vary
$L\in\{256,512,1024,2048\}$ using nested suffixes of 50 documents per
model and corpus. At every length, we score the same final 128 target
tokens within each model--corpus pair. The full-model baseline is
recomputed at each length, and selection is applied throughout the
model. 

Figure~\ref{fig:context-width-main}a shows that longer contexts require
larger selected sets to preserve performance on the same prediction
targets, while the selected fraction of context decreases. Under the
useful-token hypothesis, extending context can change both the useful
set and the competition involved in recovering it. Additional context
may contain useful tokens, but it also introduces candidates that can
rank above useful tokens already present. The observed growth in
effective attention set size can therefore arise from changes in
$K(C)$, changes in ranking, and the number of contributions needed
to preserve the attention output.

The improvement in full-model NLL (Fig.~\ref{fig:context-width-main}b) indicates that the added natural
context provides useful predictive information. This makes changes
in the useful set a plausible contributor to the observed growth,
although improved predictions do not establish how its size changes.
The required selected-set size also grows when the allowed absolute
NLL increase is held fixed
(App.~\ref{app:context-length-results}), ruling out the tightening
of the relative-loss criterion as a complete explanation. These
observations motivate the controlled-background experiments below,
which examine competition while holding the annotated supporting
fact fixed. 
Appendix~\ref{app:context-length-results} reports the normalized sizes,
uncertainty estimates, contribution-ranking results, and comparisons
between final-target and all-token scoring.

\begin{figure}[t]
\centering
\begin{minipage}{0.495\linewidth}
\centering
\includegraphics[width=\linewidth]
{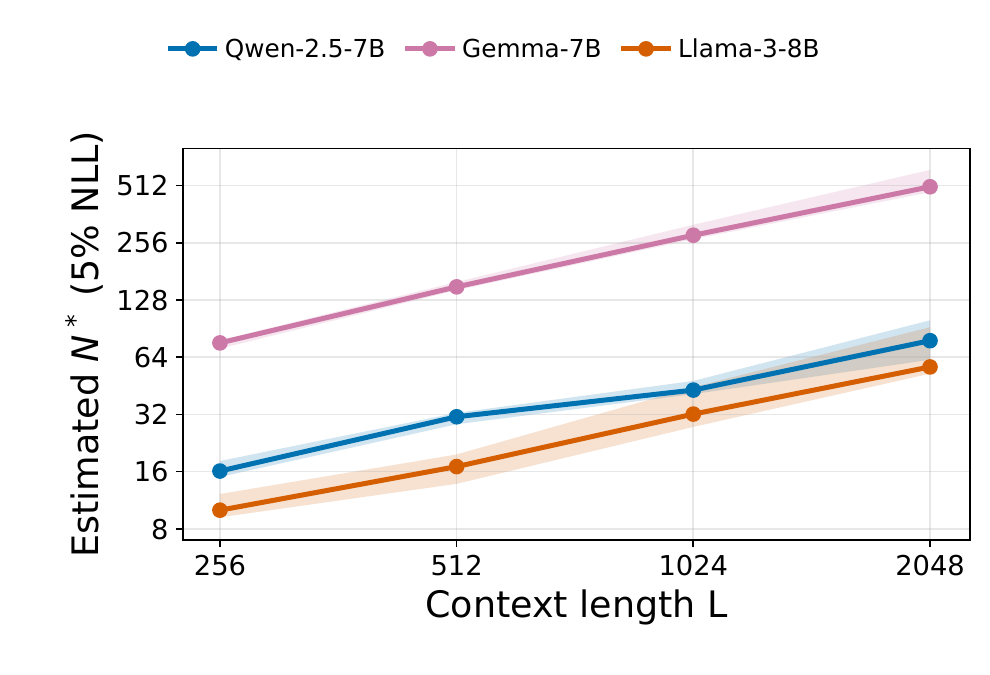}\\[-2pt]
{\small (a) Effective attention set size}
\end{minipage}\hfill
\begin{minipage}{0.495\linewidth}
\centering
\includegraphics[width=\linewidth]
{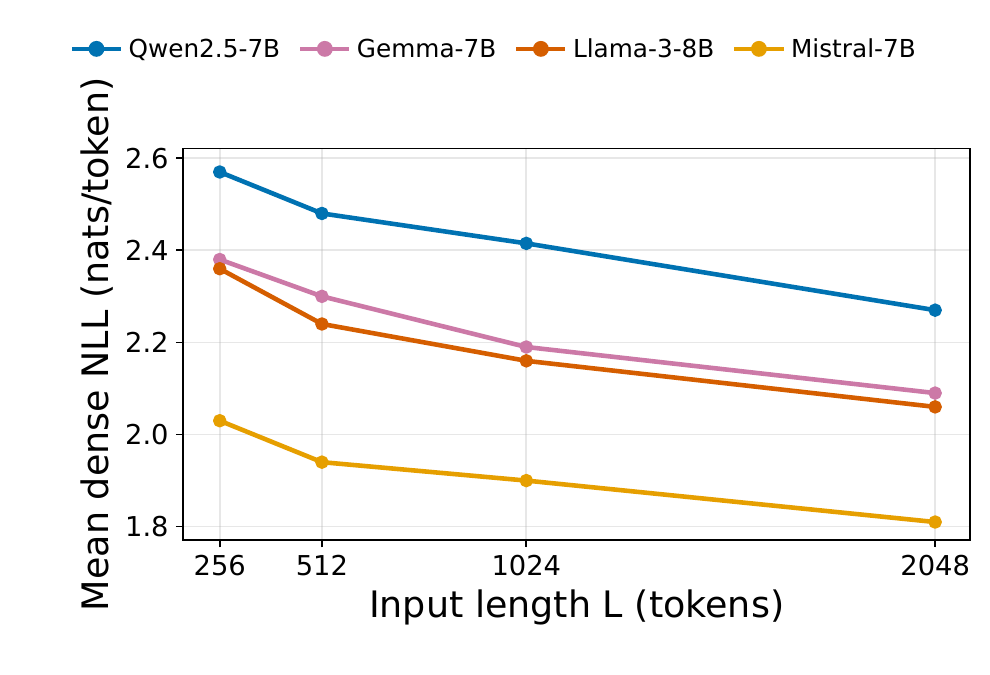}\\[-2pt]
{\small (b) Full-model prediction performance}
\end{minipage}
\caption{
\textbf{Longer context improves full-model prediction and increases
the effective attention set size.}
(a) Estimated set size at a 5\% relative NLL tolerance against the
full model at each context length.
(b) Mean full-model NLL.
Both panels use 50 OpenWebText documents per model, scoring the same
final 128 target tokens at every context length.
Shading in (a) shows 95\% paired-document bootstrap intervals
conditional on the evaluated set sizes.
}
\label{fig:context-width-main}
\end{figure}

\subsection{Retrieval under competition}
\label{sec:competition}
\label{sec:distractors}

To examine competition while controlling annotated support, we use
BABILong \texttt{qa1}, where the same supporting fact is retained as
background text increases from 0K to 4K tokens. We first evaluate the unrestricted
model to establish the reference performance for each condition.
Figure~\ref{fig:qa-distractors-main}a reports candidate accuracy.
Mean candidate answer loss increases from 0K to 4K in all four models,
with paired loss-change intervals above zero
(Appendix~\ref{app:dense-qa-background}).

We estimate the selected-set size needed to keep mean answer-loss
degradation within $0.10$ nats of the full model at each background
(Figure~\ref{fig:qa-distractors-main}b). The required size increases
substantially in several models despite the unchanged annotated
support. Within the useful-token framework, this is consistent with
additional competitors increasing the number of tokens that must be
retained to recover supporting information. Qwen's weaker and
nonmonotonic response shows that support displacement does not
translate uniformly into a larger functional requirement. Its effect
on answer loss also depends on how retained information is weighted
and used by subsequent computation.

To study token competition, we measure the position of
support tokens in the attention ranking, their total attention mass,
and their recall among the 64 highest-weight tokens. As background
grows, the same annotated support moves down the ranking, receives
less attention mass, and is less frequently retained within this
fixed-size set (Figure~\ref{fig:e1-mechanism}). These observations are
consistent with competition making it harder for attention to select
the information required by the task, motivating a functional test
of how many tokens must be retained.

\begin{figure}[h]
\centering
\begin{minipage}[b]{0.49\linewidth}
\centering
\includegraphics[width=\linewidth]
{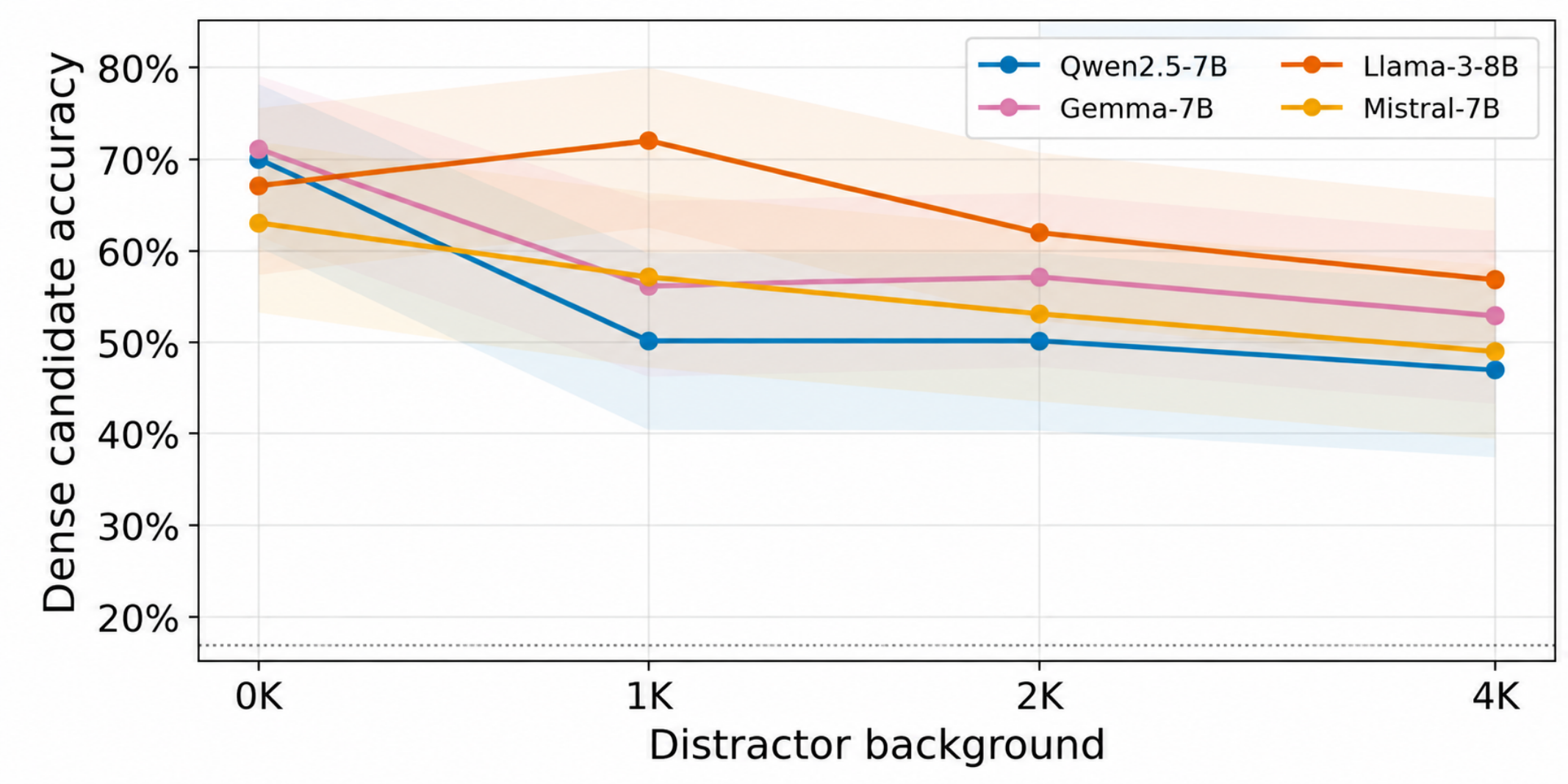}\\[2pt]
{\small (a) Full-model candidate accuracy}
\end{minipage}\hfill
\begin{minipage}[b]{0.49\linewidth}
\centering
\includegraphics[width=0.96\linewidth]
{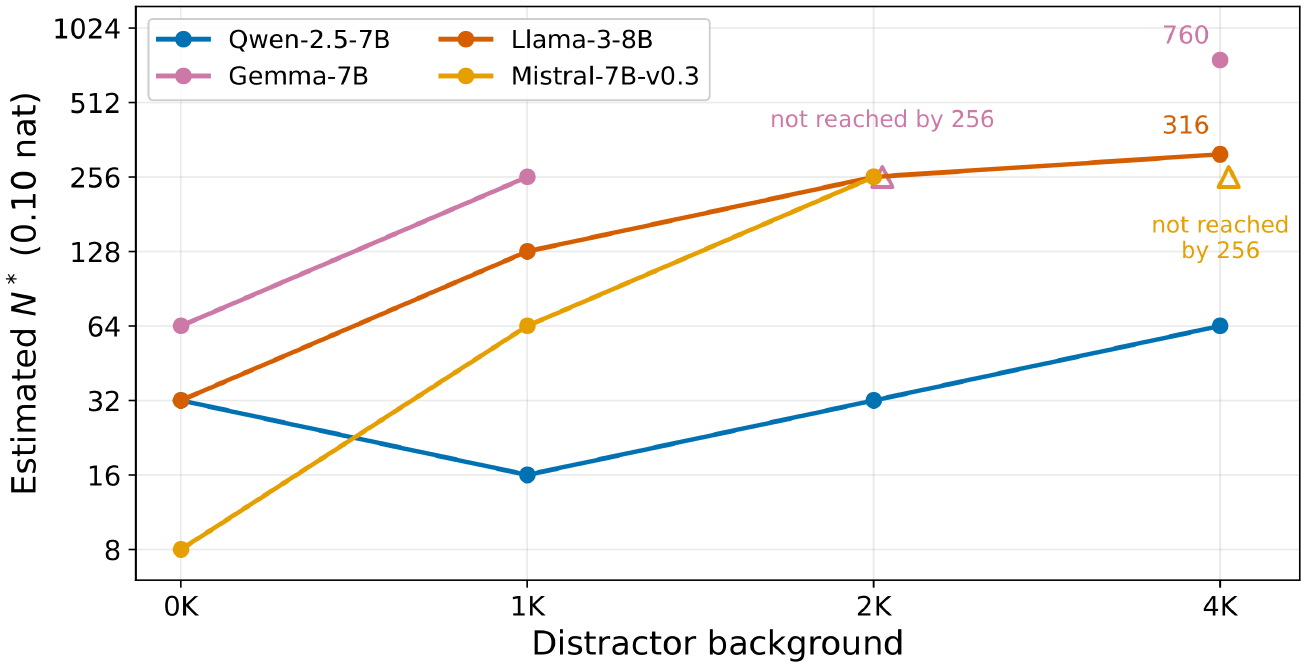}\\[2pt]
{\small (b) Effective attention set size}
\end{minipage}
\caption{
\textbf{Longer background text tends to lower full-model accuracy
and increase the effective attention set size despite fixed
annotated support.}
BABILong \texttt{qa1} results on 100 matched examples, with six
candidate answers scored using first-continuation-token logits.
(a) Full-model candidate accuracy. Shading shows 95\% Wilson
intervals, and the dotted line indicates chance accuracy.
(b) Estimated effective attention set size under attention-weight
selection at a $0.10$-nat mean answer-loss tolerance relative to
the full model at each background. Open triangles indicate that no evaluated set
size through $256$ meets the tolerance.
}
\label{fig:qa-distractors-main}
\end{figure}

The support records clarify how competition increases
(Appendix~\ref{app:support-details}). An individual non-support token
becomes less likely to outrank a support token as background grows
in every model. However, the increasing number of competitors produces
more tokens ranked above support overall. Improved pairwise ranking
can therefore coexist with poorer support recall at a fixed selected-set
size. In the useful-token framework, ranking errors accumulated over
a larger candidate set can increase the number of tokens needed to
retain a fixed required set, even when individual comparisons become
more favorable.

\begin{figure*}[b]
\centering
\begin{minipage}{0.325\textwidth}
\centering
\includegraphics[width=\linewidth]
{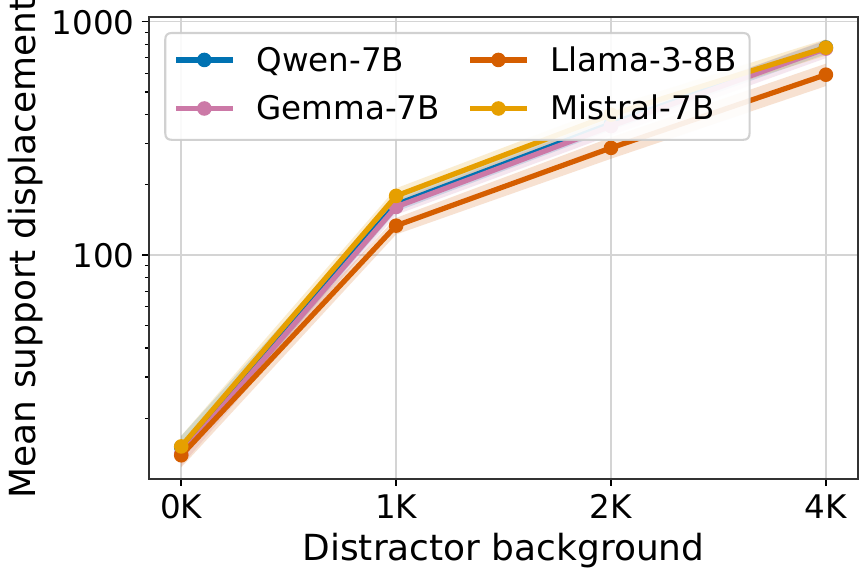}\\[-2pt]
{\small (a) Rank displacement}
\end{minipage}\hfill
\begin{minipage}{0.325\textwidth}
\centering
\includegraphics[width=\linewidth]
{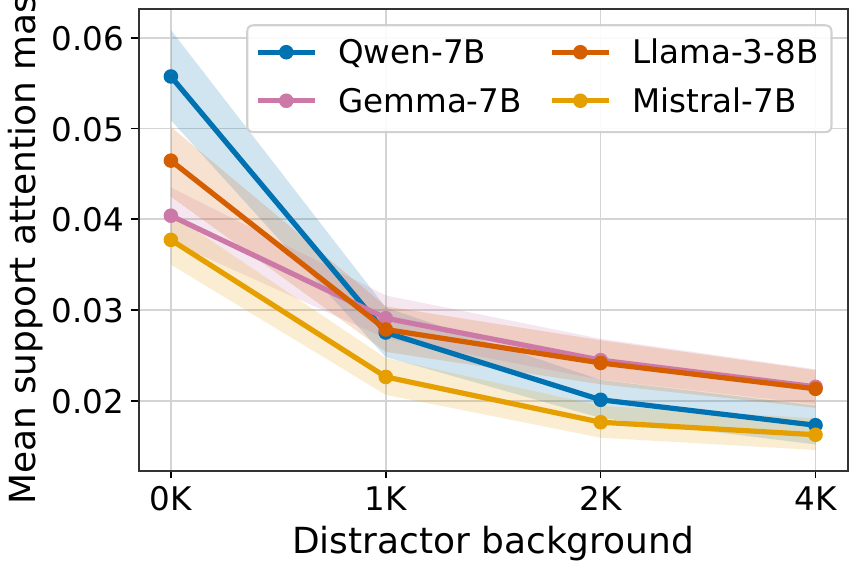}\\[-2pt]
{\small (b) Support attention mass}
\end{minipage}\hfill
\begin{minipage}{0.325\textwidth}
\centering
\includegraphics[width=\linewidth]
{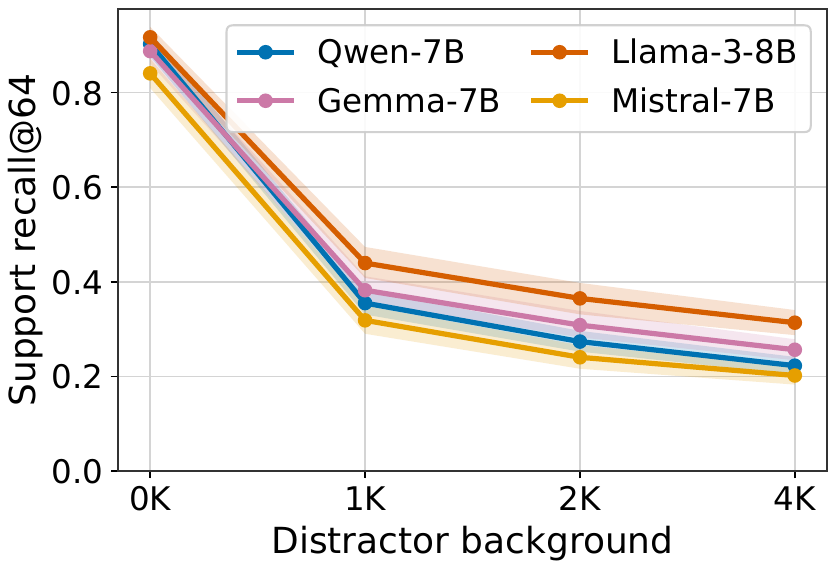}\\[-2pt]
{\small (c) Support recall@64}
\end{minipage}
\caption{
\textbf{Additional background pushes support tokens down the attention
ranking and reduces their attention mass and recall.}
Measurements at the answer query in BABILong \texttt{qa1} show support displacement, attention
mass, and recall among the 64 highest-weight tokens. Statistics use
86 examples per background after tokenizer-span validation.
}
\label{fig:e1-mechanism}
\end{figure*}

A conditional ranking model formalizes this mechanism.
Consider $L$ tokens with a fixed set of $K<L$ useful tokens.
Treat their scores as fixed, assume no ties, and draw competitor
scores independently from a common distribution. If $p_L$ is the
probability that a competitor exceeds the weakest required score,
the expected selected-set size needed to retain all useful tokens is
\[
\mathbb E[N_{\rm rec}(L)]=K+(L-K)p_L.
\]
Consequently, even a fixed useful set can need an increasing
selected set as more competitors are introduced.
Appendix~\ref{app:context-width-theory} gives the derivation and
probability bounds. The probability $p_L$ concerns the weakest
useful token and is distinct from the measured average pairwise
outranking probability. 

\subsection{Effect of attention normalization}
\label{sec:aggregation}

Removing tokens while preserving their original weights reduces total
attention mass. Results in Appendix~\ref{app:renorm-details}
show that renormalizing the retained attention weights can substantially reduce
the required set size and its increase between background conditions.
The effect varies with the model and loss tolerance. Each comparison
measures preservation of the corresponding full-model baseline,
whose answer loss also changes with background.

The local output identities explain how normalization affects the
retained computation. For fixed incoming activations, let
$M_S\in(0,1)$ denote retained attention mass, and let $\mu_S$ and
$\mu_T$ be the normalized weighted means of retained and discarded
values. The errors under deletion and renormalization are
\[
e_{\rm del}=(1-M_S)\mu_T,
\qquad
e_{\rm ren}=(1-M_S)(\mu_T-\mu_S).
\]
When these means are close, rescaling can reduce the output perturbation
despite removing substantial attention mass. The identities also allow
renormalization to increase local error, while full-model interventions
can change subsequent representations and rankings
(Appendix~\ref{app:local-bounds}).

This dependence on aggregation complements the useful-token hypothesis.
Under the distributional assumptions in
Appendix~\ref{app:theory-mass-width}, preserving a fixed fraction of
attention mass requires a selected set proportional to context length.
Additional norm and alignment conditions connect this requirement to
local output accuracy. An idealized task-loss example with identical
values demonstrates that growing sets can be needed even without
additional distinct information
(Appendix~\ref{app:theory-renormalization}). The effective attention
set size can therefore reflect both recovery of useful tokens and
preservation of the weighted sum.

\section{Discussion}
\label{sec:discussion}
\label{sec:related}
\label{sec:limitations}

Self-attention retrieves information by aggregating context-dependent
value vectors \citep{vaswani2017attention}. Modern Hopfield analyses
formalize its connection to associative memory and derive
storage-capacity guarantees under idealized assumptions
\citep{ramsauer2021hopfield}. These guarantees concern the number of
patterns that can be stored and recovered. Our analysis complements
this perspective by measuring how many highly scored tokens must
be retained at each attention operation to keep average prediction
loss within a specified tolerance.

The useful-token hypothesis provides a framework for interpreting
this measurement. The required set size can increase when more
tokens are useful, when useful tokens rank below competing sources,
or when the model is more sensitive to discarded contributions.
The useful set remains unobserved, and attention weights alone do
not establish relevance \citep{brunner2020identifiability}.
Building on \citet{mudarisov2026geometric}, we examine the geometry
of selected contributions and relate it to functional loss.
The geometric measurements depend on vector magnitudes and
directions relative to the selected aggregate. This dependence
limits their interpretation as evidence of relevance, and the
observed geometry--loss trajectories do not establish a general
predictor of effective attention set size.

Sparse attention mappings produce explicitly sparse weights
\citep{martins2016sparsemax,peters2019sparse}. Studies of attention
sinks, active and dormant heads, and geometric token selection
identify additional structure that can support efficient inference
\citep{xiao2023streaming,guo2024active,shin2025orthorank}.
Our interventions measure how restricting attention affects
predictive performance in existing models. Because they compute
dense attention before selection, the experiments do not
demonstrate an inference speedup.

The normalization controls show that effective set size depends
on how retained values are combined, complementing the ranking
explanation suggested by the fixed-support experiments.
The conditional theory in Appendix~\ref{app:context-width-theory}
describes settings in which competition or preservation of the
weighted sum requires larger sets as context grows. Such growth
can occur even when the amount of distinct task-relevant information
is fixed. These results provide possible explanations for the
observed context dependence, although their assumptions have not
been verified for the tested models and the experiments do not
establish a universal scaling law.

The measurement also has several interpretive limits. We apply a
common set size across attention operations, while allowing the
selected positions to vary with the query, head, and layer.
Consequently, the estimate does not identify a single subset of
original context tokens used by the model or recover the size of
a local useful set. Annotated BABILong support provides only a
partial reference for relevance. Average loss can conceal changes
on individual examples, and preserving a weak full-model baseline
does not establish successful task solving. Threshold estimates
and their uncertainty are conditional on the evaluated set sizes.
Localized interventions further reveal heterogeneous sensitivity
across layers and heads, whose effects cannot be treated as
independent components of the full-model requirement.

\section{Conclusion}
\label{sec:conclusion}

Measuring loss degradation under Top-$N$ selection provides an
operational estimate of how many tokens attention needs to retain
to preserve average predictive performance within a chosen loss tolerance.
Attention and contribution rankings produce substantially lower
loss than random selection at the same set sizes, although geometric
separation alone does not determine how large a set is sufficient. 
For fixed prediction targets, longer context increases the required
set size while reducing its fraction of the available context.
When annotated support is held fixed, additional background
displaces support in the ranking and increases the required size
in several models. Renormalization further shows that preserving
the weighted sum affects these estimates. The useful-token
hypothesis and conditional theory connect these findings to the
amount of useful information, its ranking among competing tokens,
and the aggregation of retained values. Effective attention set
size captures the combined effect of these factors on predictive
performance under a specified intervention.

\clearpage
\bibliography{references}

@article{yang2024qwen25,
  title   = {{Qwen2.5} Technical Report},
  author  = {Yang, An and others},
  journal = {arXiv preprint arXiv:2412.15115},
  year    = {2024}
}

@article{mesnard2024gemma,
  title   = {Gemma: Open Models Based on {Gemini} Research and Technology},
  author  = {{Gemma Team} and Mesnard, Thomas and others},
  journal = {arXiv preprint arXiv:2403.08295},
  year    = {2024}
}

@article{riviere2024gemma2,
  title   = {{Gemma 2}: Improving Open Language Models at a Practical Size},
  author  = {{Gemma Team} and Riviere, Morgane and others},
  journal = {arXiv preprint arXiv:2408.00118},
  year    = {2024}
}

@article{touvron2023llama2,
  title   = {{Llama 2}: Open Foundation and Fine-Tuned Chat Models},
  author  = {Touvron, Hugo and others},
  journal = {arXiv preprint arXiv:2307.09288},
  year    = {2023}
}

@article{grattafiori2024llama3,
  title   = {The {Llama 3} Herd of Models},
  author  = {Grattafiori, Aaron and others},
  journal = {arXiv preprint arXiv:2407.21783},
  year    = {2024}
}

@misc{meta2024llama32,
  author       = {{Meta AI}},
  title        = {{Llama 3.2} Model Card},
  year         = {2024},
  howpublished = {Model card}
}

@article{jiang2023mistral,
  title   = {{Mistral 7B}},
  author  = {Jiang, Albert Q. and others},
  journal = {arXiv preprint arXiv:2310.06825},
  year    = {2023}
}

@misc{mistralai2025mistralsmall24b,
  author       = {{Mistral AI}},
  title        = {{Mistral-Small-24B-Base-2501}},
  year         = {2025},
  howpublished = {Model card}
}

@misc{gokaslan2019openwebtext,
  title        = {{OpenWebText} Corpus},
  author       = {Gokaslan, Aaron and Cohen, Vanya and Pavlick, Ellie and Tellex, Stefanie},
  year         = {2019},
  howpublished = {\url{https://skylion007.github.io/OpenWebTextCorpus/}}
}

@inproceedings{merity2017pointer,
  title     = {Pointer Sentinel Mixture Models},
  author    = {Merity, Stephen and Xiong, Caiming and Bradbury, James and Socher, Richard},
  booktitle = {International Conference on Learning Representations},
  year      = {2017}
}

@article{kuratov2024babilong,
  title   = {{BABILong}: Testing the Limits of {LLMs} with Long Context Reasoning-in-a-Haystack},
  author  = {Kuratov, Yuri and Bulatov, Aydar and Anokhin, Petr and Rodkin, Ivan
             and Sorokin, Dmitry and Sorokin, Artyom and Burtsev, Mikhail},
  journal = {arXiv preprint arXiv:2406.10149},
  year    = {2024}
}

@article{vaswani2017attention,
  title={Attention Is All You Need},
  author={Vaswani, Ashish and Shazeer, Noam and Parmar, Niki and Uszkoreit, Jakob and Jones, Llion and Gomez, Aidan N. and Kaiser, Lukasz and Polosukhin, Illia},
  journal={arXiv preprint arXiv:1706.03762},
  year={2017}
}

@article{brunner2020identifiability,
  title={On Identifiability in Transformers},
  author={Brunner, Gino and Liu, Yang and Pascual, Dami{\'a}n and Richter, Oliver and Ciaramita, Massimiliano and Wattenhofer, Roger},
  journal={International Conference on Learning Representations},
  year={2020}
}

@article{martins2016sparsemax,
  title={From Softmax to Sparsemax: A Sparse Model of Attention and Multi-Label Classification},
  author={Martins, Andr{\'e} F. T. and Astudillo, Ram{\'o}n Fernandez},
  journal={Proceedings of the 33rd International Conference on Machine Learning},
  year={2016}
}

@article{peters2019sparse,
  title={Sparse Sequence-to-Sequence Models},
  author={Peters, Ben and Niculae, Vlad and Martins, Andr{\'e} F. T.},
  journal={Proceedings of the 57th Annual Meeting of the Association for Computational Linguistics},
  year={2019}
}

@article{xiao2023streaming,
  title={Efficient Streaming Language Models with Attention Sinks},
  author={Xiao, Guangxuan and Tian, Yuandong and Chen, Beidi and Han, Song and Lewis, Mike},
  journal={arXiv preprint arXiv:2309.17453},
  year={2023}
}

@article{guo2024active,
  title={Active-Dormant Attention Heads: Mechanistically Demystifying Extreme-Token Phenomena in {LLMs}},
  author={Guo, Tianyu and Pai, Druv and Bai, Yu and Jiao, Jiantao and Jordan, Michael I. and Mei, Song},
  journal={arXiv preprint arXiv:2410.13835},
  year={2024}
}

@article{shin2025orthorank,
  title={{OrthoRank}: Token Selection via Sink Token Orthogonality for Efficient {LLM} Inference},
  author={Shin, Seungjun and Oh, Jaehoon and Oh, Dokwan},
  journal={arXiv preprint arXiv:2507.03865},
  year={2025}
}

@article{mudarisov2026geometric,
  title={Geometric Analysis of Token Selection in Multi-Head Attention},
  author={Mudarisov, Timur and Burtsev, Mikhail and Petrova, Tatiana and State, Radu},
  journal={arXiv preprint arXiv:2602.01893},
  year={2026}
}

@inproceedings{ramsauer2021hopfield,
  title={{Hopfield} Networks Is All You Need},
  author={Ramsauer, Hubert and Sch{\"a}fl, Bernhard and Lehner, Johannes and Seidl, Philipp and Widrich, Michael and Gruber, Lukas and Holzleitner, Markus and Adler, Thomas and Kreil, David and Kopp, Michael K. and Klambauer, G{\"u}nter and Brandstetter, Johannes and Hochreiter, Sepp},
  booktitle={International Conference on Learning Representations},
  year={2021}
}
\bibliographystyle{plainnat}

\clearpage
\appendix

\FloatBarrier
\renewcommand{\topfraction}{0.95}
\renewcommand{\bottomfraction}{0.9}
\renewcommand{\textfraction}{0.05}
\renewcommand{\floatpagefraction}{0.75}
\makeatletter
\setlength{\@fptop}{0pt}
\makeatother

\section{Formal details of attention selection}
\label{app:formal}

Section~\ref{sec:setup} relates an unknown useful set to observable rankings,
geometric separation, and prediction loss. We give the recovery bounds behind
the useful-token framework, specify the geometric construction, and derive
local output-error identities for the two interventions. These results clarify
how each measurement contributes to the analysis of effective attention set size.

\subsection{Useful-set recovery from ranking inversions}
\label{app:topk-recovery}

The useful-token hypothesis allows ranking errors to increase the number of
tokens needed to recover a useful set. To quantify this relationship, fix a
context $C=(X,\ell,h,t)$ with $1\le K(C)<|\Iset_t|$ and recall that
\begin{equation}
Z_i(C)\in\{0,1\},\qquad
\Sel^*(C)=\{i:Z_i(C)=1\},\qquad K(C)=|\Sel^*(C)|.
\label{eq:latent-variable}
\end{equation}
The labels represent hypothesized downstream relevance and are unobserved.
For an observable score $r_i$, define
\begin{equation}
\Inv_r(C)=
\sum_{i\in\Sel^*(C)}\sum_{j\in\Iset_t\setminus\Sel^*(C)}
\mathbf1\{r_j\ge r_i\}.
\label{eq:latent-inversions}
\end{equation}
Counting score ties as inversions makes the following bounds valid under the
index-based tie-breaking rule in Section~\ref{sec:setup}.

\begin{proposition}[Useful-set recovery from ranking inversions]
\label{prop:topk-recovery-app}
\label{prop:topn-recovery}
Fix $C$ and $r$, and write $K=K(C)$ and $I=\Inv_r(C)$. For
$K\le N\le|\Iset_t|$, let $m=|\Sel^*(C)\setminus\Shat_N^{(r)}|$.
Then
\begin{equation}
m(N-K+m)\le I,\qquad
m\le\min\!\left\{K,
\left\lfloor\frac{\sqrt{(N-K)^2+4I}-(N-K)}{2}\right\rfloor\right\}.
\label{eq:topn-inversion-bound}
\end{equation}
Also $m\le I/(N-K+1)$, so $N\ge K+I$ is sufficient for complete
recovery whenever this choice is feasible. At $N=K$, precision, recall, and
their harmonic mean relative to $\Sel^*(C)$ satisfy
\begin{equation}
P_{\rm lat}=R_{\rm lat}=F_{\rm lat}
=1-\frac mK\ge\max\!\left\{0,1-\frac{\sqrt I}{K}\right\}.
\label{eq:latent-bound}
\end{equation}
Under Hypothesis~\ref{hyp:latent-ranking}, this value is at least
$1-\sqrt\varepsilon$ with probability at least $1-\delta$ under
$\mathcal D_0$.
\end{proposition}

\begin{proof}
The selected set contains $K-m$ useful tokens and $N-K+m$ other tokens.
Every selected token outside the useful set has a score at least as large as
every missed useful token. These pairs contribute $m(N-K+m)$ inversions.
Solving the resulting quadratic inequality and using that $m$ is an integer
gives the first bound. For $m\ge1$,
$m(N-K+1)\le m(N-K+m)\le I$, and the same bound holds when $m=0$.
If $N-K\ge I$, a missed token would imply
$I\ge N-K+1\ge I+1$, which is impossible. At $N=K$, both sets have
$K$ elements and share $K-m$ tokens, giving the stated precision and recall.
Substituting $I\le\varepsilon K^2$ proves the probabilistic statement.
\end{proof}

For $N>K$, recall is $1-m/K$, precision is $(K-m)/N$, and their harmonic
mean is $2(K-m)/(N+K)$. Increasing the selected-set size can therefore
compensate for ranking errors while $K(C)$ remains fixed. These guarantees
concern recovery of the hypothesized useful set. Their connection to
prediction loss also depends on the retained weights and subsequent
computation, as examined by the functional interventions.

The reference distribution is consequential. Under a uniformly random strict
ordering, $\mathbb E[I\mid K,|\Iset_t|]=K(|\Iset_t|-K)/2$, so the
$K^2$ bound in Hypothesis~\ref{hyp:latent-ranking} becomes demanding when
the useful set is small relative to the context. Its constants are specified
on $\mathcal D_0$ and need not hold uniformly as background grows.
Appendix~\ref{app:theory-rank-competition} develops the corresponding
context-dependent recovery model.

\subsection{Geometric construction and boundary effects}
\label{app:geometry-details}

Section~\ref{sec:exp-geometry} uses geometry to characterize the sets produced
by each ranking before testing their functional sufficiency. The construction
below makes explicit how these metrics depend on the selected aggregate and
on the treatment of boundary points.

Fix a ranking $r$, a distance $d$, and $1\le N<|\Iset_t|$. Write
$S=\Shat_N^{(r)}$, $s_N=\sum_{i\in S}y_i$, and $D_i=d(y_i,s_N)$.
For $\rho\ge0$, define
\[
B_d^{\le}(s_N,\rho)=\{i\in\Iset_t:D_i\le\rho\},\qquad
B_d^{<}(s_N,\rho)=\{i\in\Iset_t:D_i<\rho\}.
\]
When the closed neighborhood is nonempty, its precision and recall relative
to the selected set are
\[
P_{\rm geo}(\rho,N)=\frac{|S\cap B_d^{\le}(s_N,\rho)|}
{|B_d^{\le}(s_N,\rho)|},\qquad
R_{\rm geo}(\rho,N)=\frac{|S\cap B_d^{\le}(s_N,\rho)|}{N}.
\]
Let $\rho_{\min}=\min_{j\notin S}D_j$ and
$\rho_{\max}=\max_{i\in S}D_i$. The extremal metrics are
\[
P_N=P_{\rm geo}(\rho_{\max},N),\qquad
R_N=\frac{|S\cap B_d^{<}(s_N,\rho_{\min})|}{N},\qquad
F_N=\frac{2P_NR_N}{P_N+R_N}.
\]
The outer neighborhood contains every selected point, whereas the strict
inner neighborhood excludes every unselected point. This convention counts
an unselected boundary point as contamination and excludes a selected point
tied with the nearest unselected point from recall. Since $P_N>0$ on the
specified domain, the denominator of $F_N$ is positive. The two components
use different radii, so $F_N$ summarizes extremal separation.

\paragraph{Geometric inversions.}
Define the separation margin and distance-based inversion count by
\[
\gamma_N=\rho_{\min}-\rho_{\max},\qquad
\GInv_N=\sum_{i\in S}\sum_{j\notin S}\mathbf1\{D_j\le D_i\}.
\]
If $\mathrm{FP}_{\max}=|\{j\notin S:D_j\le\rho_{\max}\}|$ and
$\mathrm{FN}_{\min}=|\{i\in S:D_i\ge\rho_{\min}\}|$, then
\[
P_N=\frac{N}{N+\mathrm{FP}_{\max}}
\ge\frac{N}{N+\GInv_N},\qquad
R_N=1-\frac{\mathrm{FN}_{\min}}N
\ge\max\!\left\{0,1-\frac{\GInv_N}N\right\}.
\]
Each outer false positive forms an inversion with a farthest selected point.
Each inner false negative forms an inversion with a nearest unselected point,
which proves both inequalities. A positive margin is equivalent to
$\GInv_N=0$ and to $P_N=R_N=F_N=1$. These geometric inversions compare
distances against the observable selected set. The count $\Inv_r(C)$ in
Appendix~\ref{app:topk-recovery} instead compares scores against the unknown
useful set.

Changing $N$ changes the selected set, its aggregate, the distances, and the
extremal radii. Consequently, the geometric metrics need not be monotone in
$N$. Random controls use sets of the same size and recompute all quantities
around each random set's own aggregate. The nested sampling of these controls
is specified in Appendix~\ref{app:geometry-protocol}.

\paragraph{Singleton behavior.}
\label{app:singleton}
For Euclidean distance and $S=\{i\}$, $s_1=y_i$ and $D_i=0$.
If every unselected vector differs from $y_i$, then
$P_1=R_1=F_1=1$ for every selection rule, including random selection.
Duplicate vectors can prevent this strict separation. For the stabilized
cosine distance,
\[
d_{\rm C}(y,y)=\frac{\eta}{\|y\|_2^2+\eta},
\]
so the same exact singleton guarantee does not apply. If the aggregate is
zero, all stabilized cosine distances equal one and $R_N=F_N=0$.
These boundaries motivate the small-$N$ sensitivity analysis in
Appendix~\ref{app:geometry-boundary-check}.

\paragraph{Magnitude and direction.}
For positive $\alpha_i$ and nonzero $v_i$ and $s_N$, unregularized cosine
similarity satisfies
\[
\frac{\langle\alpha_i v_i,s_N\rangle}
{\|\alpha_i v_i\|_2\|s_N\|_2}
=\frac{\langle v_i,s_N\rangle}{\|v_i\|_2\|s_N\|_2}.
\]
The individual scale cancels, while the aggregate direction remains
attention-dependent. With the implemented regularizer, this cancellation
leaves $\eta/\alpha_i$ in the denominator and is only approximate when
the norm product dominates $\eta$. In Euclidean space,
$|\|y_j-s_N\|_2-\|s_N\|_2|\le\|y_j\|_2$, so small discarded
contributions lie near radius $\|s_N\|_2$. This property helps explain
the magnitude dependence of the metric, but alone does not imply separation.
Because $s_N$ is an unnormalized partial attention output, interpreting its
geometry requires the functional comparisons in Section~\ref{sec:exp-geometry}.

\subsection{Local output error under selection and renormalization}
\label{app:local-bounds}

The two rankings in Section~\ref{sec:setup} preserve different quantities
under deletion. Their local guarantees help interpret the functional results
and the normalization control in Section~\ref{sec:aggregation}. At fixed
incoming activations, write
\begin{equation}
y_i=\alpha_i v_i,
\label{eq:contribution}
\end{equation}
and, for $S=\Shat_N^{(r)}$,
\begin{equation}
e_t^{(r,N)}=z_t-\widetilde z_t^{(r,N)}
=\sum_{i\notin S}\alpha_i v_i.
\label{eq:tail-vector}
\end{equation}
This identity compares the full and selected sums formed from the same local
activations. Full-model interventions can also change the inputs to later
attention operations.

\begin{proposition}[Local deletion bounds]
\label{prop:local-bounds-app}
For any size-$N$ subset $S$, let
$M_S=\sum_{i\in S}\alpha_i$ and $e_S=\sum_{i\notin S}y_i$. Then
\[
\|e_S\|_2\le\sum_{i\notin S}\|y_i\|_2.
\]
Contribution ranking minimizes this upper bound over size-$N$ subsets.
Attention ranking maximizes $M_S$. If $\|v_i\|_2\le B$ for every visible
position, attention ranking also minimizes the upper bound
\[
\|e_S\|_2\le B(1-M_S).
\]
\end{proposition}

\begin{proof}
The triangle inequality gives
\[
\|e_S\|_2\le\sum_{i\notin S}\alpha_i\|v_i\|_2
\le B\sum_{i\notin S}\alpha_i=B(1-M_S).
\]
Retaining the largest $N$ contribution norms minimizes the first sum.
Retaining the largest $N$ attention weights maximizes retained mass and
minimizes the final expression.
\end{proof}

For a fixed subset, the contribution-norm bound is at least as tight as the
mass bound. Neither bound generally orders the actual errors of different
subsets, because discarded vectors can cancel. A single head's perturbation
in the residual stream is $W_O^{(h)}e_S$, and downstream loss also depends
on its direction. These local optima therefore provide motivation for the
rankings without guaranteeing optimal prediction loss.

\paragraph{Alignment of discarded contributions.}
For $\Tail=\Iset_t\setminus S$,
\[
\|e_S\|_2^2=\sum_{i\in\Tail}\|y_i\|_2^2
+2\sum_{\substack{i<j\\i,j\in\Tail}}\langle y_i,y_j\rangle.
\]
Alignment can amplify the error at fixed discarded energy, while cancellation
can reduce it. The exploratory analysis in Appendix~\ref{app:predictors}
uses the corresponding coherence statistic
\[
C_{\Tail}=\frac{\|\sum_{i\in\Tail}y_i\|_2^2}
{\sum_{i\in\Tail}\|y_i\|_2^2}
\]
when the denominator is nonzero.

\paragraph{Renormalization.}
Let $z_S=\sum_{i\in S}\alpha_i v_i$ and assume $0<M_S<1$. Define
the normalized weighted means $\mu_S=z_S/M_S$ and
$\mu_T=e_S/(1-M_S)$. The deletion and renormalization errors are
\begin{equation}
e_{\rm del}=(1-M_S)\mu_T,\qquad
e_{\rm ren}=z_t-\frac{z_S}{M_S}
=(1-M_S)(\mu_T-\mu_S).
\label{eq:renorm-centroids}
\end{equation}
For fixed weights and subset, translating every value by $b$ adds
$(1-M_S)b$ to the deletion error and leaves the renormalization error
unchanged. Renormalization reduces the local error when
$\|\mu_T-\mu_S\|_2<\|\mu_T\|_2$ and increases it when the reverse
inequality holds. If $M_S=1$, both errors vanish and a discarded-set mean
is unnecessary. Appendix~\ref{app:theory-renormalization} connects these
identities to an idealized loss-based set size.

\section{Experimental protocols}
\label{app:protocol}

The experiments in Section~\ref{sec:experiments} connect geometric structure
to loss preservation and then examine context, competition, and aggregation.
We specify the sampling, scoring, intervention, and uncertainty procedures
for each experiment. The geometric study, fixed-length language-model study,
and matched-target context-length study use distinct samples and aggregation
rules, as detailed below.

\subsection{Geometric evaluation}
\label{app:geometry-protocol}

To characterize the selected sets in Section~\ref{sec:exp-geometry}, we use
1024-position windows from individual documents and measure the final causal
query across all layers and query heads. The selected-set sizes are
$N\in\{1,2,4,8,16,32,64,128,256\}$. For each of the four models, the
geometric summaries contain 256 OpenWebText documents and 254 WikiText-103
documents. These counts differ from the 50-document functional evaluations.

Attention probabilities and value tensors are converted to FP32 before
computing $y_i=\alpha_i v_i$, distances, and norms. Stabilized cosine uses
$\eta=10^{-12}$ added to the product of norms. Measurements require
$N<|\Iset_t|$ so that both selected and unselected sets are nonempty.
The native attention implementation aligns grouped-query value states with
their corresponding query heads.

Each random control uses 16 independent random rankings. Within a repeat,
the ranking is shared across $N$, producing nested random sets. Every set
is evaluated around its own aggregate. Repeats are averaged within each
head, followed by heads and layers within a document, and then documents.
The geometric summaries use 2000 document-bootstrap replicates, with
documents as the resampling units.

\subsection{Fixed-length language modeling and threshold estimation}
\label{app:legacy-lm-protocol}

The fixed-length evaluation in Section~\ref{sec:exp-lm-width} estimates
effective attention set size across checkpoints. It uses 50 documents per
model and corpus, with exactly $L=1024$ positions per window, including
one prepended initial token. Document text is tokenized without automatic
special tokens. The initial ID is the tokenizer's beginning-of-sequence
(BOS) token when available, with the model configuration as a fallback.
A disagreement between available BOS IDs raises an error. Documents are
evaluated separately, without concatenation or sliding-window traversal.

\paragraph{Deterministic sampling.}
Documents are ordered by $H(2027,\mathrm{dataset},\mathrm{doc\_id})$.
Those with fewer than $L-1$ content tokens are skipped. For a document
with $T_x$ content tokens, the window offset is
\[
o_x=H(2027,\mathrm{model\_id},\mathrm{dataset},\mathrm{doc\_id},L)
\bmod(T_x-L+2).
\]
Here $H$ joins its arguments with \texttt{|}, encodes the result in UTF-8,
takes the first eight bytes of its SHA-256 digest in big-endian order, and
masks the result to 63 bits. The window consists of BOS followed by content
tokens $o_x,\ldots,o_x+L-2$. The first 50 eligible documents are used.
Eligibility depends on the tokenizer and offsets depend on model identity,
so cross-model comparisons need not use identical windows.

\paragraph{Loss and intervention.}
Logits at positions $0,\ldots,L-2$ predict the 1023 target tokens at
positions $1,\ldots,L-1$. Vocabulary logits are converted to FP32 for
cross-entropy. Chunk loss sums are accumulated as Python floating-point
scalars before division by the target count. The decoder processes the
entire window, and only the output-head computation is chunked. Chunk sizes
are 64 positions for the main adaptive evaluation, 32 for Mistral-Small-24B,
and 128 for the Gemma-2 and renormalization controls.

Corpus NLL is the mean document NLL, which equals token-level averaging
because all windows have the same target count. Relative degradation is
computed from the corpus means in Equation~\ref{eq:relative-nll}.
At every query, head, and layer, the selected set is recomputed from the
current activations in the intervened forward pass. Ties favor earlier
source indices, and queries with at most $N$ visible sources are unchanged.

\paragraph{Adaptive threshold search.}
The search reuses available powers-of-two evaluations and also evaluates
$N=64$. For each tolerance, let $h$ be the smallest known passing size,
using $L$ as the full-attention endpoint if necessary. Let $l<h$ be the
largest known failing size, with zero as an unevaluated sentinel when none
is available. The search evaluates $\lfloor(l+h)/2\rfloor$ and updates
the corresponding endpoint until $h-l=1$. It then evaluates the integers
from $\max(1,h-r)$ through $h-1$ in increasing order and accepts the first
pass. The radius is $r=3$ for the main and 24B evaluations and $r=2$ for
Gemma-2. Evaluations are shared across tolerances within a ranking.

These estimates resolve observed crossings locally. Since loss can be
nonmonotone in $N$, they do not certify the minimum over every integer in
Equation~\ref{eq:required-width}. Appendix~\ref{app:nll-capacity-by-dataset}
reports each corpus separately. Where a cross-corpus summary is given, it
is the maximum of the two estimates and carries no confidence-bound
interpretation. A full-$N$ intervention is checked on one configured
document per ranking. The search can otherwise use its exact zero-degradation
endpoint.

\subsection{Matched-target context-length evaluation}
\label{app:context-length-protocol}

Section~\ref{sec:context-length} isolates the effect of additional context
on a fixed set of prediction targets. This evaluation uses
$L\in\{256,512,1024,2048\}$ and 50 sufficiently long documents per model
and corpus. Within each model--corpus pair, nested suffixes end at the same
token and share the final 128 target IDs, denoted \texttt{tail128}. The
initial token counts toward $L$. The full-model baseline is evaluated
separately at each length. An auxiliary \texttt{all\_tokens} metric scores
every predicted position, changing both the target set and the distribution
of visible prefix lengths as $L$ varies. Pairing across lengths holds
within each model--corpus pair.

Both rankings use deletion without renormalization at relative NLL
tolerances of 1\%, 5\%, and 10\%. The search initializes all powers of
two through $L$, bisects a failure/pass bracket, checks $h-3$ through
$h+3$ within $[1,L]$, and checks the predecessor of the first tested pass.
After sharing evaluations across tolerances, it reports the smallest tested
passing size. As in the fixed-length evaluation, untested smaller integers
remain unresolved.

The recorded environment uses an H100 80GB GPU, bfloat16 without
quantization, PyTorch 2.6.0+cu124, Transformers 5.14.1, and seed 2027.
WikiText preparation retains 2044 of 2048 rows after removing four
duplicates with identical IDs and text. This sampling protocol differs
from Appendix~\ref{app:legacy-lm-protocol}, including at $L=1024$.

Qwen-2.5-7B, Gemma-7B, and Llama-3-8B have complete results on both
corpora at all four lengths. Mistral-7B-v0.3 has complete OpenWebText
results through $L=1024$, partial attention-ranking refinement at
$L=2048$, and no WikiText results. Thus 27 of the 32 planned
model--corpus--length conditions are complete. An aggregate uses only
selected-set sizes with all 50 document measurements. The incomplete
Mistral measurement at $N=164$ is excluded.

The recorded consistency checks cover run fingerprints, input-token hashes,
target counts, matched suffixes, and the 324 completed set-size estimates
across rankings, tolerances, and scoring protocols. All 28 saved numerical
parity checks pass, with maximum NLL discrepancy $5.76\times10^{-7}$.
These checks compare native and chunked computation on prefixes of at most
256 positions, including for runs at longer $L$. The endpoint $N=L$
reuses the full-model baseline and therefore adds no independent numerical
parity check at that length.

Uncertainty uses 5000 document-bootstrap replicates, with identical
resampling weights across lengths, rankings, and metrics within each
model--corpus pair. The resulting intervals are conditional on the
evaluated selected-set sizes and do not account for unmeasured crossings.

\subsection{Controlled QA, candidate scoring, and support measurements}
\label{app:qa-scoring}

The BABILong experiment in Section~\ref{sec:competition} examines competition
while holding the annotated supporting fact fixed. We use
\texttt{RMT-team/babilong-1k-samples}, task \texttt{qa1}, at backgrounds
0K, 1K, 2K, and 4K. Supplementary \texttt{qa2} and \texttt{qa3}
evaluations use 0K background. A NumPy permutation seeded by
$H(2027,\mathrm{task})$ selects 100 indices, shared across models and
backgrounds. The prompt is
\begin{quote}\begin{minipage}{\linewidth}\small\ttfamily\raggedright
Read the context and answer the question. Return only the short answer.\\[3pt]
Context:\\
\{input\}\\[3pt]
Question: \{question\}\\
Answer:
\end{minipage}\end{quote}
Blank lines separate the instruction, context, and question blocks. There
is no trailing space after the final colon. One resolved initial token is
prepended, and the remaining prompt uses
\texttt{add\_special\_tokens=False}.

\paragraph{Candidate-normalized answer loss.}
The candidate vocabulary is the sorted set of unique target strings in the
task split. The six labels are \texttt{bathroom}, \texttt{bedroom},
\texttt{garden}, \texttt{hallway}, \texttt{kitchen}, and \texttt{office}.
For candidate $a$, let $c_M(a)$ be the first token obtained by tokenizing
one ASCII space followed by $a$, separately from the prompt and without
special tokens. If $u_M(x)$ is the vocabulary-logit vector at the final
prompt position, the candidate score is
\[
s_M(a\mid x)=u_M(x)_{c_M(a)}.
\]
The six first-token IDs are distinct in the reported runs. The implementation
supports mean continuation-token log-probability as a fallback when IDs
coincide, but this fallback is not used here. Vocabulary logits are
extracted in FP32, and candidate log-sum-exp is evaluated in FP64.

Prediction uses the largest candidate score, with ties resolved in sorted
candidate order. Accuracy normalizes case, exterior whitespace, repeated
whitespace, and terminal punctuation. This evaluates six-way constrained
choice. Uniform candidate probabilities give accuracy $1/6$ and loss
$\log6$. Full and intervened predictions use identical prompts and
candidate IDs, and the mean degradation is
\begin{equation}
\Delta\mathcal L_{\rm ans}(N)=\frac1{|\mathcal D|}
\sum_{x\in\mathcal D}
[\ell_{\rm ans}(M_N;x)-\ell_{\rm ans}(M;x)].
\label{eq:answer-loss-degradation}
\end{equation}
The tolerances are 0.10 and 0.20 nats. The configured competence screen
requires full-model accuracy at least $1/6+0.05$ and loss at most
$\log6-0.05$. Baseline performance is reported explicitly because this
screen alone does not establish reliable task solving.

The initial selected-set sizes are $1,2,4,8,16,32,64,128,256$, with
adaptive extensions for selected conditions. If $N$ is at least the prompt
length, the intervention is an exact no-op. Such cases are included with
zero degradation in the complete 100-example cohort, including rows
omitted by the initial evaluator. Unresolved crossings indicate that no
tested size passes the criterion. They do not certify failure at every
untested smaller integer. Accuracy intervals use the Wilson method, and
background differences use 5000 paired bootstrap replicates over examples.

\paragraph{Support mapping.}
Support sentences are obtained from the source bAbI English and English-10k
\texttt{qa1} test annotations. Questions and targets are lowercased and
whitespace-normalized for matching. A candidate annotation is eligible
when every normalized supporting sentence occurs in the normalized
BABILong input. Identical supporting-fact tuples are deduplicated, and an
example is accepted only when exactly one distinct tuple remains. This
step identifies 90 unambiguous examples and 10 ambiguous examples per
background, with no unmatched cases.

Each accepted original support sentence must then occur exactly once,
with matching case, in the formatted prompt. Fast-tokenizer offsets
$[a_j,b_j)$ map a character span $[a,b)$ to tokens satisfying
$b_j>a$, $a_j<b$, and $b_j>a_j$. Indices are shifted by one for the
initial token and deduplicated. Examples require nonempty visible support
and non-support context. The resulting validated cohort contains 86
examples per background. The four additional exclusions are not assigned
individual reasons in the recorded output, so the combined mapping,
occurrence, and span checks define validation. Paired comparisons use the
recorded example identities and support spans.

\paragraph{Support statistics and aggregation.}
For example $x$, let $S_x$ be the support-token positions and $T_x$ the
positions overlapping the input-context field. Define
$D_x=T_x\setminus S_x$, $k_x=|S_x|$, and $d_x=|D_x|$. At the final
prompt query, attention scores $r_{ai}=\alpha_{ai}$ in layer/head
$a=(\ell,h)$ give
\begin{align}
\operatorname{Disp}_a(x)
&=\frac1{k_x}\sum_{i\in S_x}\sum_{j\in D_x}
\mathbf1\{r_{aj}\ge r_{ai}\}, &
\operatorname{Mass}_a(x)&=\sum_{i\in S_x}\alpha_{ai},
\label{eq:support-displacement-mass}\\
\operatorname{Rec64}_a(x)
&=\frac{|S_x\cap\operatorname{Top64}(r_a;\Iset_t)|}{k_x}, &
\widehat p_a(x)&=\frac{\operatorname{Disp}_a(x)}{d_x}.
\label{eq:support-recall-outrank}
\end{align}
Displacement counts non-support tokens in the context field. Top-64
selection ranges over all visible prompt positions, including the initial
token, instructions, and question. Recall is the fraction of support tokens
retained, and mass is their total attention weight. Displacement uses weak
score comparisons, while Top-64 follows the specified tie-breaking rule.

All valid layer/query-head measurements receive equal weight within each
example, after which examples receive equal weight:
\[
\bar q(x)=\frac1{|A_x|}\sum_{a\in A_x}q_a(x),\qquad
\bar q_{m,b}=\frac1{|\mathcal E_{m,b}|}
\sum_{x\in\mathcal E_{m,b}}\bar q(x).
\]
Here $\mathcal E_{m,b}$ is the validated cohort for model $m$ and
background $b$. Cross-model displacement and mass ratios average
$\bar q_{m,4K}/\bar q_{m,0K}$. Recall changes average the corresponding
differences in percentage points. Bootstrap sampling is performed over
examples after the within-example reduction. Descriptive log--log fits
use actual non-support counts, including the approximately 32--34 such
tokens already present at 0K.

\subsection{Implementation and numerical validation}
\label{app:numerical-implementation}

Numerical checks ensure that measured degradation reflects the intervention
while retaining each checkpoint's native attention computation. Selection
operates on eager attention probabilities and aligned value states after
the model's positional encoding, masking, scaling, and any attention-logit
softcapping. Equation~\ref{eq:index-set} supplies the common notation.
The output projection and subsequent model operations retain their native
implementations, and no parameters are trained.

The language-model evaluator runs the decoder once and chunks the output
head. For Gemma-2, it applies the model's final-logit tanh softcapping
before cross-entropy. The recorded comparison of native eager computation,
chunked NLL, and a full-$N$ intervention gives maximum chunked/eager NLL
differences of $6.19\times10^{-7}$ on OpenWebText and
$5.18\times10^{-7}$ on WikiText-103. The full-$N$ hook has zero discrepancy
in this check. The sequence-length scope of the separate matched-target
checks is specified in Appendix~\ref{app:context-length-protocol}.

\paragraph{Mistral-Small-24B precision.}
This checkpoint uses a single A6000 48GB GPU and bitsandbytes 8-bit
weight loading, with outlier threshold 6.0 and FP32 CPU offload disabled.
The requested dtype for non-quantized computation is bfloat16. This
configuration uses mixed-precision computation, including possible internal
FP16 casts. NLL logits and selection diagnostics are converted to FP32 as
specified by their routines. The recorded configuration does not pin an
immutable bitsandbytes build or checkpoint revision, which limits exact
reproduction of the historical binary environment. Comparisons with this
checkpoint therefore combine differences in model architecture, size,
and numerical precision.

\section{Additional geometry and language-model results}
\label{app:additional-width-results}

Sections~\ref{sec:exp-geometry} and \ref{sec:context-length} show that
score-based selection produces structured sets whose functional sufficiency
depends on the model and available context. This section extends those
comparisons across corpora and rankings, examines sensitivity to geometric
boundaries, and reports the complete set-size estimates and uncertainty.

\subsection{Geometric separation across corpora and rankings}
\label{app:geometry-extra}

The main text reports Euclidean geometry on OpenWebText under attention
ranking. Figures~\ref{fig:geometry-owt-attention}--\ref{fig:geometry-wt-contribution}
add cosine geometry, contribution ranking, and WikiText-103. Across these
comparisons, score-selected sets have stronger separation than random sets
of the same size. The advantage is larger in Euclidean geometry, consistent
with the magnitude dependence described in Appendix~\ref{app:geometry-details}.
The cosine results provide a complementary comparison of direction.

\begin{figure}[!htbp]
\centering
\includegraphics[width=0.85\linewidth]
{figures/geometry_owt_euclidean_attention.pdf}

\vspace{-4pt}
{\small (a) Euclidean geometry}

\vspace{6pt}
\includegraphics[width=0.85\linewidth]
{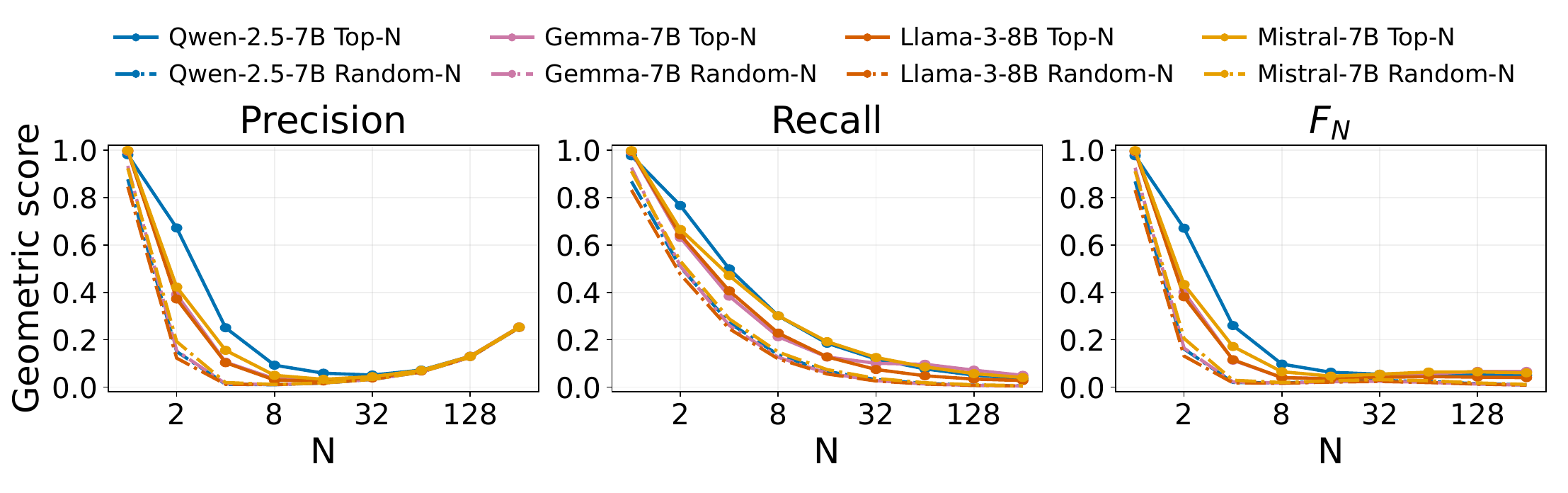}

\vspace{-4pt}
{\small (b) Cosine geometry}

\caption{
\textbf{Attention-selected sets show stronger separation than random sets on OpenWebText.}
Geometric precision, recall, and extremal $F_N$ under Euclidean and cosine
distance. Solid curves show attention ranking and dash-dotted curves show
random sets of the same size, each evaluated around its own aggregate.
Measurements use the final query of 1024-position windows and are averaged
over heads, layers, and documents. The Euclidean advantage is larger.
}
\label{fig:geometry-owt-attention}
\end{figure}
\begin{figure*}[!htbp]
\centering

\includegraphics[width=0.85\textwidth]
{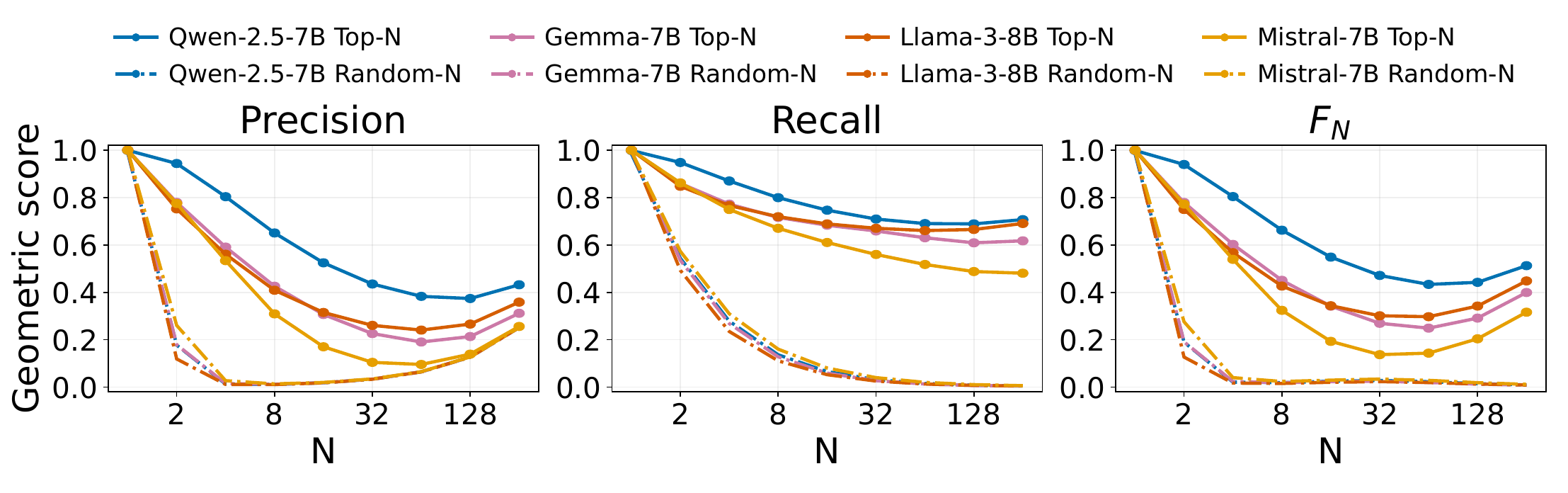}

{\small (a) OpenWebText, Euclidean, contribution ranking}

\vspace{6pt}

\includegraphics[width=0.85\textwidth]
{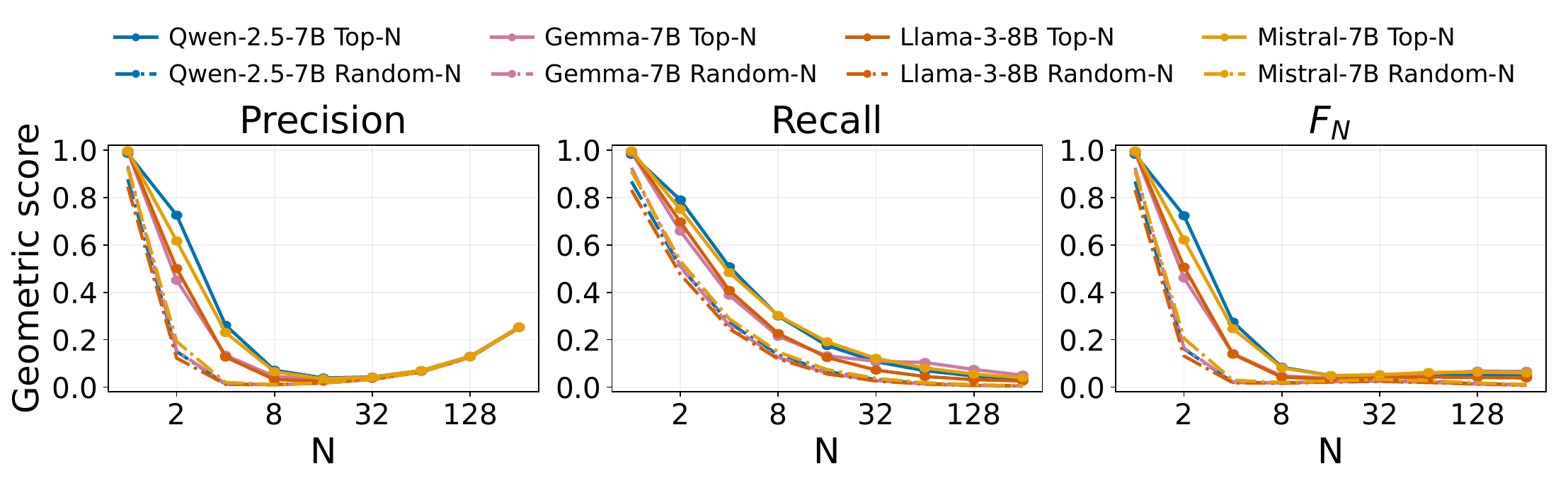}

{\small (b) OpenWebText, cosine, contribution ranking}

\caption{
\textbf{Contribution ranking reproduces the geometric pattern on OpenWebText.}
Precision, recall, and extremal $F_N$ under Euclidean and cosine distance.
Solid curves show contribution ranking and dash-dotted curves show random
selection. The evaluation and aggregation follow
Figure~\ref{fig:geometry-owt-attention}.
}
\label{fig:geometry-owt-contribution}
\end{figure*}
\begin{figure*}[!htbp]
\centering

\includegraphics[width=0.85\textwidth]
{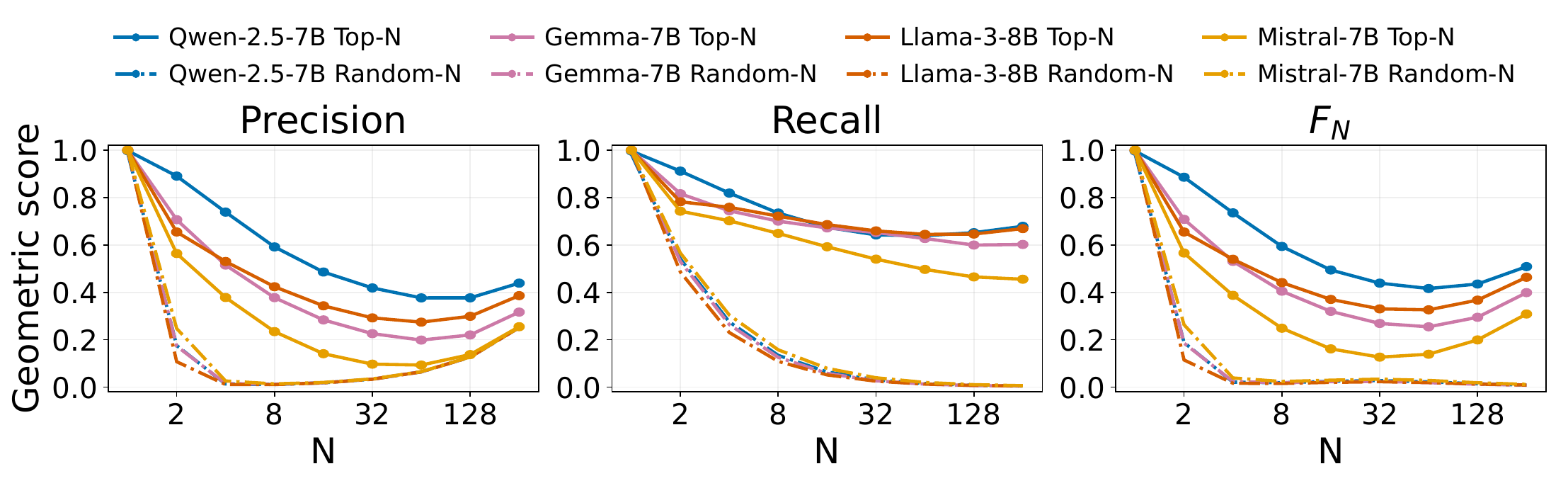}

{\small (a) WikiText-103, Euclidean, attention ranking}

\vspace{6pt}

\includegraphics[width=0.85\textwidth]
{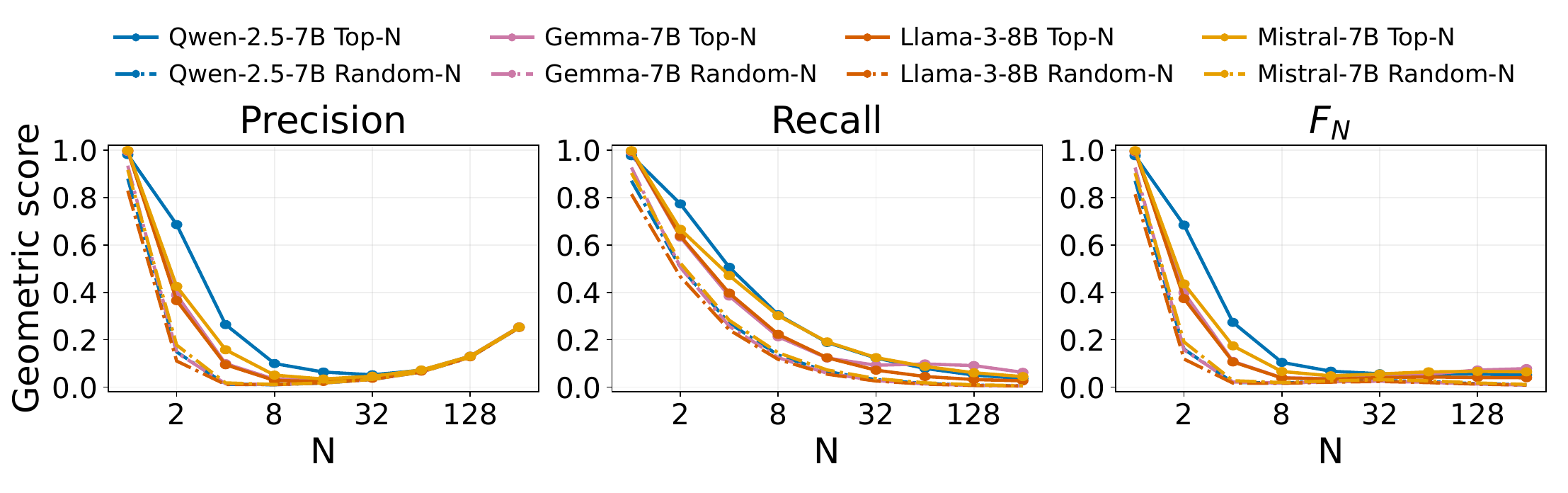}

{\small (b) WikiText-103, cosine, attention ranking}

\caption{
\textbf{Attention-selected geometry is consistent across corpora.}
WikiText-103 precision, recall, and extremal $F_N$ under Euclidean and
cosine distance. Solid curves show attention ranking and dash-dotted curves
show random selection. The evaluation and aggregation follow
Figure~\ref{fig:geometry-owt-attention}.
}
\label{fig:geometry-wt-attention}
\end{figure*}
\begin{figure*}[!htbp]
\centering

\includegraphics[width=0.85\textwidth]
{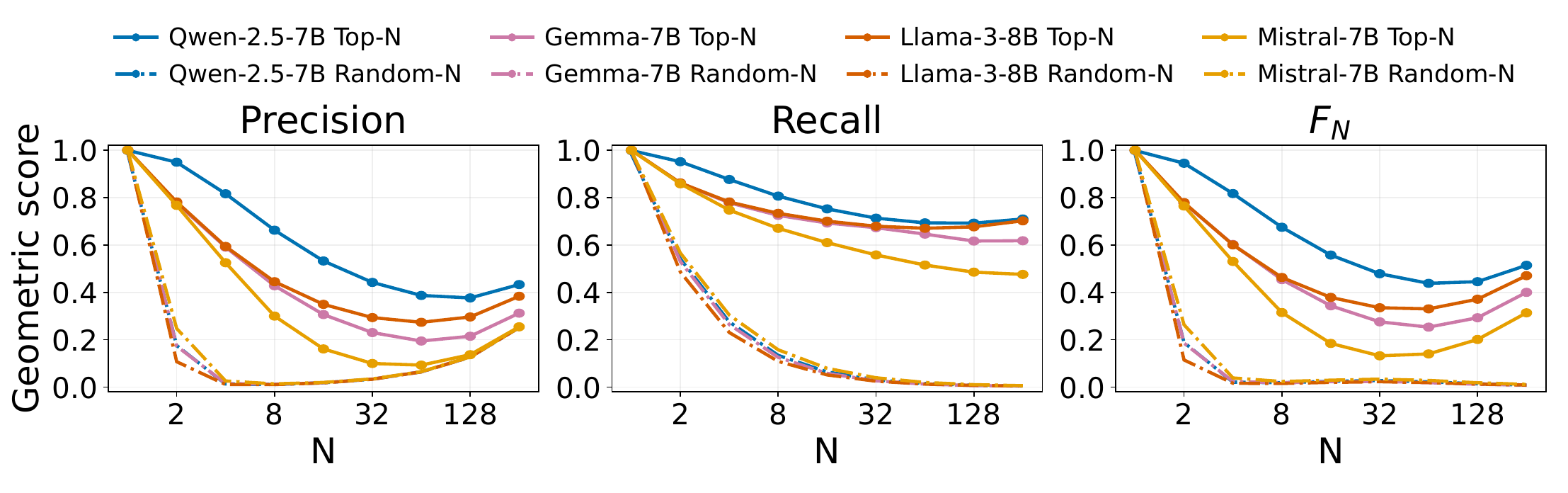}

{\small (a) WikiText-103, Euclidean, contribution ranking}

\vspace{6pt}

\includegraphics[width=0.85\textwidth]
{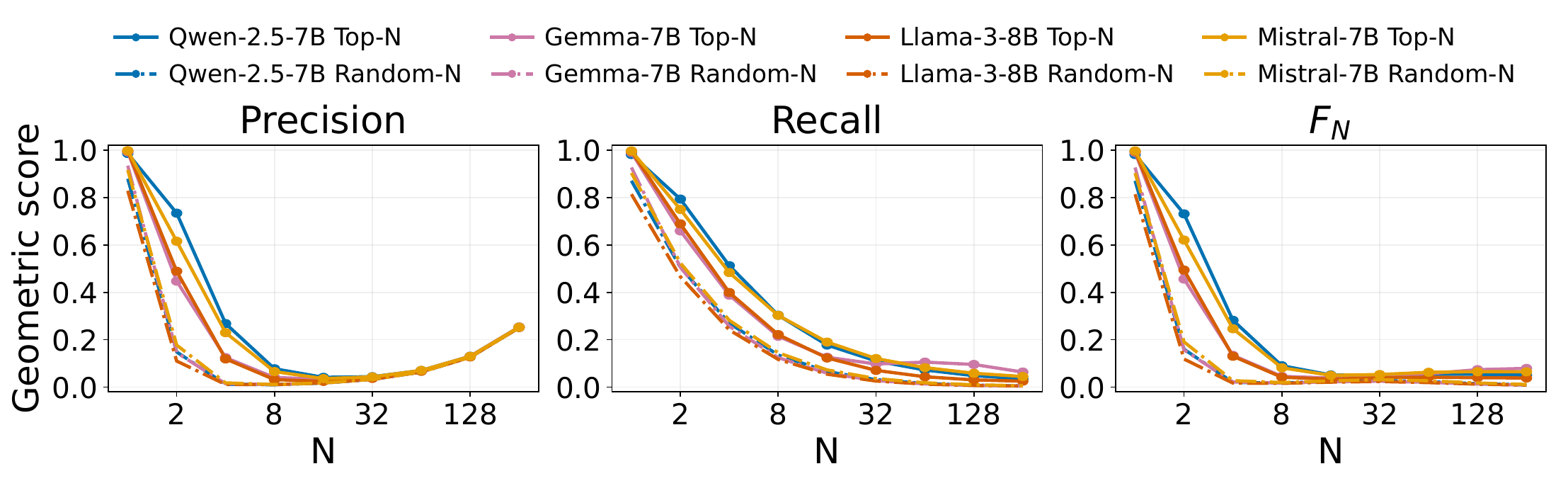}

{\small (b) WikiText-103, cosine, contribution ranking}

\caption{
\textbf{Contribution-selected sets also exhibit geometric separation on WikiText-103.}
Precision, recall, and extremal $F_N$ under Euclidean and cosine distance.
Solid curves show contribution ranking and dash-dotted curves show random
selection. The evaluation and aggregation follow
Figure~\ref{fig:geometry-owt-attention}.
}
\label{fig:geometry-wt-contribution}
\end{figure*}

\subsection{Geometry and functional degradation}
\label{app:geometry-function-scatter}

Figure~\ref{fig:geometry-nll-scatter-contribution} extends the main
geometry--loss comparison to contribution ranking. Each point pairs
final-query geometric separation with full-model NLL degradation at the
same selected-set size, averaging the two corpora. Both measurements vary
with $N$, and geometry and loss are evaluated over different query
populations. These trajectories therefore describe their joint variation
without establishing that geometric separation predicts functional
sufficiency for an individual attention operation.

\begin{figure*}[!htbp]
\centering

\includegraphics[width=0.85\textwidth]
{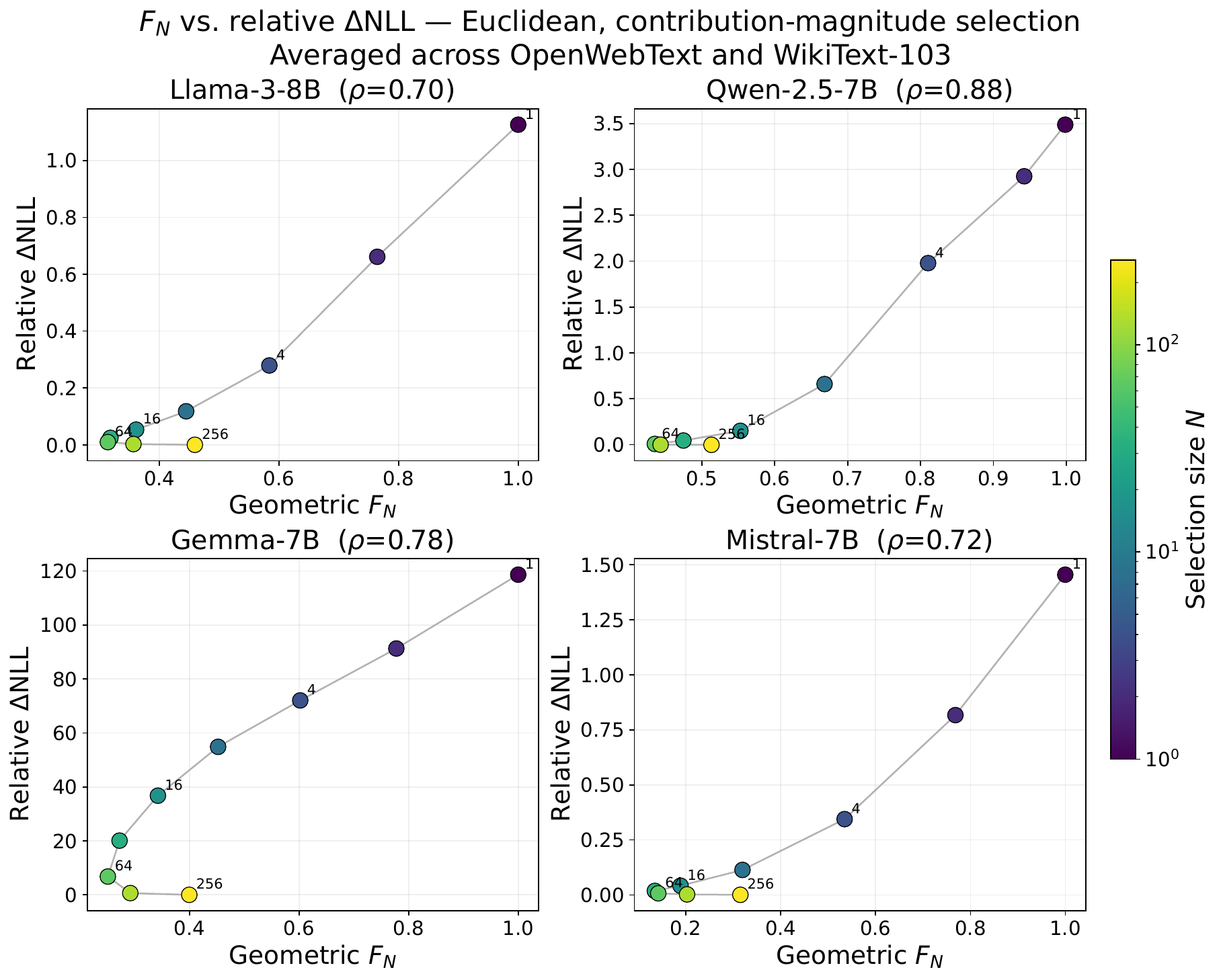}

{\small (a) Euclidean geometry}

\vspace{6pt}

\includegraphics[width=0.85\textwidth]
{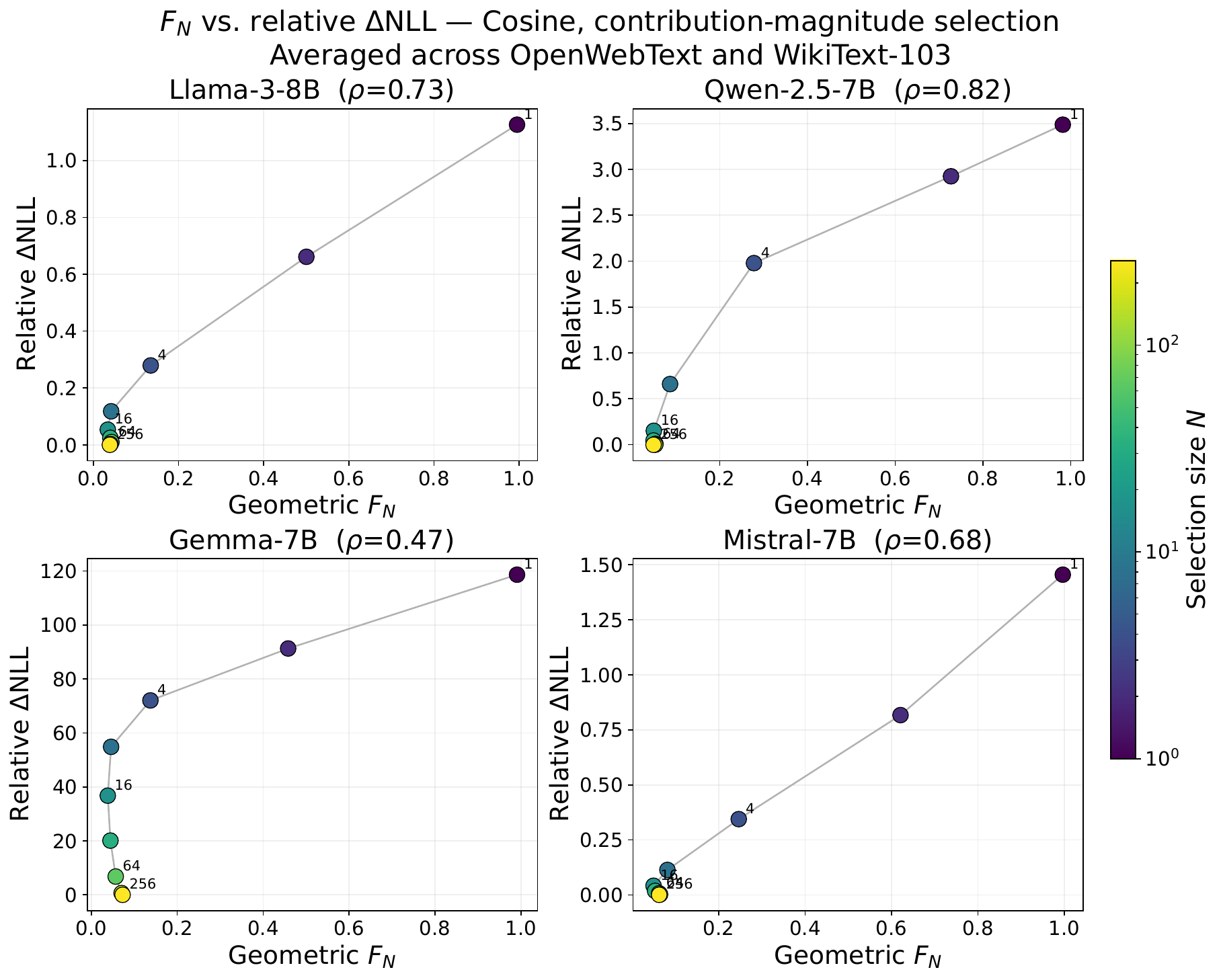}

{\small (b) Cosine geometry}

\caption{
\textbf{Contribution ranking gives similar geometry--loss trajectories.}
Each point corresponds to a selected-set size $N$, with geometric $F_N$
and relative NLL degradation averaged across OpenWebText and WikiText-103.
Color denotes $N$, lines connect successive sizes, and $\rho$ is the
Spearman correlation along each trajectory. Geometry is measured at the
final query, while loss uses interventions throughout the model.
}
\label{fig:geometry-nll-scatter-contribution}
\end{figure*}

\paragraph{Sensitivity to singleton and small-set boundaries.}
\label{app:geometry-boundary-check}
The Euclidean singleton identity in Appendix~\ref{app:singleton} motivates
checking whether small selected sets dominate the correlations in
Figure~\ref{fig:geometry-nll-scatter}. Table~\ref{tab:geo-boundary-correlations}
reports Spearman correlations after excluding these sizes. Geometric and
NLL summaries are matched by model, corpus, ranking, and $N$, then averaged
across corpora as in the figure. A positive correlation means that stronger
separation accompanies greater degradation.

Removing $N=1$ weakens most associations, and removing $N\le4$ can change
their sign. The correlations use 9, 8, or 6 selected-set sizes from the
same trajectories. This sensitivity supports the main text's qualified
interpretation of geometry as a description of selected sets, with
functional sufficiency established through the loss measurements.

\begin{table}[!htbp]
\centering\small
\caption{
Attention-ranking Spearman correlations between geometric $F_N$ and
relative NLL degradation after excluding small selected sets. The columns
use 9, 8, and 6 measured powers-of-two sizes, respectively.
}
\label{tab:geo-boundary-correlations}
\begin{tabular}{llrrr}
\toprule Model & Geometry & $N\ge1$ & $N\ge2$ & $N\ge8$\\\midrule
Qwen-2.5-7B & Euclidean & 0.817 & 0.738 & 0.371 \\
Qwen-2.5-7B & Cosine & 0.983 & 0.976 & 0.943 \\
Gemma-7B & Euclidean & 0.783 & 0.690 & 0.257 \\
Gemma-7B & Cosine & 0.433 & 0.190 & -0.943 \\
Llama-3-8B & Euclidean & 0.700 & 0.571 & -0.029 \\
Llama-3-8B & Cosine & 0.550 & 0.357 & -0.543 \\
Mistral-7B & Euclidean & 0.633 & 0.476 & -0.257 \\
Mistral-7B & Cosine & 0.717 & 0.595 & 0.029 \\
\bottomrule\end{tabular}
\end{table}

\subsection{Effective attention set size by corpus}
\label{app:nll-capacity-by-dataset}

Table~\ref{tab:effective-n-by-dataset} provides the corpus-specific
estimates underlying Section~\ref{sec:exp-lm-width}. It uses the
fixed-length protocol in Appendix~\ref{app:legacy-lm-protocol}, with
deletion at every query, head, and layer. Gemma-2-9B was evaluated only
at 5\% and 10\% tolerances, so its 1\% entries are unmeasured.

\begin{table*}[!htbp]
\caption{
Effective attention set-size estimates by corpus for 1024-position windows.
Deletion is applied at every query, head, and layer. Entries are locally
refined crossings at the stated relative NLL tolerance. Gemma-2-9B was
evaluated only at 5\% and 10\%, so its 1\% entries are unmeasured.
}
\label{tab:effective-n-by-dataset}
\centering
\small
\setlength{\tabcolsep}{4.5pt}
\begin{tabular}{llccc|ccc}
\toprule
& &
\multicolumn{3}{c}{OpenWebText} &
\multicolumn{3}{c}{WikiText-103} \\
\cmidrule(lr){3-5}
\cmidrule(lr){6-8}
Model & Ranking &
$N^*_{1\%}$ & $N^*_{5\%}$ & $N^*_{10\%}$ &
$N^*_{1\%}$ & $N^*_{5\%}$ & $N^*_{10\%}$ \\
\midrule

Qwen-2.5-1.5B & attention weight
& 99 & 39 & 24 & 108 & 42 & 26 \\
Qwen-2.5-1.5B & contribution magnitude
& 104 & 41 & 25 & 112 & 43 & 27 \\

Qwen-2.5-7B & attention weight
& 61 & 33 & 21 & 63 & 31 & 21 \\
Qwen-2.5-7B & contribution magnitude
& 59 & 31 & 21 & 63 & 30 & 20 \\

Gemma-7B & attention weight
& 283 & 226 & 206 & 253 & 195 & 173 \\
Gemma-7B & contribution magnitude
& 281 & 227 & 205 & 253 & 195 & 173 \\

Gemma-2-9B & attention weight
& -- & 12 & 7 & -- & 12 & 7 \\
Gemma-2-9B & contribution magnitude
& -- & 11 & 6 & -- & 11 & 7 \\

Llama-2-7B & attention weight
& 191 & 85 & 59 & 234 & 101 & 65 \\
Llama-2-7B & contribution magnitude
& 121 & 60 & 46 & 140 & 65 & 47 \\

Llama-3.2-1B & attention weight
& 61 & 18 & 10 & 65 & 19 & 10 \\
Llama-3.2-1B & contribution magnitude
& 60 & 17 & 9 & 64 & 19 & 10 \\

Llama-3-8B & attention weight
& 60 & 16 & 10 & 71 & 20 & 11 \\
Llama-3-8B & contribution magnitude
& 59 & 16 & 9 & 72 & 20 & 10 \\

Mistral-7B & attention weight
& 51 & 14 & 10 & 51 & 16 & 10 \\
Mistral-7B & contribution magnitude
& 51 & 14 & 9 & 51 & 15 & 9 \\

Mistral-Small-24B & attention weight
& 105 & 32 & 16 & 119 & 40 & 22 \\
Mistral-Small-24B & contribution magnitude
& 107 & 31 & 16 & 120 & 39 & 22 \\

\bottomrule
\end{tabular}
\end{table*}

The main cross-model differences occur on both corpora. Gemma-7B requires
the largest selected sets, while Mistral-7B and the Llama-3 checkpoints
reach the same tolerances with substantially smaller sets. Contribution
ranking gives its clearest reduction for Llama-2-7B. At 5\%, the estimate
decreases from 85 to 60 on OpenWebText and from 101 to 65 on WikiText-103.
The estimates describe a common set-size limit throughout the model.
Appendix~\ref{app:localization-extra} examines the variation in sensitivity
across individual layers and heads.

\subsection{Context-length dependence and scoring protocol}
\label{app:context-length-results}

The matched-target evaluation in Section~\ref{sec:context-length} asks how
much attention must be retained when additional context is available for
the same predictions. Figure~\ref{fig:context-width-app} and
Table~\ref{tab:context-width-tail128} report both corpora for the three
models with complete results. Across these six model--corpus pairs,
increasing $L$ eightfold increases the 5\% attention-ranked set-size
estimate by $4.12$--$6.63\times$. Tables~\ref{tab:context-tail128-all}
and \ref{tab:context-all_tokens-all} provide both rankings and all
tolerances under the two scoring protocols.

\begin{figure}[!htbp]
\centering
\begin{minipage}{0.495\linewidth}
\centering
\includegraphics[width=\linewidth]
{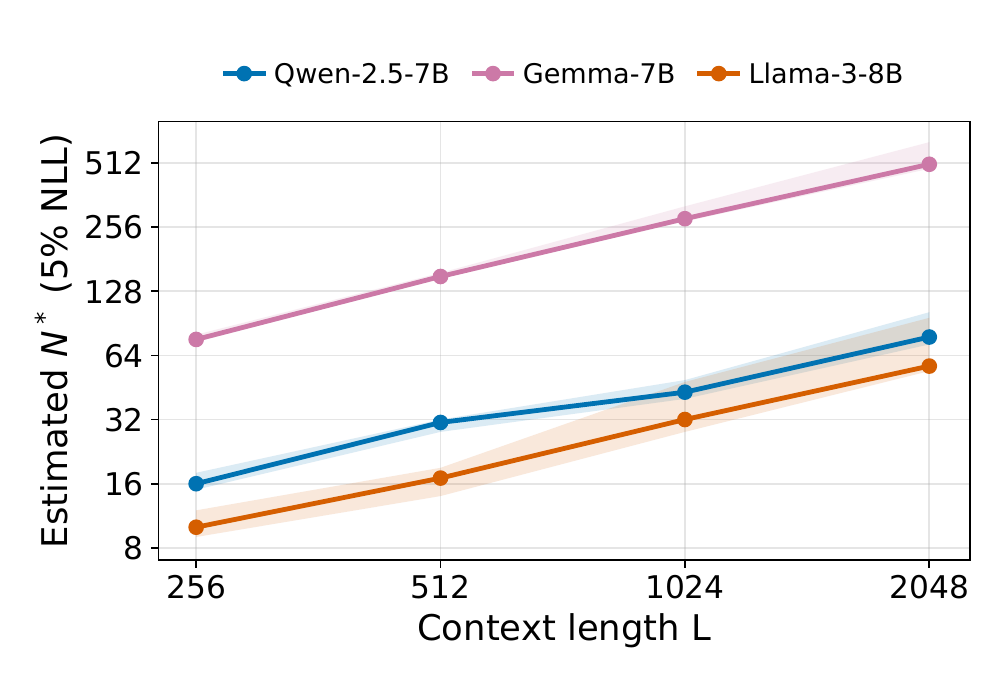}\\[-2pt]
{\small (a) OpenWebText}
\end{minipage}\hfill
\begin{minipage}{0.495\linewidth}
\centering
\includegraphics[width=\linewidth]
{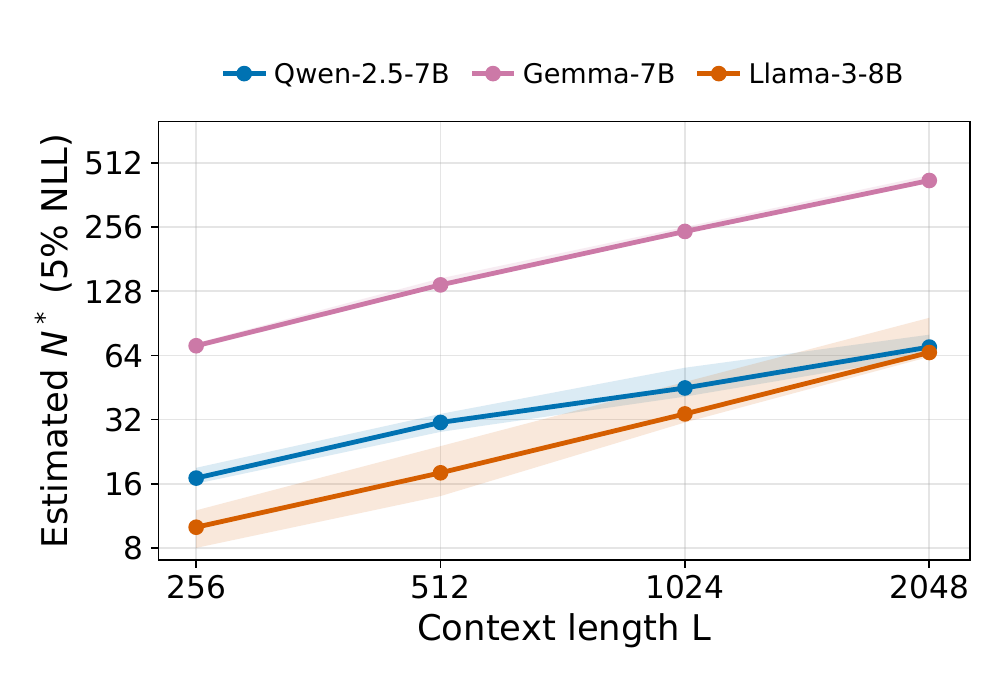}\\[-2pt]
{\small (b) WikiText-103}
\end{minipage}
\caption{
\textbf{Matched-target set-size estimates increase on both corpora.}
Effective attention set size at a 5\% relative NLL tolerance, using
50 documents per model and corpus. The same final 128 target tokens are
scored at every length within each model--corpus pair. Bands are 95\%
paired-document bootstrap intervals conditional on the evaluated sizes.
Both axes use logarithmic spacing.
}
\label{fig:context-width-app}
\end{figure}
\begin{table}[!htbp]
\centering\small
\caption{
Attention-ranked effective set-size estimates at a 5\% relative NLL
tolerance under matched final-128-target scoring. Column headings give
context length $L$. Only complete model--corpus pairs are included.
OWT denotes OpenWebText and WikiText denotes WikiText-103.
}
\label{tab:context-width-tail128}
\begin{tabular}{llrrrr}
\toprule
Model & Corpus & 256 & 512 & 1024 & 2048 \\
\midrule
Qwen-2.5-7B & OWT & 16 & 31 & 43 & 78 \\
Qwen-2.5-7B & WikiText & 17 & 31 & 45 & 70 \\
Gemma-7B & OWT & 76 & 150 & 280 & 504 \\
Gemma-7B & WikiText & 71 & 137 & 244 & 423 \\
Llama-3-8B & OWT & 10 & 17 & 32 & 57 \\
Llama-3-8B & WikiText & 10 & 18 & 34 & 66 \\
\bottomrule
\end{tabular}
\end{table}

\paragraph{Selected fraction and finite-range growth.}
Figure~\ref{fig:context-fraction} shows that $N^*/L$ decreases at the
5\% tolerance over the measured range. This behavior is compatible with
several growth patterns. For example, Llama-3-8B on WikiText-103 has
estimates $(10,18,34,66)$, which satisfy $N^*(L)=2+L/32$ exactly at the
four tested lengths. This affine fit also gives a decreasing fraction.

Descriptive log--log exponents at 5\% are $0.733/0.666$ for Qwen,
$0.909/0.856$ for Gemma, and $0.845/0.908$ for Llama, with OpenWebText
and WikiText-103 reported in that order. Bootstrap intervals conditional
on the evaluated sizes include one for Gemma/OpenWebText and for Llama
on both corpora. At 1\%, Llama/WikiText estimates $(24,57,121,241)$
give an exponent of approximately 1.107. The four measured lengths
therefore support context dependence over this range without establishing
a universal sublinear or asymptotic law.

\begin{figure}[!htbp]\centering
\begin{minipage}{0.49\linewidth}\centering
\includegraphics[width=\linewidth]{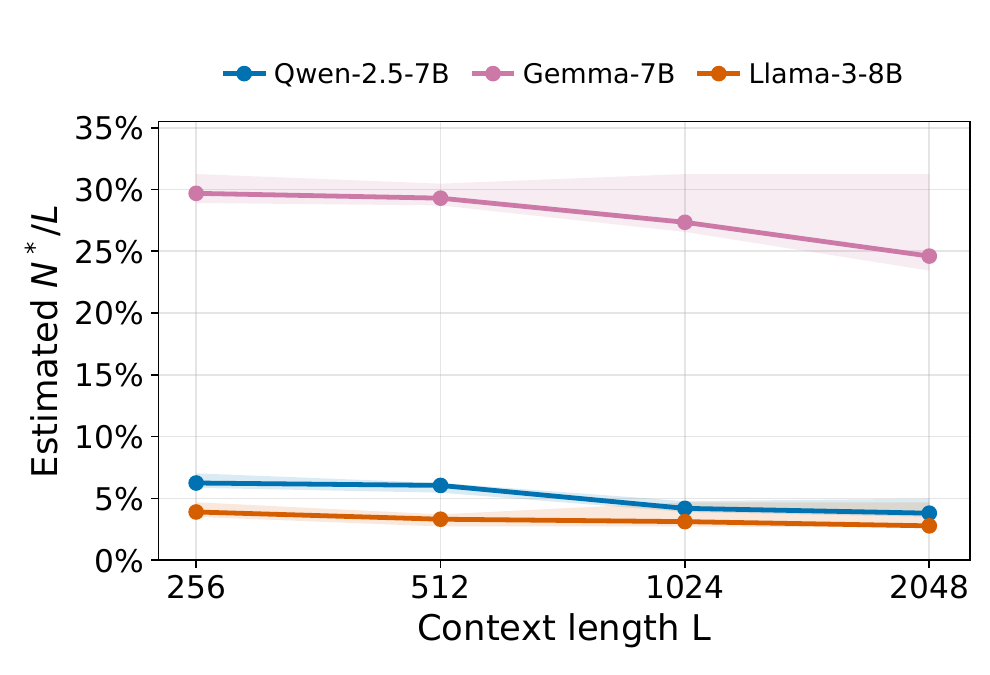}\\
{\small (a) OpenWebText}
\end{minipage}\hfill
\begin{minipage}{0.49\linewidth}\centering
\includegraphics[width=\linewidth]{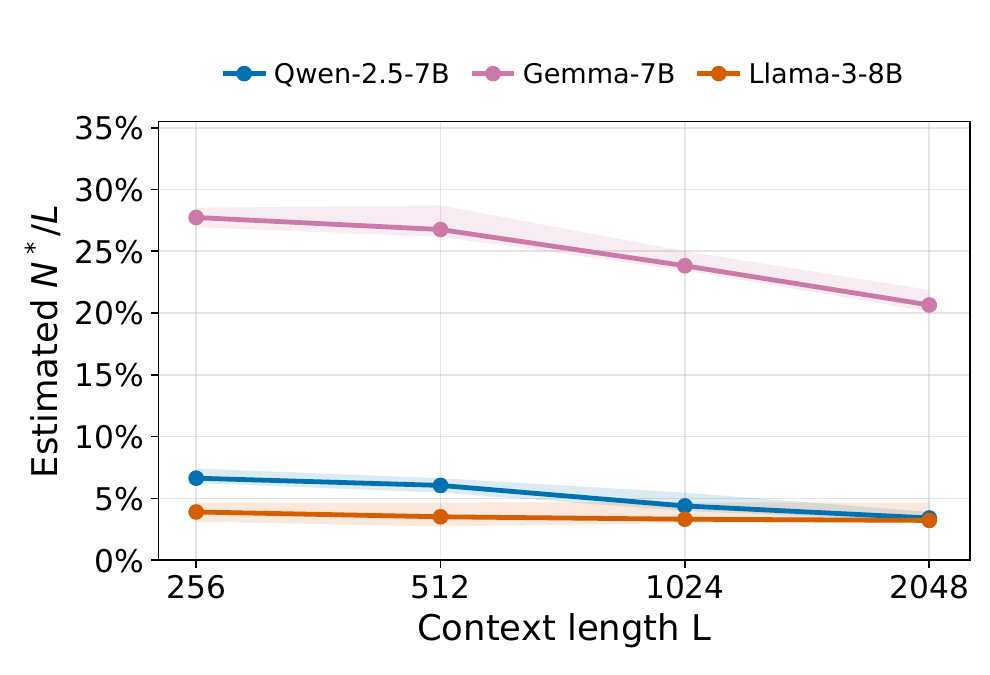}\\
{\small (b) WikiText-103}
\end{minipage}
\caption{
\textbf{The retained fraction decreases over the measured range.}
Attention-ranked $N^*/L$ at a 5\% relative NLL tolerance under matched
final-target scoring. Bands have the same conditional bootstrap
interpretation as in Figure~\ref{fig:context-width-app}. A decreasing
fraction is also compatible with affine growth having a positive intercept.
}
\label{fig:context-fraction}
\end{figure}

\paragraph{Uncertainty and local nonmonotonicity.}
At $L=2048$, the 5\% attention estimates and 95\% intervals on
OpenWebText are $78\,[72,102]$ for Qwen, $504\,[480,640]$ for Gemma,
and $57\,[54,96]$ for Llama. WikiText-103 gives $70\,[64,80]$,
$423\,[412,448]$, and $66\,[63,96]$, respectively. Gaps in the evaluated
sizes contribute to coarse upper endpoints.

Local nonmonotonicity also affects threshold interpretation. The smallest
tested passing size can differ from the smallest tested size for which
every larger tested size passes. These values are 45 and 47 for
Qwen/WikiText at $L=1024$, and 423 and 425 for Gemma/WikiText at
$L=2048$. Neither estimate resolves every unmeasured integer.

\paragraph{All-token scoring and contribution ranking.}
At $L=2048$, 5\% attention estimates for final-128-target and all-token
scoring are $78/51$, $504/394$, and $57/30$ on OpenWebText, and
$70/48$, $423/332$, and $66/36$ on WikiText-103, for Qwen, Gemma,
and Llama. All-token scoring includes earlier queries with shorter
prefixes, some of which retain every visible source under a given $N$.
It also changes the targets being averaged. Figure~\ref{fig:context-metric-comparison}
shows this protocol dependence, and Appendix~\ref{app:theory-target-positions}
gives a conditional explanation.

Contribution ranking preserves the context-length trend. Its OpenWebText
5\% estimates at $L=2048$ are 74, 503, and 56 for the three models.
These small differences from attention ranking do not establish a
consistent ranking advantage.

\begin{table}[!htbp]
\centering\small
\caption{
Matched final-128-target set-size estimates for the three complete models.
Each cell gives $N^*_{1\%}/N^*_{5\%}/N^*_{10\%}$, conditional on evaluated
sizes. Column headings give $L$. OWT denotes OpenWebText and WT denotes
WikiText-103. Attn. and Contr. denote attention and contribution ranking.
The sampling protocol differs from Table~\ref{tab:effective-n-by-dataset}.
}
\label{tab:context-tail128-all}
\setlength{\tabcolsep}{4pt}
\begin{tabular}{lllrrrr}
\toprule
Model & Corpus & Rank & 256 & 512 & 1024 & 2048 \\
\midrule
Qwen-2.5-7B & OWT & Attn. & 27/16/12 & 46/31/19 & 73/43/29 & 164/78/50 \\
Qwen-2.5-7B & OWT & Contr. & 26/16/11 & 47/28/19 & 71/42/29 & 158/74/49 \\
Qwen-2.5-7B & WT & Attn. & 28/17/13 & 59/31/22 & 88/45/31 & 137/70/49 \\
Qwen-2.5-7B & WT & Contr. & 27/16/12 & 57/31/21 & 84/45/30 & 141/69/49 \\
Gemma-7B & OWT & Attn. & 106/76/68 & 192/150/136 & 340/280/256 & 646/504/463 \\
Gemma-7B & OWT & Contr. & 105/76/68 & 190/150/136 & 346/279/256 & 635/503/461 \\
Gemma-7B & WT & Attn. & 95/71/63 & 180/137/120 & 390/244/222 & 641/423/387 \\
Gemma-7B & WT & Contr. & 94/71/64 & 183/137/120 & 383/245/221 & 650/423/388 \\
Llama-3-8B & OWT & Attn. & 30/10/6 & 56/17/10 & 107/32/17 & 188/57/30 \\
Llama-3-8B & OWT & Contr. & 29/10/6 & 55/16/9 & 106/31/16 & 185/56/29 \\
Llama-3-8B & WT & Attn. & 24/10/6 & 57/18/11 & 121/34/17 & 241/66/33 \\
Llama-3-8B & WT & Contr. & 24/9/6 & 57/18/10 & 120/34/17 & 243/66/33 \\
\bottomrule
\end{tabular}
\end{table}
\begin{table}[!htbp]
\centering\small
\caption{
All-token set-size estimates from the matched-suffix evaluation.
Each cell gives $N^*_{1\%}/N^*_{5\%}/N^*_{10\%}$, conditional on evaluated
sizes. Column headings give $L$. Abbreviations follow
Table~\ref{tab:context-tail128-all}. All predicted positions in each window
are scored, so the target population changes with $L$.
}
\label{tab:context-all_tokens-all}
\setlength{\tabcolsep}{4pt}
\begin{tabular}{lllrrrr}
\toprule
Model & Corpus & Rank & 256 & 512 & 1024 & 2048 \\
\midrule
Qwen-2.5-7B & OWT & Attn. & 22/13/9 & 35/20/14 & 55/31/21 & 102/51/34 \\
Qwen-2.5-7B & OWT & Contr. & 21/12/9 & 34/19/13 & 54/30/20 & 100/49/33 \\
Qwen-2.5-7B & WT & Attn. & 23/13/9 & 39/20/14 & 61/31/22 & 94/48/34 \\
Qwen-2.5-7B & WT & Contr. & 23/12/9 & 38/20/14 & 61/30/21 & 93/48/34 \\
Gemma-7B & OWT & Attn. & 89/67/59 & 157/125/111 & 275/227/204 & 483/394/358 \\
Gemma-7B & OWT & Contr. & 88/67/59 & 157/125/111 & 276/227/203 & 484/395/358 \\
Gemma-7B & WT & Attn. & 80/63/57 & 147/110/97 & 247/192/173 & 416/332/300 \\
Gemma-7B & WT & Contr. & 79/63/57 & 148/110/97 & 245/192/173 & 419/331/299 \\
Llama-3-8B & OWT & Attn. & 20/7/5 & 32/10/6 & 61/17/10 & 113/30/16 \\
Llama-3-8B & OWT & Contr. & 19/6/4 & 32/10/6 & 59/17/9 & 110/29/15 \\
Llama-3-8B & WT & Attn. & 18/7/5 & 36/11/7 & 76/19/11 & 140/36/18 \\
Llama-3-8B & WT & Contr. & 18/7/4 & 35/11/6 & 76/19/10 & 139/36/18 \\
\bottomrule
\end{tabular}
\end{table}
\begin{figure}[!htbp]\centering
\begin{minipage}{0.49\linewidth}\centering
\includegraphics[width=\linewidth]{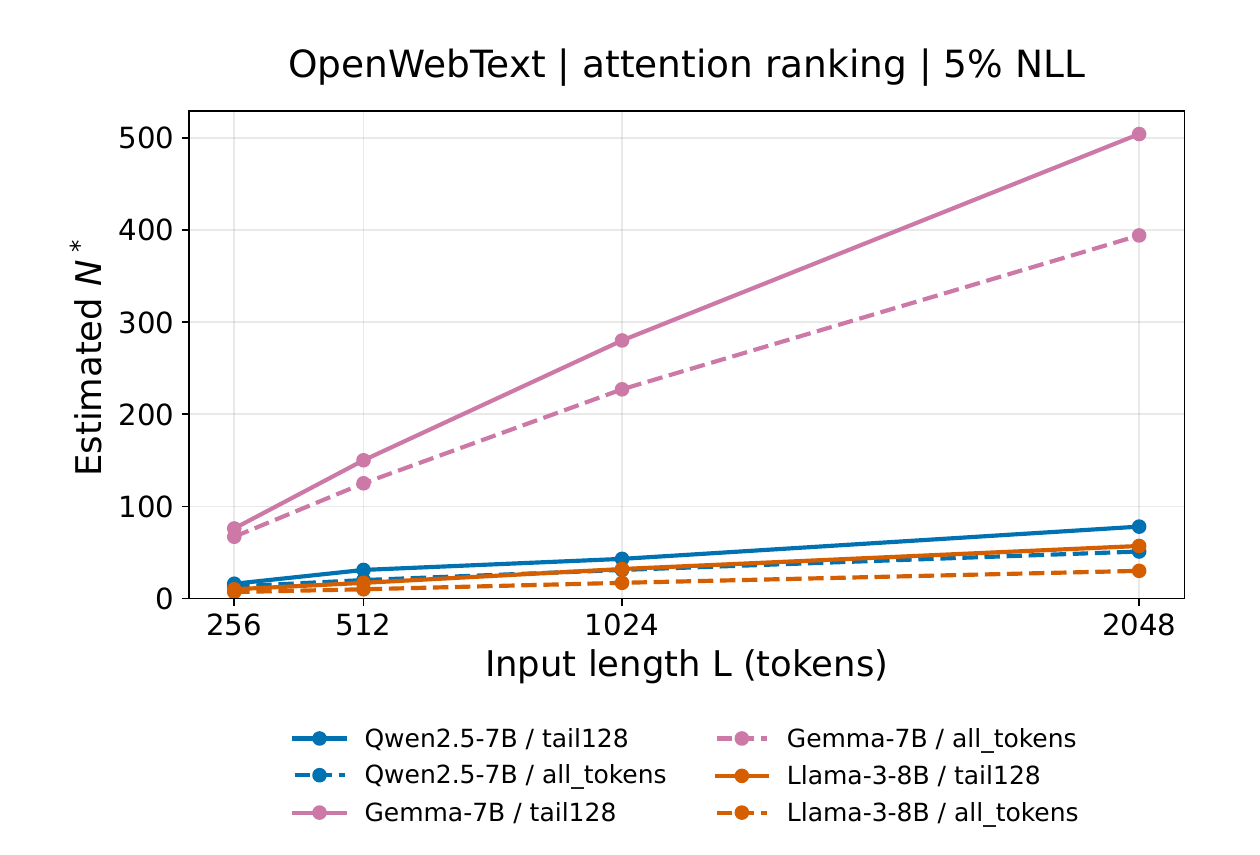}\\
{\small (a) OpenWebText}
\end{minipage}\hfill
\begin{minipage}{0.49\linewidth}\centering
\includegraphics[width=\linewidth]{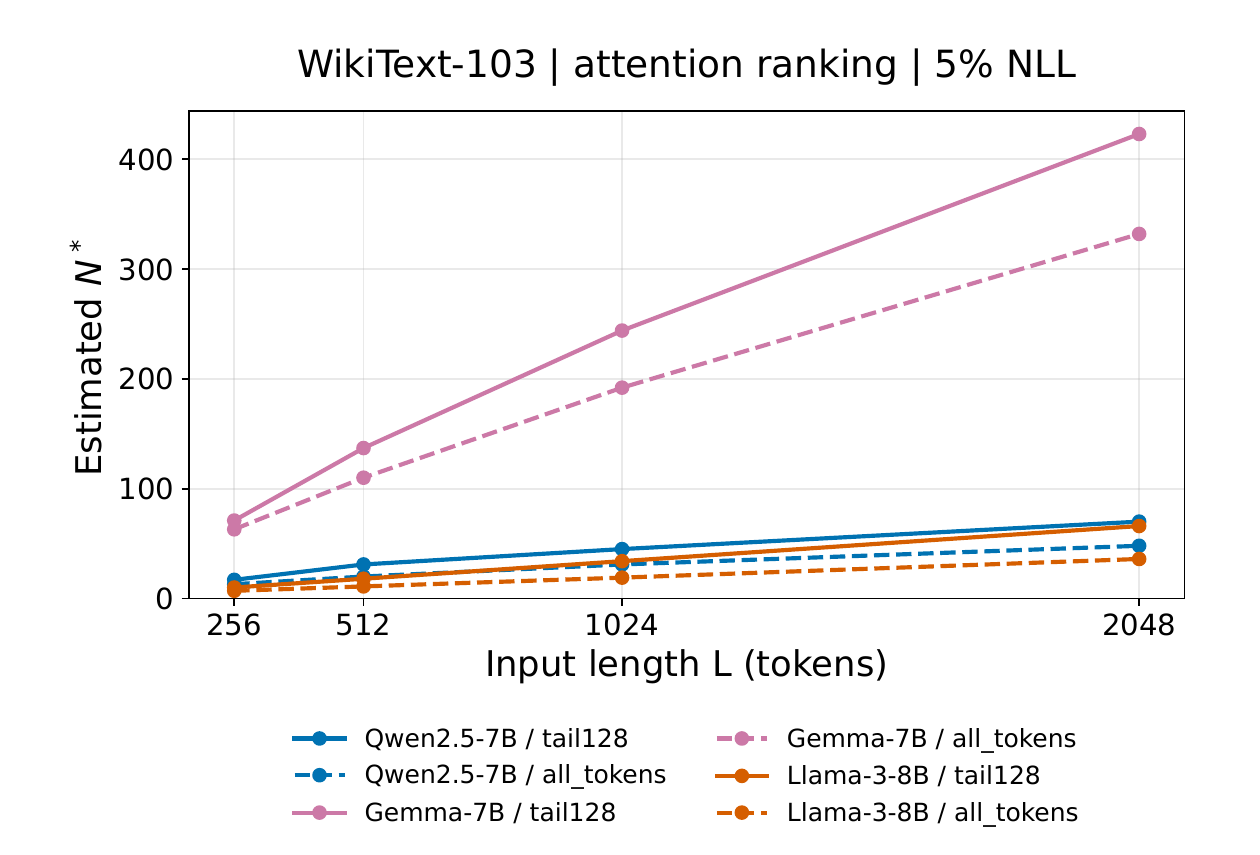}\\
{\small (b) WikiText-103}
\end{minipage}
\caption{
\textbf{Effective set size depends on which targets are scored.}
Matched final-target and all-token scoring use the same forward evaluations
but average different target populations. The final-target protocol fixes
the last 128 targets across lengths within each model--corpus pair.
}
\label{fig:context-metric-comparison}
\end{figure}

\paragraph{Full-model baseline and fixed absolute tolerance.}
Figure~\ref{fig:context-dense-nll} reports full-model NLL on the matched
final 128 targets. From $L=256$ to 2048, NLL decreases by
$0.292$--$0.342$ nats per token across the six complete pairs, with paired
95\% intervals for the change below zero. The additional natural context
therefore supplies useful predictive information, which may contribute to
the growing required set size.

An improving baseline also tightens a relative tolerance in absolute units.
To examine this effect, fix the allowed increase at
$0.05\,\NLL(M;L=256)$ for every length. The resulting OpenWebText
attention estimates still grow from 16 to 72 for Qwen, 76 to 496 for
Gemma, and 10 to 54 for Llama. This calculation uses already evaluated
sizes without additional refinement. It shows that relative-tolerance
tightening alone does not account for the growth reported in the main text.

\begin{figure}[!htbp]\centering
\begin{minipage}{0.49\linewidth}\centering
\includegraphics[width=\linewidth]{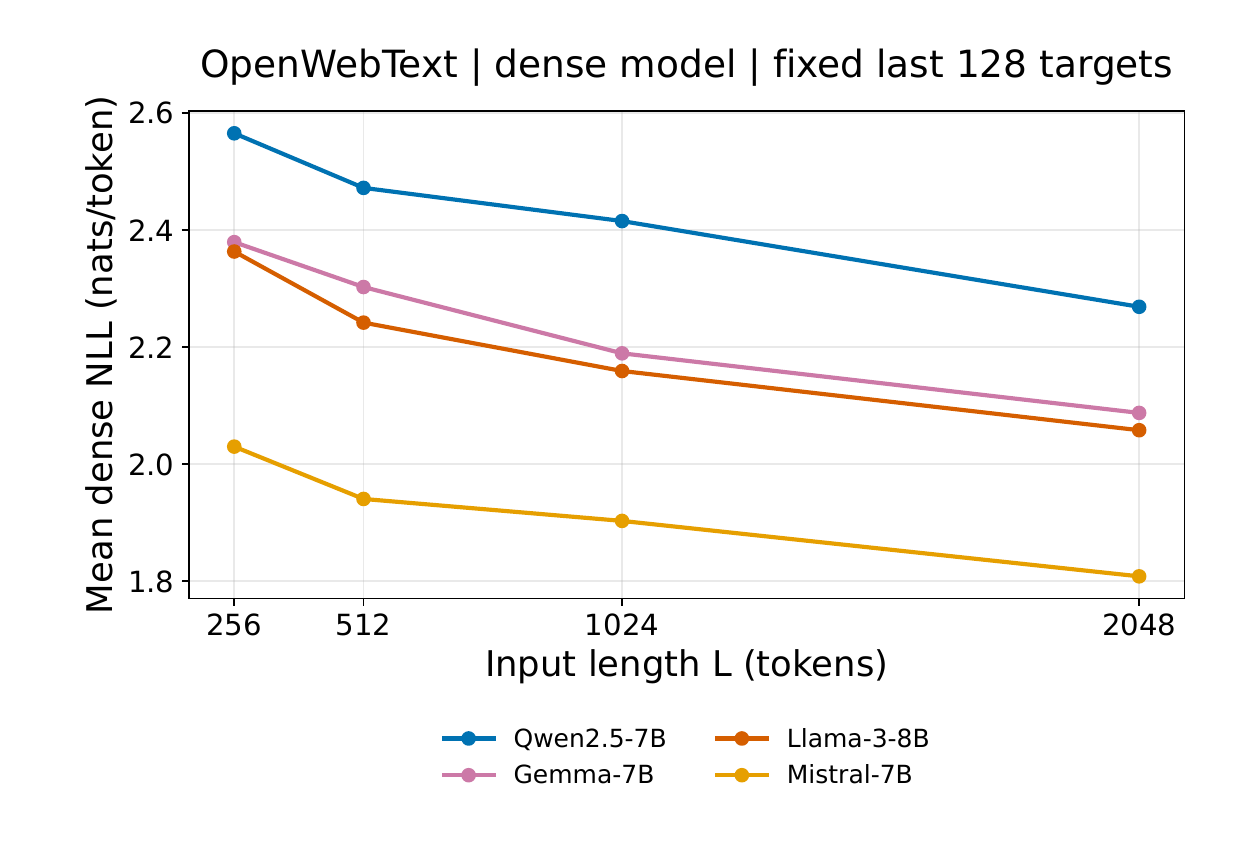}\\
{\small (a) OpenWebText}
\end{minipage}\hfill
\begin{minipage}{0.49\linewidth}\centering
\includegraphics[width=\linewidth]{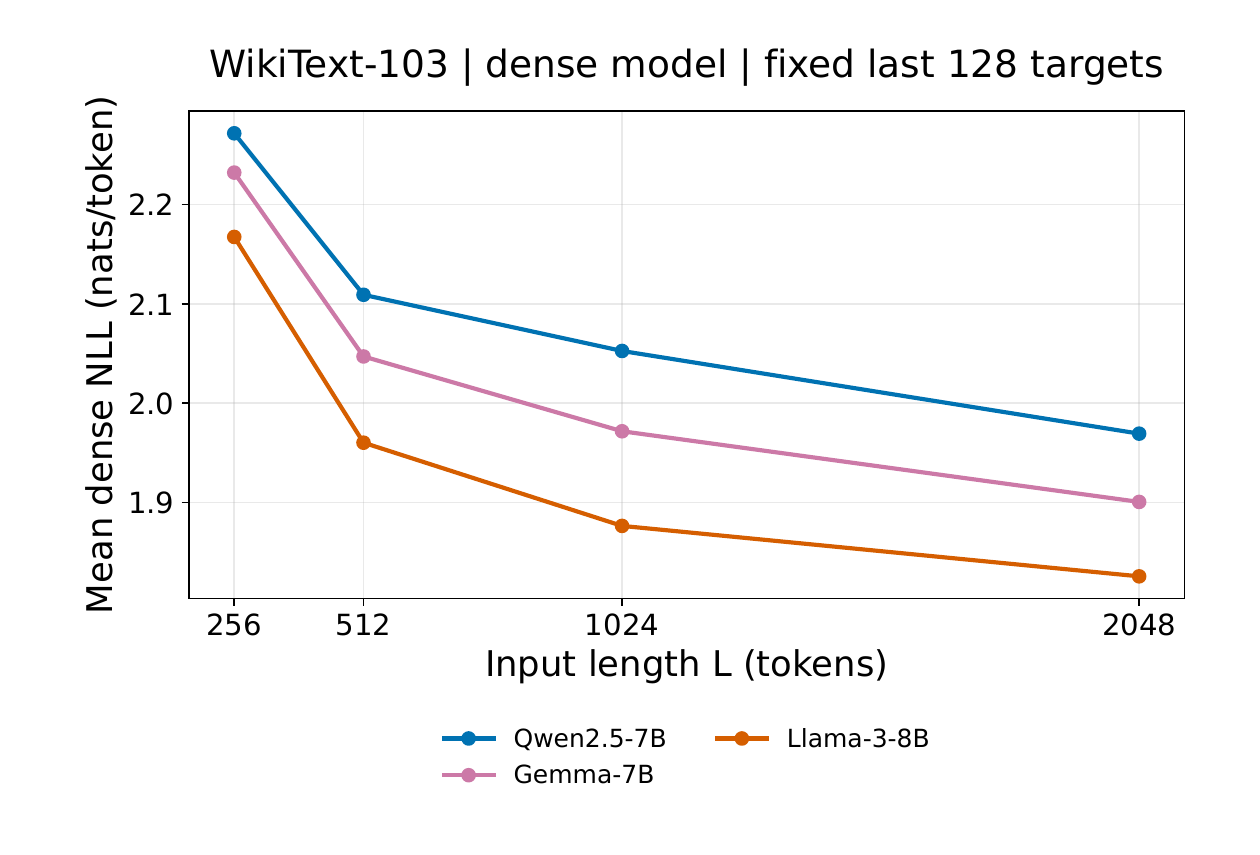}\\
{\small (b) WikiText-103}
\end{minipage}
\caption{
\textbf{Additional natural context improves full-model NLL.}
NLL on the same final 128 targets at each length. The Mistral/OpenWebText
curve includes full-model measurements from its partial $L=2048$ evaluation.
Those measurements do not imply a completed set-size estimate.
}
\label{fig:context-dense-nll}
\end{figure}

\paragraph{Partial Mistral results.}
Mistral-7B-v0.3 has OpenWebText 5\% attention estimates of 9, 13,
and 22 at $L=256,512,1024$. At $L=2048$, complete 50-document
measurements give $\Delta\NLL_{\rm rel}(49)=0.0519852$ and
$\Delta\NLL_{\rm rel}(50)=0.0498641$, yielding a provisional crossing
at 50. The full refinement is incomplete, contribution-ranking results
at this length are unavailable, and WikiText-103 was not completed.
This provisional value is excluded from the complete-model comparisons.

\subsection{Mistral-Small-24B results}
\label{app:mistral24}

The larger checkpoint extends the model coverage in
Section~\ref{sec:exp-lm-width}. Taking the maximum of the two
corpus-specific estimates gives $(119,40,22)$ under attention ranking
and $(120,39,22)$ under contribution ranking at tolerances
$(1\%,5\%,10\%)$. Figure~\ref{fig:mistral24-nll} shows the refined
loss curves, including local nonmonotonicity near the strict 1\% boundary.
The 8-bit weight configuration in Appendix~\ref{app:numerical-implementation}
means that this comparison does not isolate the effect of model size.

\begin{figure*}[!htbp]
\centering
\includegraphics[width=0.82\textwidth]
{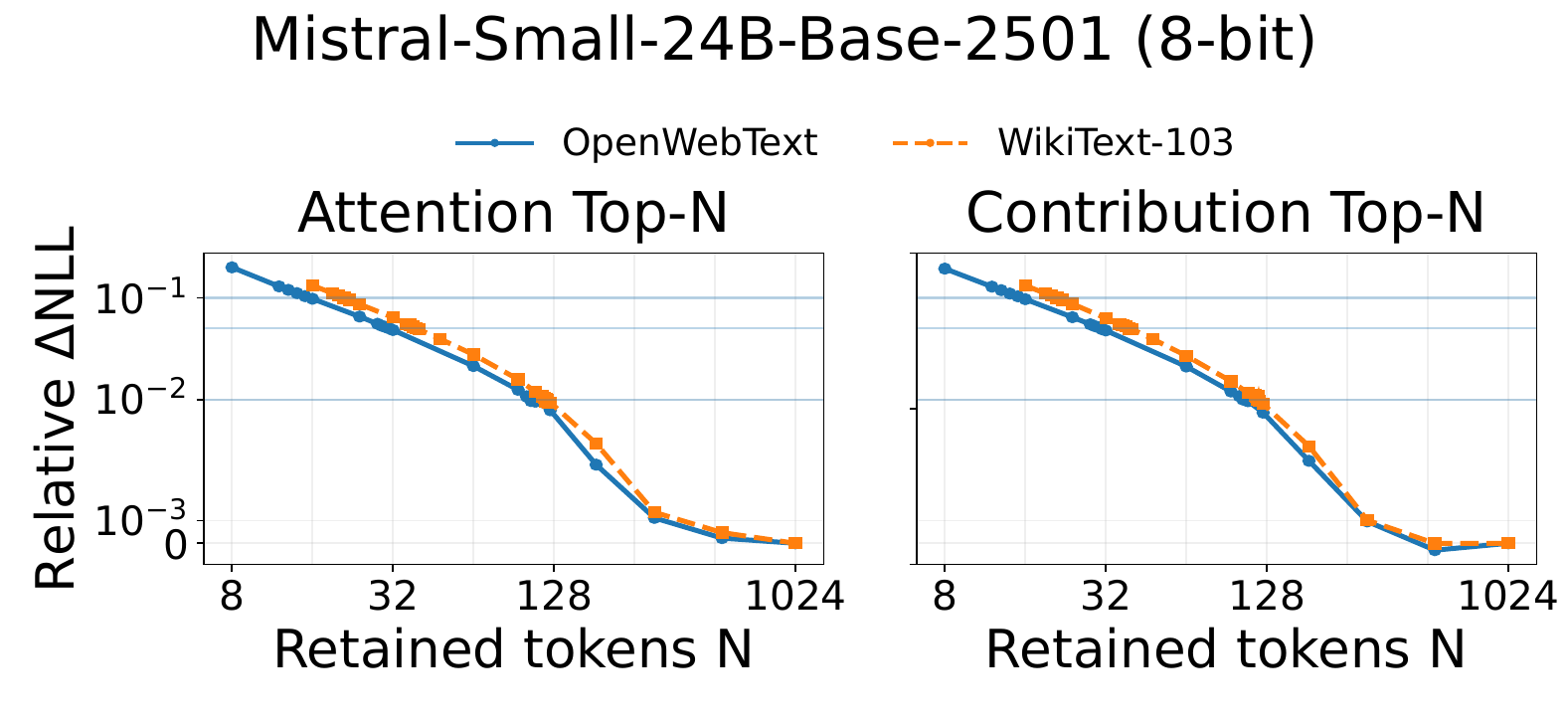}
\caption{
\textbf{Mistral-Small-24B retains low NLL degradation with restricted attention.}
Adaptive refinement of relative NLL degradation for the 8-bit-weight
checkpoint. Solid and dashed curves show OpenWebText and WikiText-103.
Horizontal lines indicate 1\%, 5\%, and 10\% tolerances.
}
\label{fig:mistral24-nll}
\end{figure*}

\section{Additional controlled-retrieval results}
\label{app:babilong-full}

Section~\ref{sec:competition} uses fixed annotated support to investigate
competition separately from the additional predictive information supplied
by natural context. We report the detailed support statistics, extended
answer-loss thresholds, and full-model baselines. A supplementary
multi-fact comparison examines the limits of interpreting these thresholds
as evidence about the useful-set size.

\subsection{Support displacement, mass, and recall}
\label{app:support-details}

The measurements in Figure~\ref{fig:e1-mechanism} describe how the same
annotated supporting fact is represented in the attention ranking as
background grows. They use the 86 validated examples per condition from
Appendix~\ref{app:qa-scoring}, after unambiguous support mapping identifies
90 of the 100 sampled examples.

Averaging the per-model 4K/0K ratios across Qwen-2.5-7B, Gemma-7B,
Llama-3-8B, and Mistral-7B gives a $50.0\times$ increase in mean
support displacement and a reduction of support attention mass to
$0.435\times$ its initial value. Mean support recall@64 falls by
63.8 percentage points. Displacement ratios range from $42.8\times$ to
$54.6\times$ across models. Log--log fits against actual non-support
counts give slopes of 0.79--0.84 and $R^2\ge0.998$, describing the
four measured background conditions.

\paragraph{Pairwise ranking and the number of competitors.}
The mean non-support count increases by approximately $108$--$111\times$
from 0K to 4K. Over the same range, mean empirical outranking probability
decreases from 0.475 to 0.224 for Qwen, 0.451 to 0.226 for Gemma,
0.443 to 0.172 for Llama, and 0.451 to 0.205 for Mistral
(Figure~\ref{fig:support-outrank-probability}). Thus more competitors
are ranked above support overall even though an individual competitor
becomes less likely to outrank a support token.

For each layer/head measurement,
$\operatorname{Disp}_a(x)=d_x\widehat p_a(x)$ by definition.
This identity requires no independence assumption. Averaging yields
$\mathbb E[d_x\widehat p_a(x)]$, which need not factor into the product
of the separate means. The observed probabilities also vary with background,
so a constant outranking coefficient does not describe these measurements.

Mean displacement averages over support tokens. Complete support recovery
depends on the lowest-ranked support token, while the loss criterion also
depends on retained weights and subsequent computation. This distinction
helps interpret the different model responses. Qwen's displacement grows
$51.6\times$ while its 0.10-nat functional set-size estimate doubles.
Gemma's displacement grows $54.6\times$ while its estimate grows
approximately $11.9\times$. As a descriptive association, full-model
correct examples have greater support mass in all 16 model--background
conditions and greater support recall@64 in 14 of them. These comparisons
motivate the functional test while leaving the causal role of individual
support measurements unresolved.

\begin{figure}[!htbp]\centering
\includegraphics[width=0.76\linewidth]{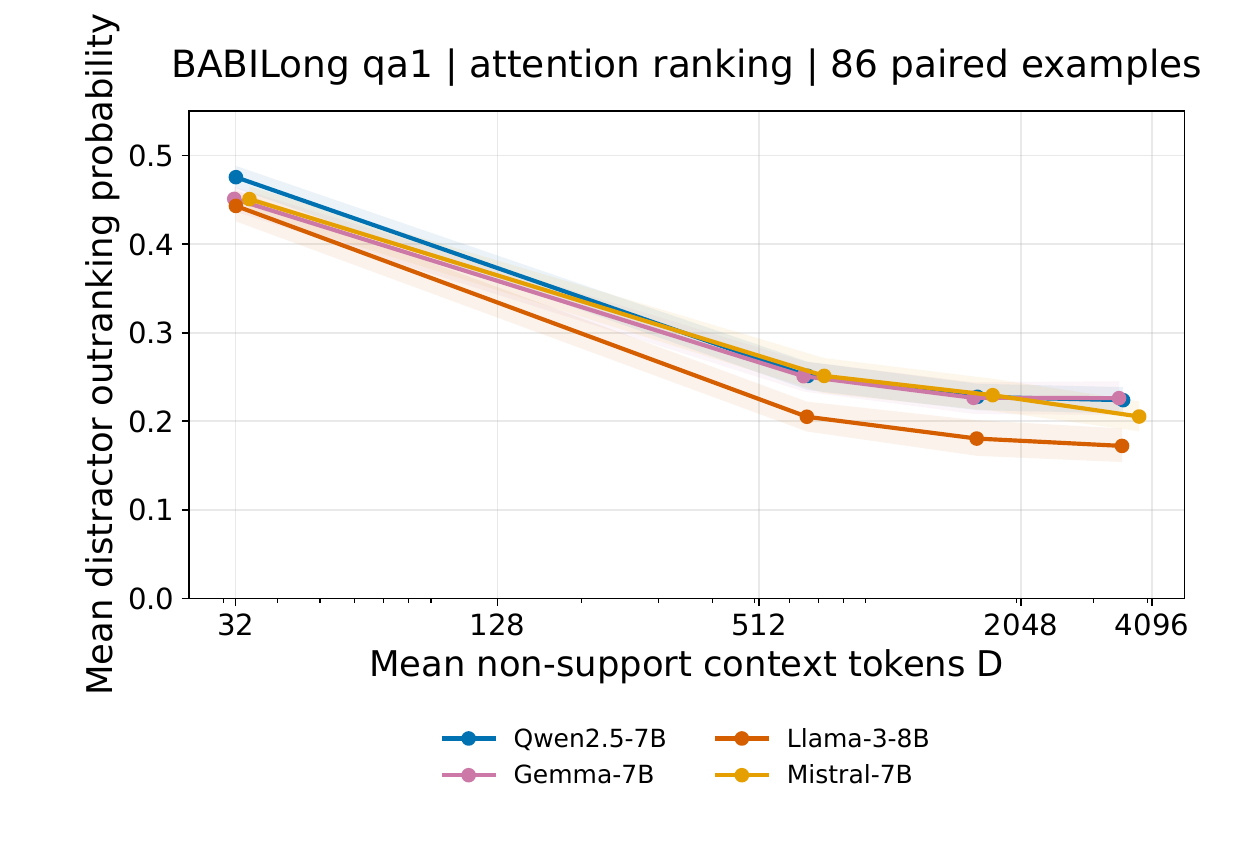}
\caption{
\textbf{Pairwise support ranking improves while total displacement grows.}
Mean empirical probability that a non-support context token scores at least
as highly as a support token at the answer query. Background growth
increases the number of competitors while this probability decreases.
The 0K condition already includes non-support tokens.
}
\label{fig:support-outrank-probability}
\end{figure}

\subsection{Extended answer-loss thresholds}
\label{app:answerloss-details}

Table~\ref{tab:qa-answerloss-widths} extends the initial search summarized
in Figure~\ref{fig:qa-distractors-main}. It reports attention-ranked
effective set sizes at a 0.10-nat mean candidate-answer-loss tolerance,
measured relative to the full model at each background.

\begin{table}[!htbp]
\caption{
Attention-ranked BABILong \texttt{qa1} set-size estimates at a 0.10-nat
mean candidate-answer-loss tolerance relative to each background's full
model. Adaptive extensions are included. An entry $>256$ denotes no
passing evaluated size through 256, with untested integers unresolved.
}
\label{tab:qa-answerloss-widths}
\centering
\small
\begin{tabular}{lrrrr}
\toprule
Model & 0K & 1K & 2K & 4K \\
\midrule
Qwen-2.5-7B & 32 & 16 & 32 & 64 \\
Llama-3-8B  & 32 & 128 & 256 & 316 \\
Gemma-7B    & 64 & 256 & $>256$ & 760 \\
Mistral-7B  & 8  & 64  & 256 & $>256$ \\
\bottomrule
\end{tabular}
\end{table}

Adaptive extensions resolve crossings at $N=316$ for Llama-3-8B and
$N=760$ for Gemma-7B at 4K. Gemma at 2K and Mistral at 4K remain
unresolved on the evaluated sizes through 256. The $>256$ notation
records this search limit and does not certify failure at every smaller
untested integer. Qwen's estimates of $(32,16,32,64)$ across the four
backgrounds show that support displacement need not produce a monotone
functional response. The main text accordingly reports increasing
requirements in several models over the tested range.

\subsection{Full-model performance across backgrounds}
\label{app:dense-qa-background}

The effective set size measures loss preservation relative to a full-model
baseline. To interpret this criterion, Table~\ref{tab:qa-dense-loss} and
Figure~\ref{fig:dense-qa-background} report that baseline on the same
100-example cohorts used for the answer-loss curves. The evaluation is
six-way candidate choice, with uniform candidate loss
$\log6\simeq1.792$ nats.

\begin{table}[!htbp]\centering\small
\caption{
Full-model \texttt{qa1} candidate-normalized answer loss in nats on
100 matched examples per background. These baselines are used for the
corresponding answer-loss degradation measurements.
}
\label{tab:qa-dense-loss}
\begin{tabular}{lrrrr}\toprule
Model & 0K & 1K & 2K & 4K \\\midrule
Qwen-2.5-7B & 0.690 & 1.178 & 1.303 & 1.385 \\
Gemma-7B & 0.647 & 0.882 & 0.933 & 1.150 \\
Llama-3-8B & 0.740 & 0.773 & 0.857 & 1.163 \\
Mistral-7B-v0.3 & 0.792 & 1.018 & 1.151 & 1.311 \\
\bottomrule\end{tabular}\end{table}

Accuracy changes from 0K to 4K are $-23$, $-18$, $-10$, and $-14$
percentage points for Qwen, Gemma, Llama, and Mistral. Llama improves
at 1K, and its paired 0K-to-4K accuracy-change interval is approximately
$[-22,+2]$ percentage points, which includes zero. Mean candidate loss
increases from 0K to 4K for all four models, with paired loss-change
intervals above zero. This supports the main text's loss-based account of
competition while qualifying the accuracy trend. A small effective set
size at a difficult background condition can preserve a baseline that
is itself less accurate.

\begin{figure}[!htbp]\centering
\begin{minipage}{0.49\linewidth}\centering
\includegraphics[width=\linewidth]{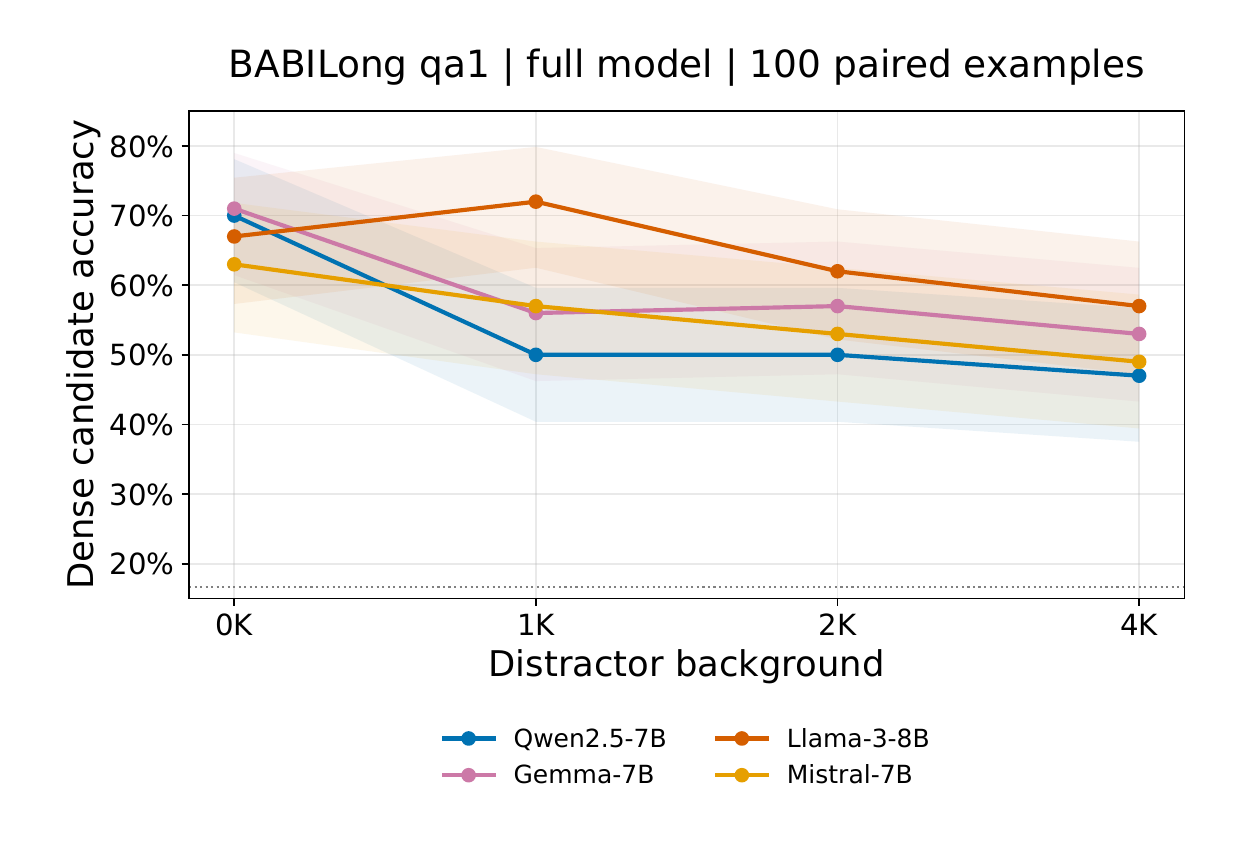}\\
{\small (a) Dense candidate accuracy}
\end{minipage}\hfill
\begin{minipage}{0.49\linewidth}\centering
\includegraphics[width=\linewidth]{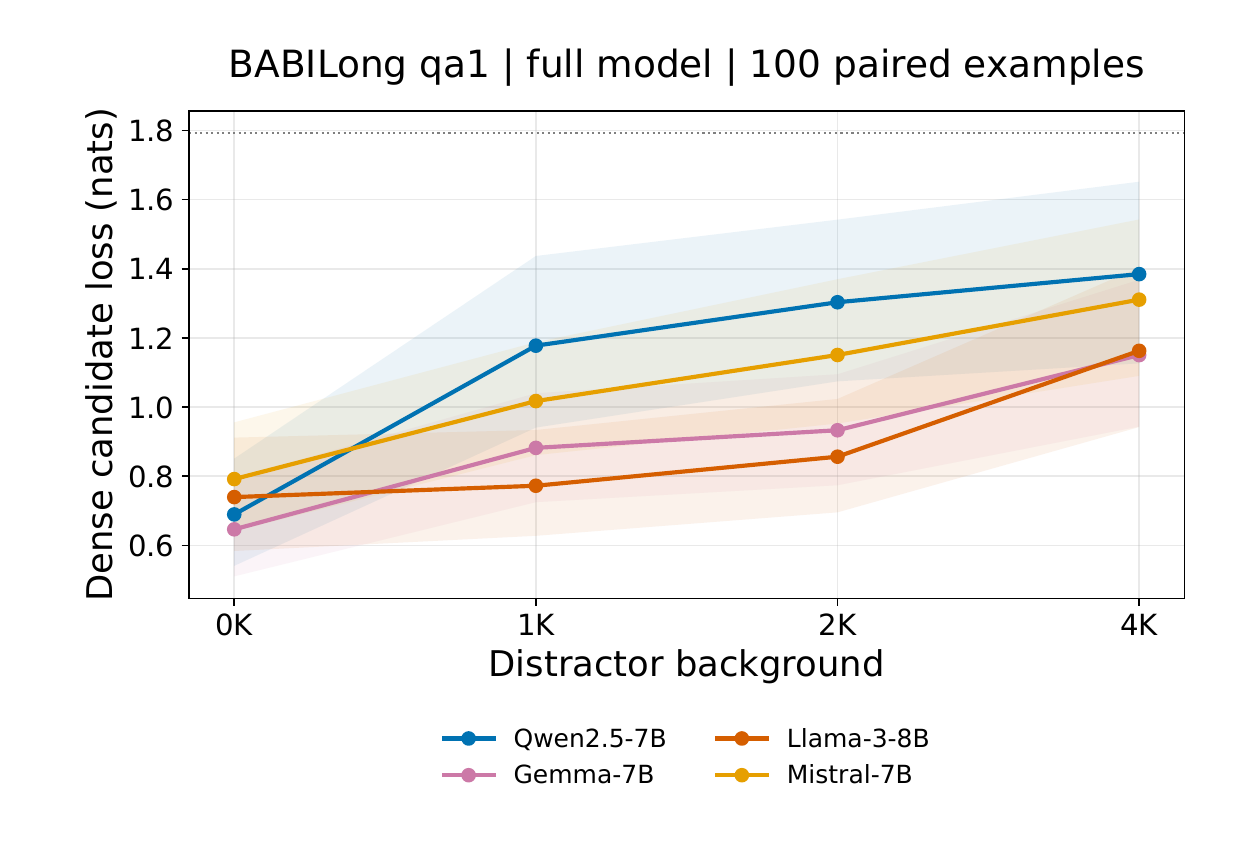}\\
{\small (b) Dense candidate loss}
\end{minipage}
\caption{
\textbf{Full-model answer loss increases with background.}
Candidate accuracy and mean candidate-normalized answer loss before
intervention. Accuracy bands are 95\% Wilson intervals, and loss bands
use bootstrap resampling over examples. Dotted references indicate uniform
six-candidate performance. Accuracy can vary nonmonotonically despite
increasing loss.
}
\label{fig:dense-qa-background}
\end{figure}

\subsection{Fact multiplicity and baseline competence}
\label{app:qa-full-accuracy}

The useful-token framework allows $K(C)$ to vary with the computation.
As a supplementary comparison, \texttt{qa1}, \texttt{qa2}, and
\texttt{qa3} vary the number of annotated supporting facts at 0K
background. Table~\ref{tab:babilong-full-accuracy} and
Figure~\ref{fig:qa-facts-app} use the same candidate-scoring protocol
and 100-example cohorts for each model and task.

\begin{table}[!htbp]
\centering\small
\caption{
Full-model candidate accuracy at 0K background, using 100 examples per
task and model. These are the baselines for the loss curves in
Figure~\ref{fig:qa-facts-app}.
}
\label{tab:babilong-full-accuracy}
\begin{tabular}{lccc}\toprule
Model & \texttt{qa1} & \texttt{qa2} & \texttt{qa3}\\\midrule
Qwen-2.5-7B & 0.70 & 0.49 & 0.28\\
Gemma-7B & 0.71 & 0.49 & 0.30\\
Llama-3-8B & 0.67 & 0.49 & 0.26\\
Mistral-7B-v0.3 & 0.63 & 0.52 & 0.23\\
\bottomrule\end{tabular}
\end{table}

Full-model \texttt{qa3} accuracy is 0.23--0.30, compared with chance
accuracy $1/6$. This limits how much net accuracy can decrease under
intervention and makes a small accuracy-based set-size estimate difficult
to interpret as successful multi-fact retrieval. Candidate loss remains
sensitive to score changes, but its preservation still refers to this
baseline competence. Moreover, an annotated fact spans several tokens
and need not specify every useful token. The comparison therefore
examines sensitivity to task structure without identifying $K(C)$.

\begin{figure*}[!htbp]
\centering
\begin{minipage}{0.49\textwidth}\centering
\includegraphics[width=\linewidth]{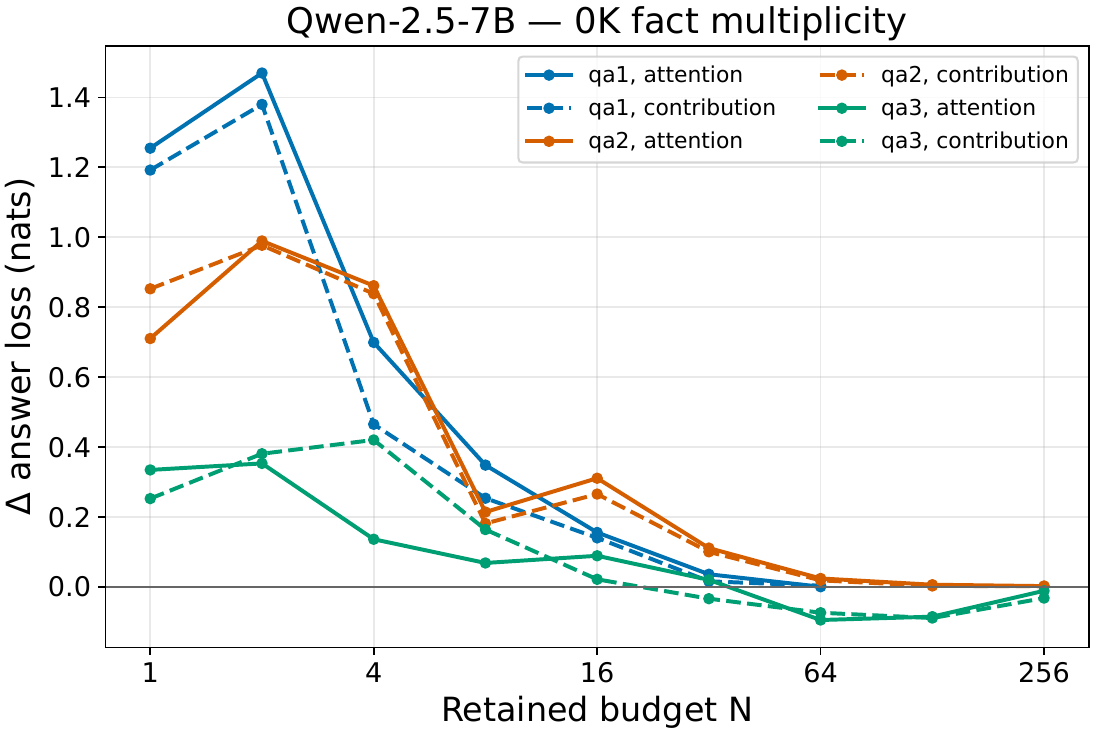}\\
{\small (a) Qwen-2.5-7B}
\end{minipage}\hfill
\begin{minipage}{0.49\textwidth}\centering
\includegraphics[width=\linewidth]{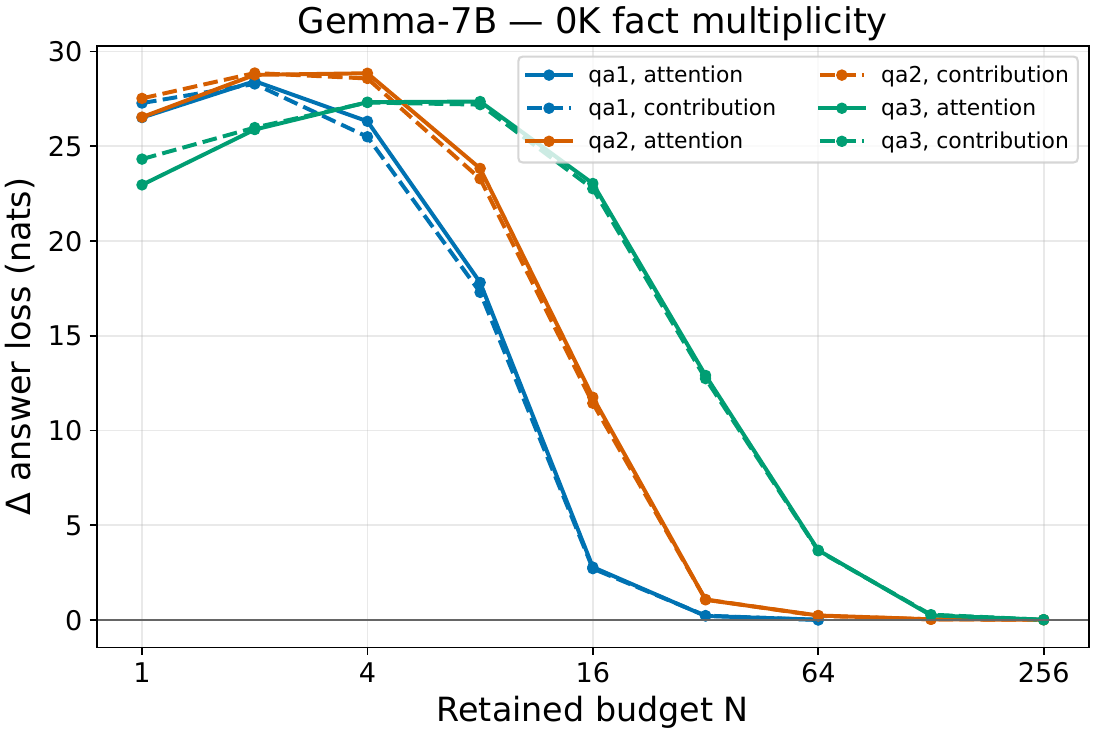}\\
{\small (b) Gemma-7B}
\end{minipage}\\[5pt]
\begin{minipage}{0.49\textwidth}\centering
\includegraphics[width=\linewidth]{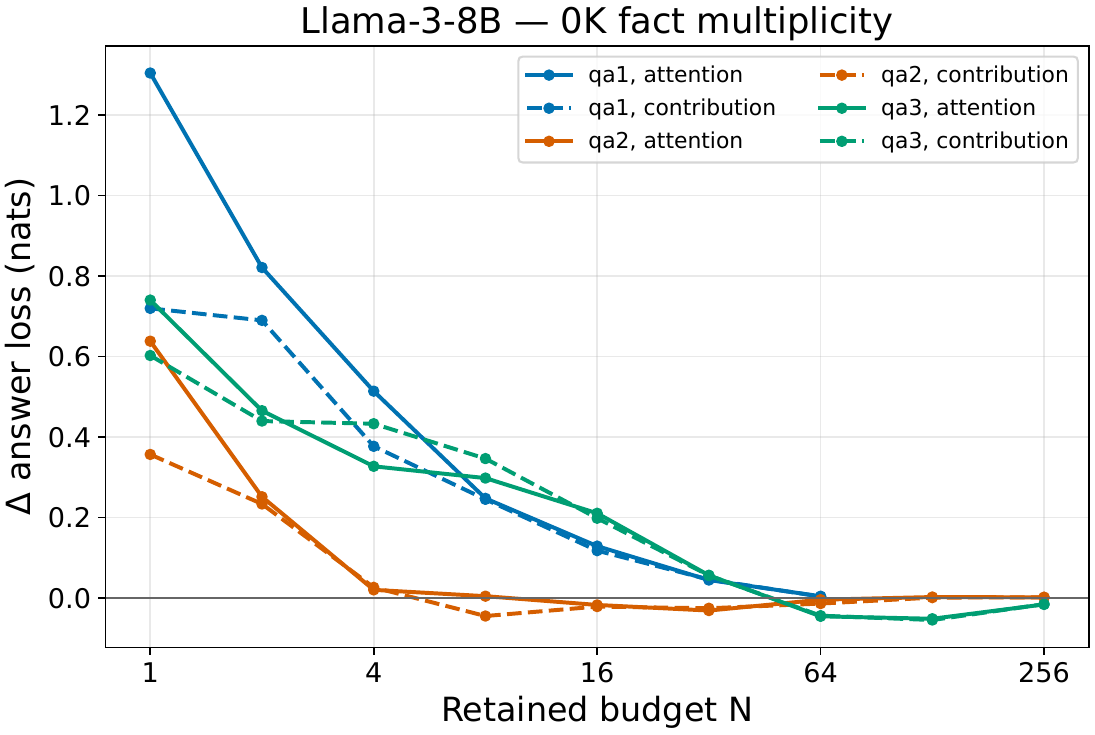}\\
{\small (c) Llama-3-8B}
\end{minipage}\hfill
\begin{minipage}{0.49\textwidth}\centering
\includegraphics[width=\linewidth]{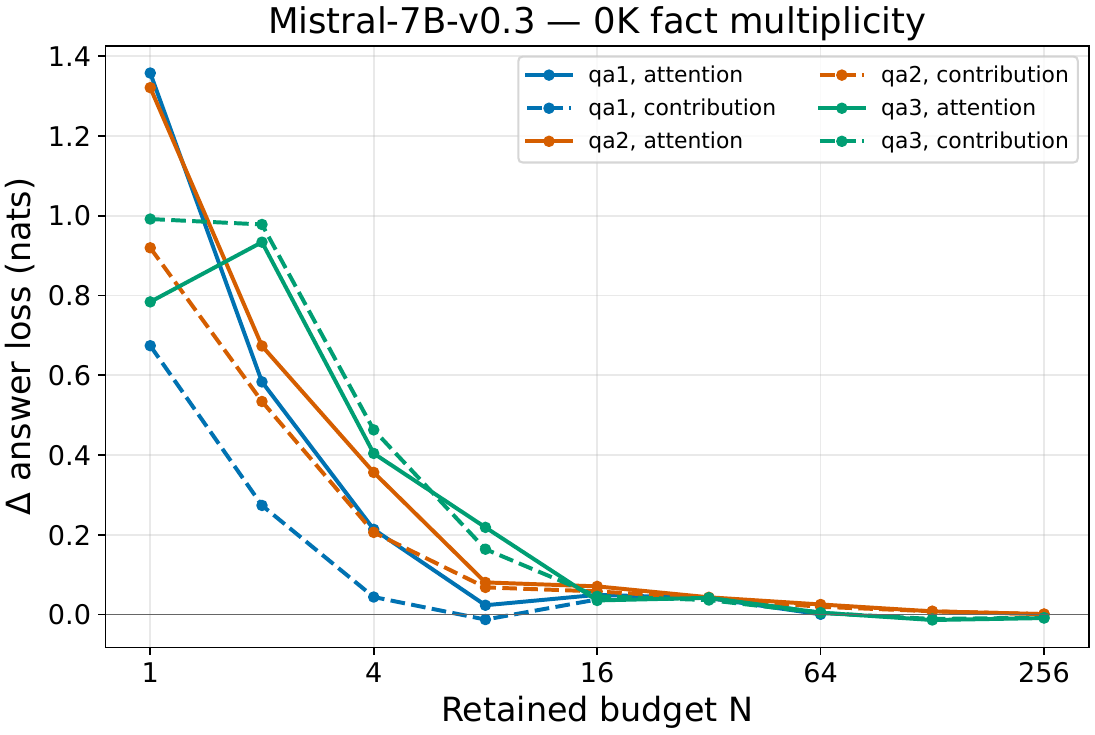}\\
{\small (d) Mistral-7B-v0.3}
\end{minipage}
\caption{
\textbf{Multi-fact loss sensitivity must be interpreted with baseline competence.}
Mean candidate-answer-loss degradation at 0K background. Color denotes
task, and solid and dashed curves denote attention and contribution
ranking. Each curve includes the full 100-example cohort, with zero
degradation when $N$ is at least the prompt length. Full-model accuracies
are given in Table~\ref{tab:babilong-full-accuracy}.
}
\label{fig:qa-facts-app}
\end{figure*}

\section{Aggregation controls and functional diagnostics}
\label{app:supp-controls}

The main text interprets effective attention set size through ranking,
aggregation, and sensitivity to discarded contributions. The controls below
examine these factors by rescaling retained weights, comparing related
checkpoints, localizing interventions, and testing functional predictors
and positional selection.

\subsection{Renormalization of selected weights}
\label{app:renorm-details}
\label{sec:intervention-result}

Section~\ref{sec:aggregation} asks whether preserving the original attention
weights contributes to the measured set-size requirement. The renormalized
control keeps the relative weights of selected tokens and restores their
sum to one:
\begin{equation}
\widehat\alpha_i^{(r,N)}=
\frac{\alpha_i m_i^{(r,N)}}{\sum_j\alpha_jm_j^{(r,N)}}.
\label{eq:renorm}
\end{equation}
Comparing this intervention with deletion evaluates how rescaling the
selected sum affects functional sufficiency. The corresponding local error
identities are given in Equation~\ref{eq:renorm-centroids}.

\paragraph{Language-model estimates.}
Table~\ref{tab:renorm-lm-absolute} compares OpenWebText estimates for both
rankings at 5\% and 10\% relative NLL tolerances. The renormalized
evaluation contains 50 documents per model. Deletion estimates use the
OpenWebText entries in Table~\ref{tab:effective-n-by-dataset} to keep the
corpus consistent. Per-document deletion traces are unavailable for
verifying that the two historical evaluations used identical windows and
numerical settings. These are therefore comparisons of corpus-level
thresholds, with document pairing unverified.

\begin{table}[!htbp]
\centering\small
\caption{
OpenWebText effective set-size estimates under deletion (Primary) and
renormalization (Renorm) at 5\% and 10\% relative NLL tolerances.
Both rankings are shown. Historical document pairing is unverified, so
the comparison is at the corpus-summary level.
}
\label{tab:renorm-lm-absolute}
\begin{tabular}{llrrrr}
\toprule &&\multicolumn{2}{c}{5\%}&\multicolumn{2}{c}{10\%}\\
\cmidrule(lr){3-4}\cmidrule(lr){5-6}
Model & Ranking & Primary & Renorm & Primary & Renorm\\\midrule
Qwen-2.5-7B & attention & 33 & 15 & 21 & 9 \\
Qwen-2.5-7B & contribution & 31 & 24 & 21 & 13 \\
Gemma-7B & attention & 226 & 7 & 206 & 4 \\
Gemma-7B & contribution & 227 & 11 & 205 & 7 \\
Mistral-7B & attention & 14 & 8 & 10 & 5 \\
Mistral-7B & contribution & 14 & 14 & 9 & 11 \\
\bottomrule\end{tabular}
\end{table}

The largest change occurs for Gemma-7B, whose 5\% attention-ranked
estimate decreases from 226 to 7. The effect depends on the model, ranking,
and tolerance. For Mistral contribution ranking at 10\%, renormalization
increases the estimate from 9 to 11. Figure~\ref{fig:e8-renorm}a reports
the corresponding ratios. This variation is consistent with the local
identities, which allow rescaling to either reduce or increase output
error. The comparisons show intervention dependence without identifying
the deleted positions as intrinsically irrelevant.

\paragraph{Controlled QA estimates.}
The QA normalization control covers 0K and 4K background for Qwen,
Gemma, and Mistral. Table~\ref{tab:renorm-qa-absolute} reports both
rankings at 0.10- and 0.20-nat mean answer-loss tolerances.
At 0.10 nats under attention ranking, the 4K/0K set-size ratio is one
for Qwen and Gemma and two for Mistral. Deletion gives ratios of two
for Qwen and 11.875 for Gemma. Mistral's deletion search is unresolved
through 256 at 4K, compared with an estimate of eight at 0K, so its
plotted endpoint does not represent a resolved ratio.

\begin{table*}[!htbp]
\centering\small\setlength{\tabcolsep}{5pt}
\caption{
Renormalized \texttt{qa1} effective set-size estimates on the evaluated
sizes at 0.10- and 0.20-nat mean answer-loss tolerances. Dashes indicate
background conditions that were not evaluated for this control.
}
\label{tab:renorm-qa-absolute}
\begin{tabular}{llrrrrrrrr}
\toprule &&\multicolumn{4}{c}{0.10 nat}&\multicolumn{4}{c}{0.20 nat}\\
\cmidrule(lr){3-6}\cmidrule(lr){7-10}
Model & Ranking &0K&1K&2K&4K&0K&1K&2K&4K\\\midrule
Qwen-2.5-7B & attention & 16 & -- & -- & 16 & 8 & -- & -- & 16 \\
Qwen-2.5-7B & contribution & 32 & -- & -- & 32 & 8 & -- & -- & 16 \\
Gemma-7B & attention & 32 & -- & -- & 32 & 8 & -- & -- & 32 \\
Gemma-7B & contribution & 16 & -- & -- & 32 & 8 & -- & -- & 32 \\
Mistral-7B & attention & 8 & -- & -- & 16 & 8 & -- & -- & 8 \\
Mistral-7B & contribution & 32 & -- & -- & 32 & 16 & -- & -- & 32 \\
\bottomrule\end{tabular}
\end{table*}

These comparisons support the role of aggregation in the fixed-support
experiment. Their interpretation remains relative to each condition's
full-model baseline, whose answer loss increases with background
(Appendix~\ref{app:dense-qa-background}). The tolerance also matters.
Gemma's renormalized attention estimate is unchanged at 32 under the
0.10-nat criterion, while its 0.20-nat estimate increases from eight
to 32. Two background endpoints cannot establish a general scaling law.

\begin{figure*}[!htbp]
\centering
\begin{minipage}{0.49\textwidth}\centering
\includegraphics[width=\linewidth]{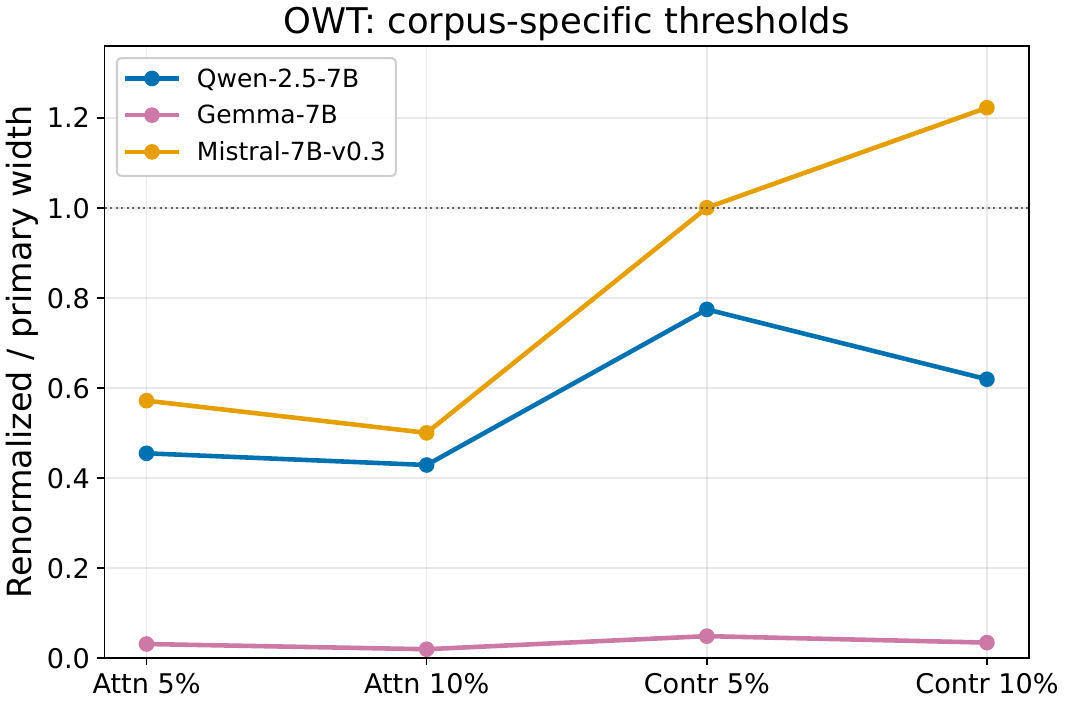}\\[-2pt]
{\small (a) OWT threshold ratio: renormalized / primary}
\end{minipage}\hfill
\begin{minipage}{0.49\textwidth}\centering
\includegraphics[width=\linewidth]{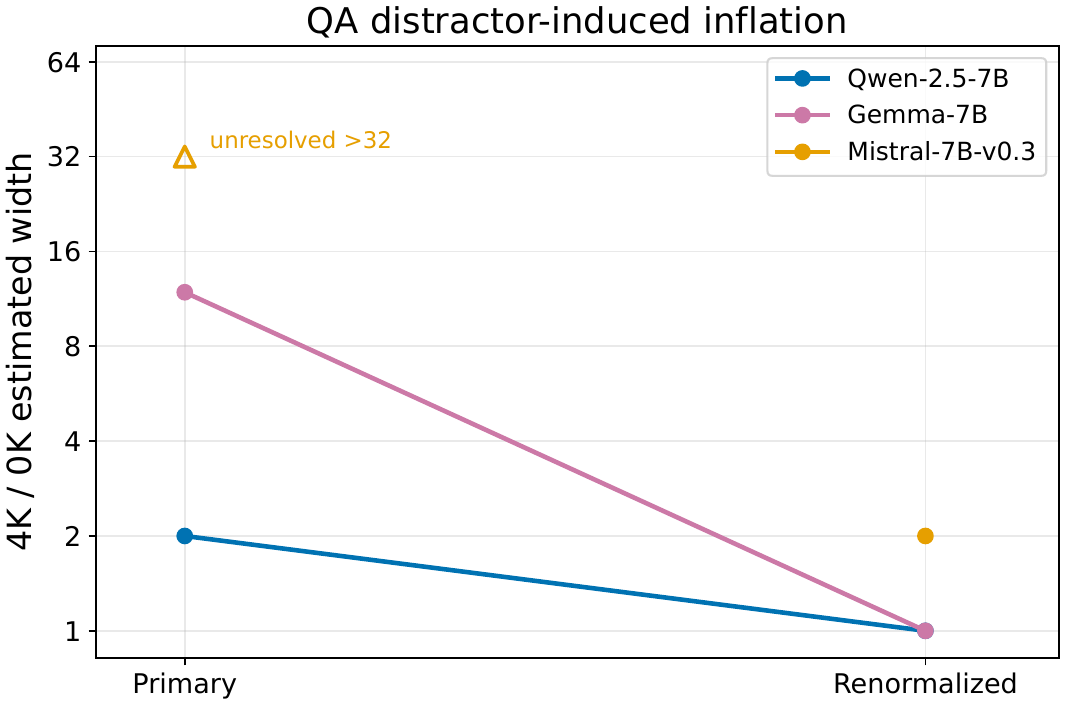}\\[-2pt]
{\small (b) 4K/0K QA width inflation}
\end{minipage}
\caption{
\textbf{Renormalization changes both required set size and background sensitivity.}
(a) Ratios of renormalized to deletion estimates on OpenWebText from
Table~\ref{tab:renorm-lm-absolute}. (b) Attention-ranked 4K/0K ratios at
a 0.10-nat answer-loss tolerance. Mistral's unresolved deletion result is
marked at the tested endpoint and is not a resolved ratio. These controls
are separate from the matched-target context-length evaluation.
}
\label{fig:e8-renorm}
\end{figure*}

\subsection{Comparison within the Gemma family}
\label{app:gemma2}

Gemma-7B requires unusually large selected sets in
Section~\ref{sec:exp-lm-width}. Gemma-2-9B provides a comparison within
the same model family. Taking the larger of the two corpus-specific
estimates gives
\[
(N^*_{5\%},N^*_{10\%})=(12,7)
\quad\text{for attention ranking},\qquad
(11,7)\quad\text{for contribution ranking}.
\]
The corresponding Gemma-7B estimates are $(226,206)$ and $(227,205)$.
The approximately 19--29-fold differences show that model-family identity
and parameter count alone do not account for the Gemma-7B requirement.
The comparison does not isolate individual architectural or training
changes. Numerical checks for the Gemma-2 implementation are reported in
Appendix~\ref{app:numerical-implementation}.

\subsection{Layer- and head-localized interventions}
\label{app:localization-extra}
\label{sec:localization-result}
\label{sec:architecture}

The common set-size limit in Section~\ref{sec:exp-lm-width} can reflect
heterogeneous sensitivity across attention operations. To examine this
variation, we apply Top-$N$ deletion to one layer or head at a time,
leaving all other attention operations unrestricted. Each layer is tested
on 50 OpenWebText windows. The reported effect is the mean of per-document
relative NLL increases, which differs from the ratio of corpus means used
for the main functional metric.

For each model, the two layers with the largest measured effects are
selected for a head scan on the first 20 windows. Every query head in
these layers is tested separately. This targeted scan measures sensitivity
within selected layers and does not provide a held-out whole-network
localization test. Figure~\ref{fig:localization-app} shows all scanned
layers and the ten largest measured Gemma head effects.

At $N=128$, Gemma layer~0 has a reported 70.9\% NLL increase and its
head~12 has a 62.1\% increase. The layer and head estimates use different
sample sizes and baseline averages, so their ratio cannot assign a share
of the layer effect to that head. Reference sizes also differ across
models. Qwen uses $N=16$ and Mistral uses $N=8$, with largest sampled
layer effects of 1.86\% and 0.33\%, respectively. The results reveal
variation within models, while comparisons across models remain conditional
on these different intervention sizes.

\begin{figure*}[!htbp]
\centering
\begin{minipage}{0.49\textwidth}\centering
\includegraphics[width=\linewidth]
{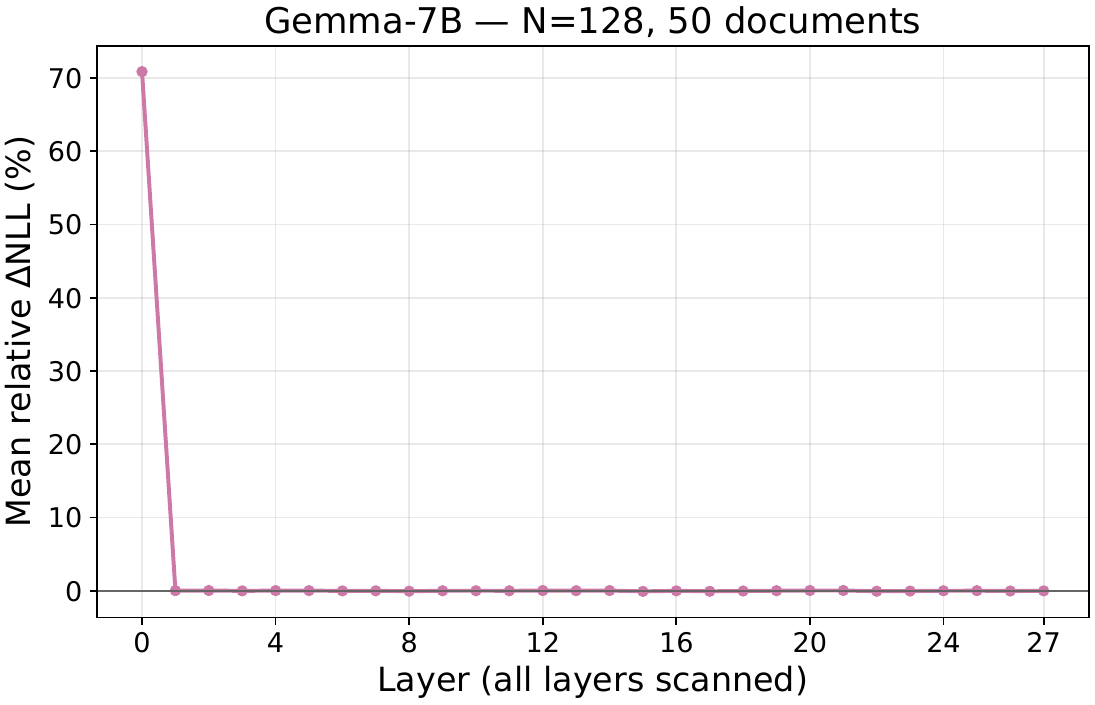}\\[-2pt]
{\small (a) Gemma-7B layer localization}
\end{minipage}\hfill
\begin{minipage}{0.49\textwidth}\centering
\includegraphics[width=\linewidth]
{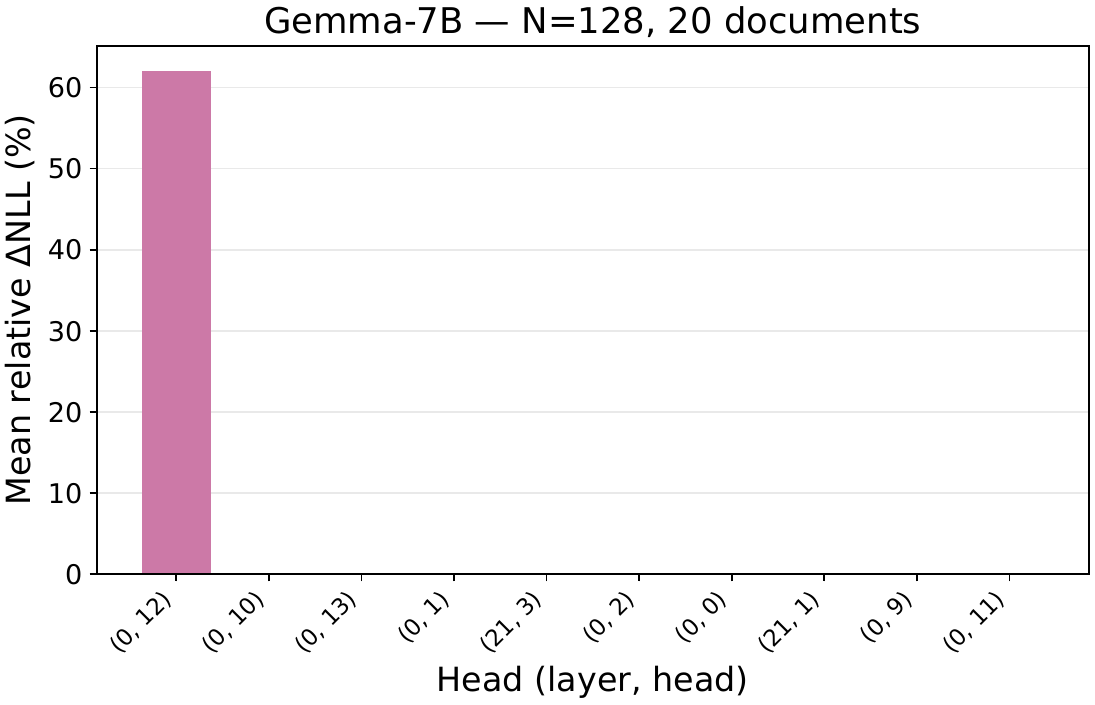}\\[-2pt]
{\small (b) Selected Gemma-7B heads}
\end{minipage}

\vspace{6pt}

\begin{minipage}{0.49\textwidth}\centering
\includegraphics[width=\linewidth]
{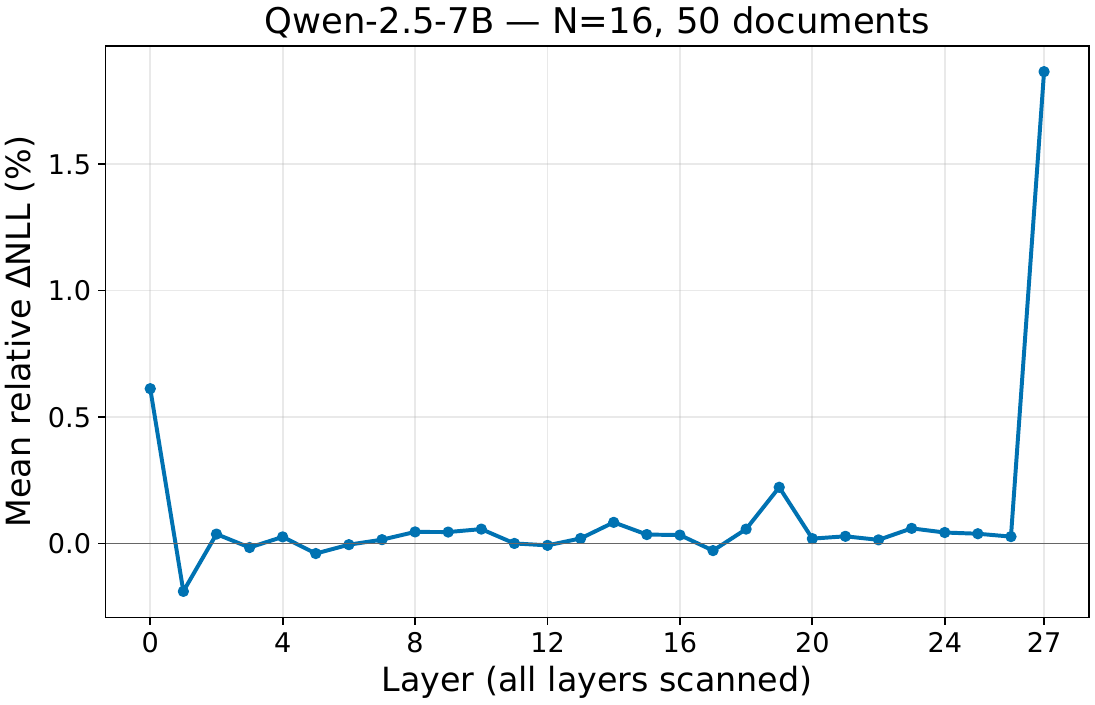}\\[-2pt]
{\small (c) Qwen-2.5-7B}
\end{minipage}\hfill
\begin{minipage}{0.49\textwidth}\centering
\includegraphics[width=\linewidth]
{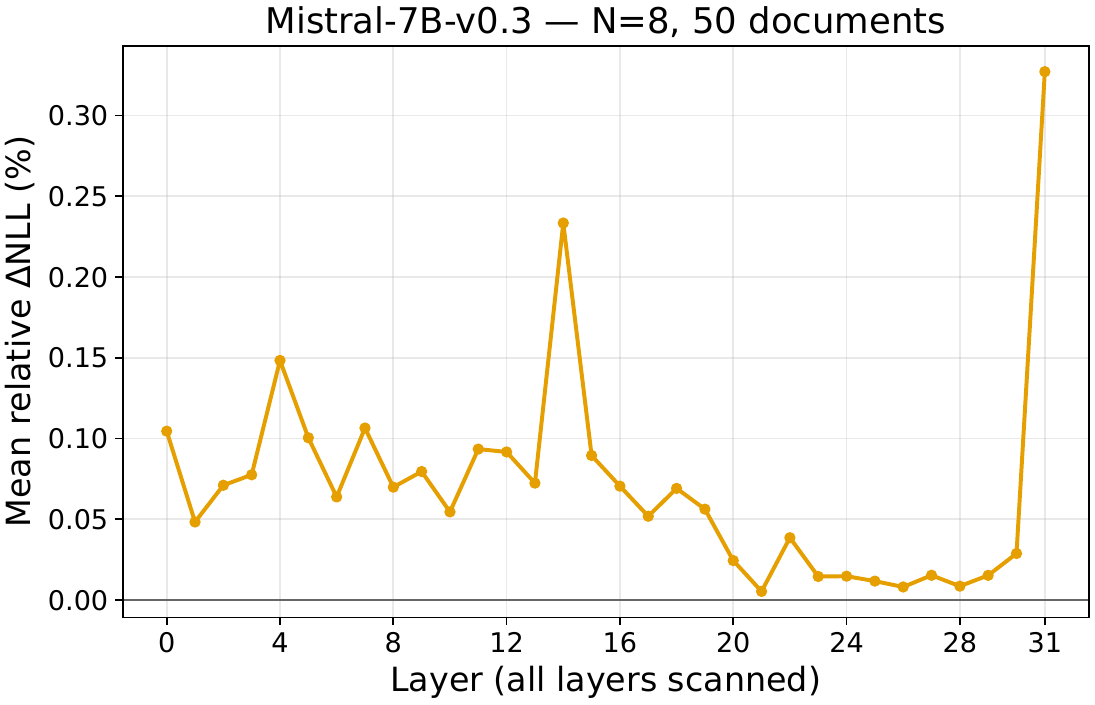}\\[-2pt]
{\small (d) Mistral-7B}
\end{minipage}

\caption{
\textbf{Sensitivity to deletion varies substantially across layers and heads.}
Reference selected-set sizes are 128 for Gemma, 16 for Qwen, and eight
for Mistral. All scanned layers and the ten largest measured Gemma head
effects are shown. Layer and head evaluations use 50 and 20 documents,
respectively. Separate interventions do not form an additive decomposition
of the full-model requirement.
}
\label{fig:localization-app}
\end{figure*}

\subsection{Exploratory prediction of effective set size}
\label{app:predictors}

The geometry--loss comparison in Section~\ref{sec:exp-geometry} motivates
testing whether diagnostics that include discarded contributions and loss
sensitivity are more informative than geometric separation alone. This
exploratory analysis uses the four geometry models, 20 OpenWebText
documents per model, $L=1024$, both rankings, selected-set sizes
$N\in\{8,16,32,64,128,256\}$, and relative NLL tolerances of 5\%
and 10\%. It addresses this fixed-length setting.

Diagnostics are measured in full-model activations across heads and query
positions. They include geometric separation and margins, discarded mass,
discarded-vector norms and energy, tail coherence from
Appendix~\ref{app:local-bounds}, and the sensitivity statistic below.
Parameters are frozen, with gradients enabled at embedding inputs so that
activation gradients can be computed without allocating parameter gradients.

\paragraph{Sensitivity to the discarded contribution.}
Let $z_{\ell ht}$ be the attention output before its output projection,
and let $e_{\ell ht}^{(r,N)}$ be the detached discarded sum. Output-head
targets are partitioned into chunks $c$ of 64 positions. Each chunk
contributes $\ell_c=(L-1)^{-1}\sum_{t\in c}\ell_t$ to mean sequence
NLL. Separate backward calls produce the implemented statistic
\begin{equation}
S_{\ell ht}^{(r,N)}=
\sum_c\left|\left\langle
\nabla_{z_{\ell ht}}\ell_c,e_{\ell ht}^{(r,N)}
\right\rangle\right|.
\label{eq:implemented-sensitivity}
\end{equation}
The absolute value is taken before summing over chunks. Consequently,
this statistic can differ from
$|\langle\nabla_z\sum_c\ell_c,e\rangle|$ and can depend on chunk size
even when total NLL is unchanged. The discarded vector is stored in FP16
and converted to the gradient dtype when the inner product is computed.
There is no additional normalization across heads.

For each layer and document, the mean and 95th percentile are computed
over the flattened head--query measurements. These statistics are then
averaged across layers and documents. Thus the reported P95 is an average
of within-layer, within-document percentiles. The same reduction is used
for the other predictors, and unavailable or nonfinite cases are excluded
with their counts reported.

\paragraph{Leave-one-model-out evaluation.}
For each ranking, tolerance, and scalar diagnostic, three models train a
one-dimensional threshold classifier for
$\mathbf1\{\Delta\NLL_{\rm rel}(N)\le\tau\}$. The search considers
both threshold directions, midpoints between sorted finite feature values,
and thresholds just outside the observed range using an offset of
$10^{-12}$. Ties are resolved lexicographically by classification error,
threshold, and direction. The predicted set size for the held-out model
is the smallest tested $N$ classified as passing.

Prediction is evaluated against the smallest passing measured size using
$|\log_2(N_{\rm pred}/N_{\rm measured})|$. If the classifier predicts
no passing size, the case is left unresolved. The raw-$N$ baseline uses
the same classifier and folds with $N$ as its feature, giving mean error
1.75 over all 16 model--ranking--tolerance cases. Both the mean and P95
sensitivity statistics give mean error 0.75 over all 16 cases, with
maximum error one. Figure~\ref{fig:e7-predictor} reports the other
diagnostics and their valid-case counts.

\begin{figure}[!htbp]
\centering
\includegraphics[width=0.92\linewidth]
{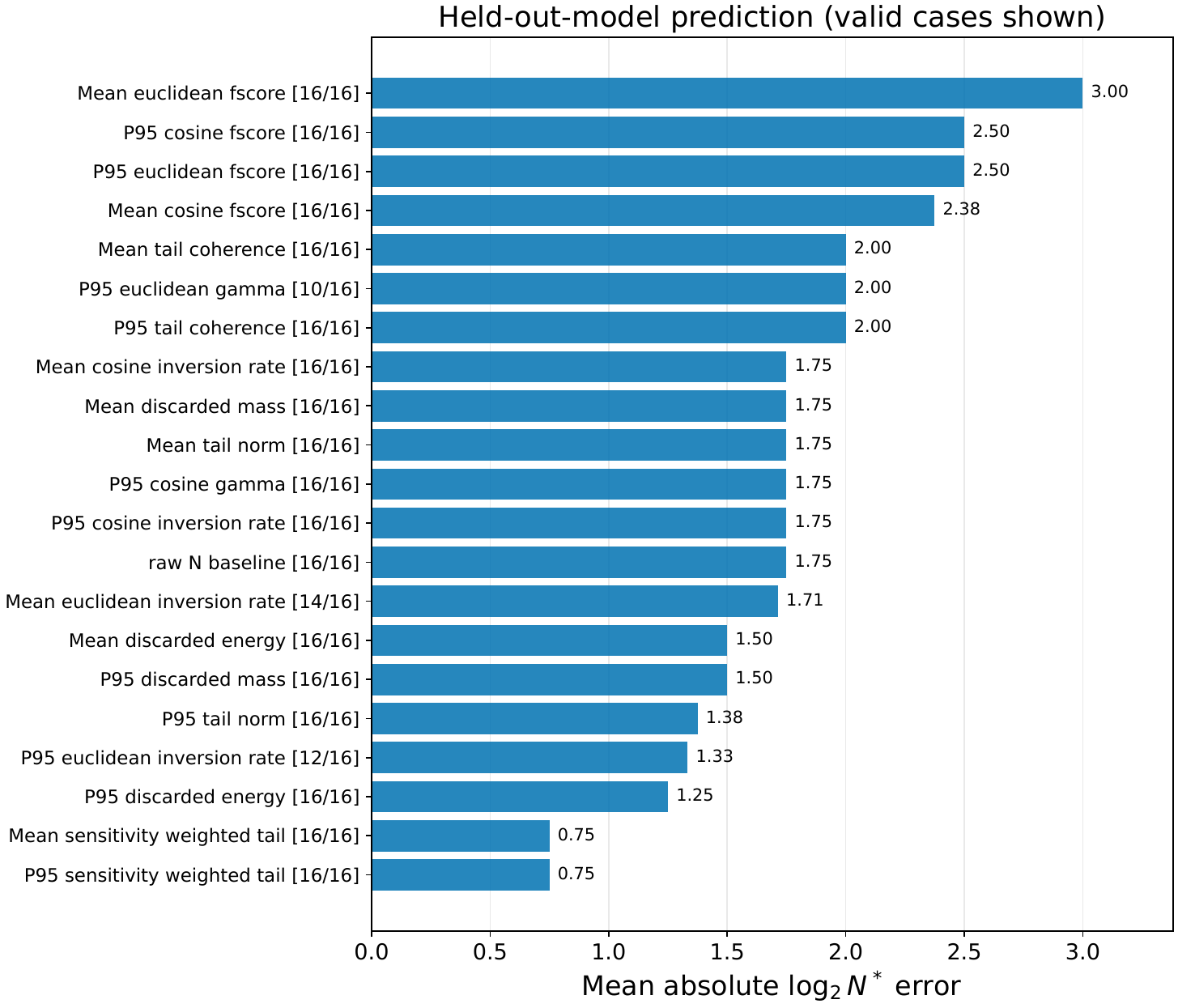}
\caption{
\textbf{Loss-sensitive diagnostics improve exploratory prediction in this sample.}
Mean absolute $\log_2$ error in effective set size under leave-one-model-out
evaluation. Lower is better. Labels report valid cases out of 16.
Diagnostics with different unresolved-case patterns are not evaluated on
identical sets of cases.
}
\label{fig:e7-predictor}
\end{figure}

The result suggests that loss sensitivity provides useful information
beyond selected-set geometry in this sample. The analysis has only four
models, 20 documents per model, and a coarse set-size grid. Selecting
the best diagnostic from these results also requires a separate evaluation
to establish predictive performance after model selection.

\subsection{Positional selection at matched set sizes}
\label{app:kv-baseline}

Section~\ref{sec:competition} attributes part of the selected-set
requirement to recovering support among competing tokens. A positional
control tests whether retaining the initial token and recent context can
satisfy the same answer-loss criterion. For a prefix longer than $N$,
the rule retains source index zero and the $N-1$ most recent visible
positions, including the current query. Shorter prefixes retain every
source. The tested sizes are $N\in\{16,32,64,128,256\}$.

The mask is applied at every query, head, and layer without renormalization
or retraining. Dense probabilities are computed before masking, as in
the main intervention. The comparison uses the same QA prompt, candidate
scoring, and sampled example IDs for Qwen-2.5-7B and Mistral-7B at
0K, 2K, and 4K backgrounds.

Qwen has no passing positional-control size through 256 at the 0.10-nat
tolerance in these conditions. Mistral passes at 32 at 0K and has no
passing size through 256 at 2K or 4K. Figure~\ref{fig:e9-kv} compares
these curves with attention-ranked selection.

\begin{figure*}[!htbp]
\centering
\begin{minipage}{0.49\textwidth}\centering
\includegraphics[width=\linewidth]
{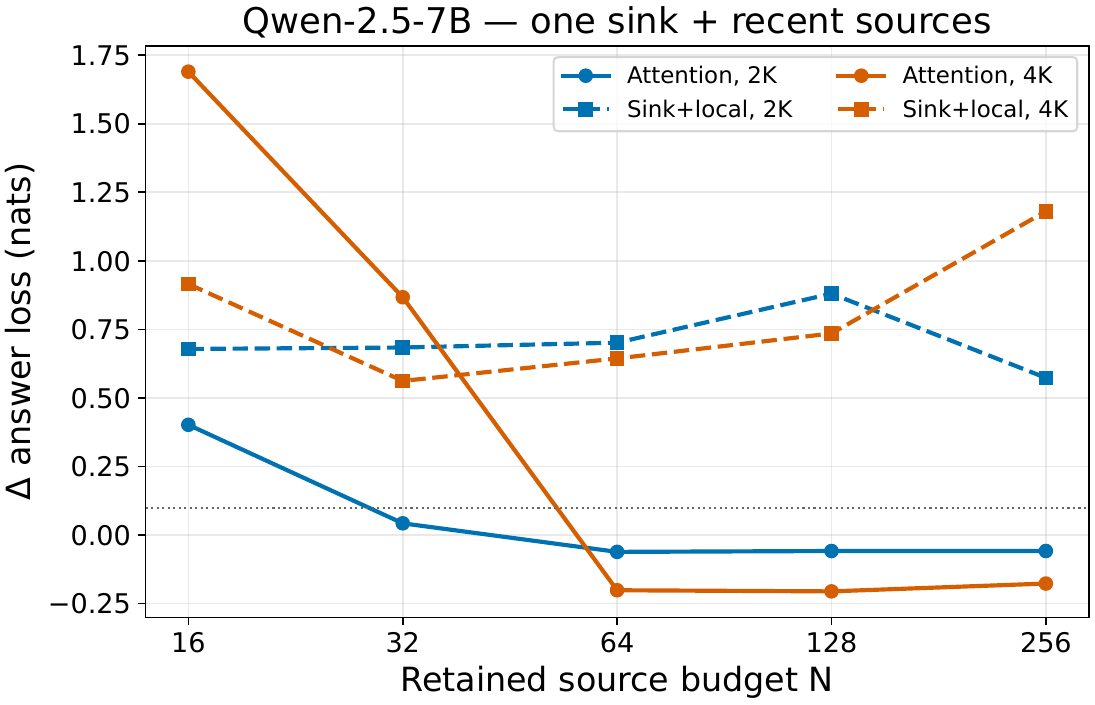}\\[-2pt]
{\small (a) Qwen-2.5-7B}
\end{minipage}\hfill
\begin{minipage}{0.49\textwidth}\centering
\includegraphics[width=\linewidth]
{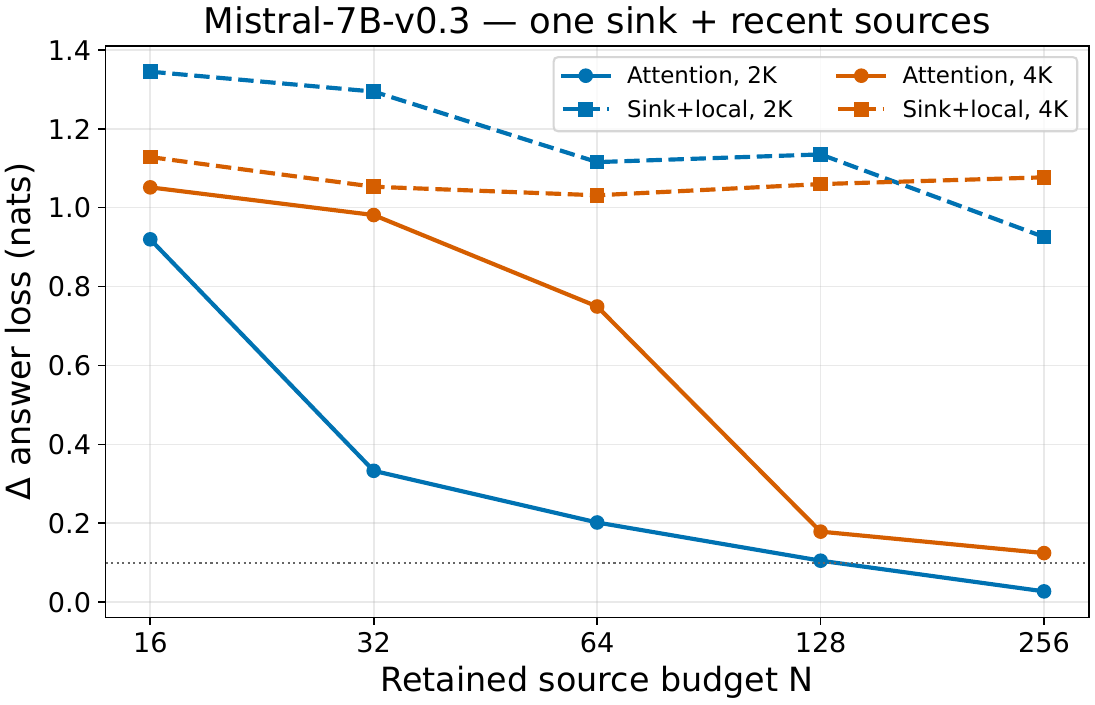}\\[-2pt]
{\small (b) Mistral-7B}
\end{minipage}
\caption{
\textbf{Positional and attention-ranked selection have different functional effects.}
The positional mask retains one initial source and $N-1$ recent sources.
Both rules use deletion without renormalization at the same set sizes.
At 2K and 4K, the positional control does not reach the 0.10-nat criterion
within the tested grid. The experiment evaluates selection after computing
dense attention.
}
\label{fig:e9-kv}
\end{figure*}

Support outside the retained initial and recent positions is excluded as
a direct source at the answer query. Earlier representations can still
carry support information, and other prompt cues can affect the answer.
The comparison therefore evaluates the functional consequences of the two
selection rules under the stated intervention. Because it computes dense
attention before masking, it does not assess the efficiency of a streaming
cache implementation.

\section{Conditional explanations of context-dependent set size}
\label{app:context-width-theory}

Sections~\ref{sec:context-length}--\ref{sec:aggregation} identify context,
competition, and aggregation as factors in the effective attention set size.
We develop conditional models that connect these observations to useful-set
recovery and preservation of the weighted sum. The models distinguish the
number of tokens needed to recover a specified set, retain attention mass,
preserve a local output, and satisfy a task-loss tolerance. Relating these
quantities requires the assumptions stated in each result. The experiments
do not establish those assumptions for individual heads or an asymptotic
law beyond the measured context lengths.

\subsection{Competition with a fixed required set}
\label{app:theory-rank-competition}

The fixed-support experiment in Section~\ref{sec:competition} motivates a
ranking model in which required information remains fixed as competitors
are added. Consider one query with $L$ candidate positions and a specified
set of $1\le K<L$ required positions. This set may represent annotated
support or the hypothesized useful set, depending on the application.
Rank positions by scalar scores, using pre-softmax logits for attention
ranking. Assume no score ties, let $s_*$ be the lowest required score,
and write $D=L-K$ for the number of competitors.

\begin{proposition}[Set size for complete recovery under competition]
\label{prop:context-recall-width}
Treat the required scores as fixed. Suppose competitor scores
$X_1,\ldots,X_D$ are independent and identically distributed, with
$p_L=\Prb(X_j>s_*)$. The smallest Top-$N$ set containing every required
position has size
\begin{equation}
N_{\rm rec}(L)=K+B_L,\qquad
B_L\sim\operatorname{Binomial}(D,p_L).
\label{eq:context-recall-binomial}
\end{equation}
Hence $\mathbb E[N_{\rm rec}(L)]=K+Dp_L$. For $0<\delta<1$, with
probability at least $1-\delta$,
\begin{equation}
N_{\rm rec}(L)\le
K+Dp_L+\sqrt{2Dp_L\log(1/\delta)}
+\frac23\log(1/\delta).
\label{eq:context-recall-bound}
\end{equation}
If $K$ is fixed and $p_L\to p\in(0,1)$, then
$N_{\rm rec}(L)/L\to p$ in probability.
\end{proposition}

\begin{proof}
The weakest required position has $K-1$ required positions and
$\sum_{j=1}^D\mathbf1\{X_j>s_*\}$ competitors above it. Its rank is
therefore $K+B_L$, the necessary and sufficient size for complete recovery.
Independence gives the binomial distribution. Bernstein's inequality and
$\operatorname{Var}(B_L)\le Dp_L$ imply
Equation~\ref{eq:context-recall-bound}. Finally,
$\operatorname{Var}(B_L/L)\le D/(4L^2)\to0$, while
$\mathbb E[B_L/L]\to p$ and $K/L\to0$.
\end{proof}

For constant exceedance probability $p$, the expected required set size is
affine in context length:
\begin{equation}
\mathbb E[N_{\rm rec}(L)]=pL+(1-p)K.
\label{eq:context-recall-affine}
\end{equation}
If $p_L\asymp L^{-\gamma}$ for $0<\gamma<1$, the expected excess
above $K$ is of order $L^{1-\gamma}$. Bounded expected set size at
fixed $K$ requires $Dp_L=O(1)$. The rank identity itself holds without
independence, although the binomial law and probability bound use the
distributional assumptions.

The probability $p_L$ concerns the weakest required token. The measured
pairwise probability in Appendix~\ref{app:support-details} averages over
support tokens and can therefore differ from $p_L$. For the useful-set
inversion count in Appendix~\ref{app:topk-recovery}, define
$\bar p_L=\Inv_r/[K(L-K)]$. Requiring
$\Inv_r\le\varepsilon K^2$ at every length with fixed $K$ would require
$\bar p_L\le\varepsilon K/(L-K)$. Hypothesis~\ref{hyp:latent-ranking}
is imposed on a reference distribution and makes no such uniform
requirement. This distinction permits the accumulation of ranking errors
as the candidate population grows. Recovering the required positions
still leaves their weights and downstream use to be assessed by the
functional intervention.

\subsection{Attention-mass retention under a stable score distribution}
\label{app:theory-mass-width}

The normalization results in Section~\ref{sec:aggregation} motivate
examining how many highly scored tokens are needed to retain a fixed
fraction of attention mass. For one local operation, write
$W_i=\exp(s_i)>0$ and $\alpha_i=W_i/\sum_{j=1}^L W_j$.
Let $S_{L,N}$ contain the $N$ largest weights, and define
\begin{equation}
M_L(N)=\sum_{i\in S_{L,N}}\alpha_i,\qquad
N_{\varepsilon}^{\rm mass}(L)
=\min\{N\in\{1,\ldots,L\}:M_L(N)\ge1-\varepsilon\},
\label{eq:context-mass-budget}
\end{equation}
where $0<\varepsilon<1$ is a discarded-mass tolerance.

\begin{proposition}[Limiting retained-mass profile]
\label{prop:context-mass-profile}
Let $W_1,W_2,\ldots$ be independent copies of a positive random variable
$W$ with a continuous distribution and $0<\mathbb E[W]<\infty$.
For $0<\rho<1$, let $q_{1-\rho}$ be its $(1-\rho)$-quantile.
Then, almost surely,
\begin{equation}
M_L(\lfloor\rho L\rfloor)\longrightarrow
m(\rho):=\frac{\mathbb E[W\mathbf1\{W\ge q_{1-\rho}\}]}
{\mathbb E[W]}.
\label{eq:context-mass-profile}
\end{equation}
The extension $m(0)=0$, $m(1)=1$ is continuous and strictly increasing.
If $m(\rho_\varepsilon)=1-\varepsilon$, then
\begin{equation}
\frac{N_{\varepsilon}^{\rm mass}(L)}{L}
\longrightarrow\rho_\varepsilon\in(0,1)
\quad\text{almost surely}.
\label{eq:context-linear-mass}
\end{equation}
\end{proposition}

\begin{proof}
The strong law gives $L^{-1}\sum_iW_i\to\mathbb E[W]$.
For any fixed $T$, the bounded empirical quantile functions of
$\min(W_i,T)$ converge almost everywhere to their population quantile
function. Bounded convergence then gives convergence of the upper-$\rho$
trimmed sum divided by $L$. The difference from the unbounded trimmed
sum is at most $L^{-1}\sum_i(W_i-T)_+$, whose almost-sure limit is
$\mathbb E[(W-T)_+]$. This quantity tends to zero as $T\to\infty$,
which proves the numerator limit in Equation~\ref{eq:context-mass-profile}.

The limiting numerator can also be written as
$\int_{1-\rho}^1F_W^{-1}(u)\,du$. Finite expectation implies
continuity, and positivity implies strict increase. For sufficiently small
$h>0$, the retained-mass limits at $\rho_\varepsilon-h$ and
$\rho_\varepsilon+h$ lie below and above $1-\varepsilon$, respectively.
Monotonicity of $M_L(N)$ brackets the normalized required size between
these fractions. Taking $h\downarrow0$ proves
Equation~\ref{eq:context-linear-mass}.
\end{proof}

The result permits strongly nonuniform attention while requiring a stable
distribution of unnormalized weights. Scores that sharpen with $L$,
a fixed set of positions retaining nonvanishing mass, or a changing
mixture of local and distant sources can violate this assumption.
Extending the result to dependent scores would require convergence of
their empirical distribution and control of the weight tails.

\paragraph{Gaussian-logit example.}
For $s_i\sim\mathcal N(\mu_s,\sigma^2)$ with fixed $\sigma>0$,
completing the square in the Gaussian density gives
\begin{equation}
m(\rho)=\Phi\!\left(\sigma+\Phi^{-1}(\rho)\right),\qquad
\rho_\varepsilon=
\Phi\!\left(\Phi^{-1}(1-\varepsilon)-\sigma\right),
\label{eq:context-lognormal-mass}
\end{equation}
where $\Phi$ is the standard normal distribution function. The common
mean $\mu_s$ cancels under softmax. This example shows how score spread
can determine the retained fraction under a stable distribution. It is
an illustrative calculation and is not fitted to the empirical logits
or loss-based set sizes.

\subsection{From retained mass to local output accuracy}
\label{app:theory-coherent-tail}

Mass retention helps explain a functional requirement only when it also
constrains the attention output. The upper bound in
Appendix~\ref{app:local-bounds} supplies a sufficient condition when
values are bounded. A necessary condition additionally requires control
of cancellation among discarded values. At fixed incoming activations,
write $z_L=\sum_i\alpha_i v_i$, $z_{L,N}=\sum_{i\in S_{L,N}}\alpha_i v_i$,
and $e_{L,N}=z_L-z_{L,N}$.

\begin{proposition}[Local set size under coherent values]
\label{prop:context-coherent-width}
Suppose constants $0<a\le B<\infty$, independent of $L$, and a unit
vector $u$ satisfy $\|v_i\|_2\le B$ and $\langle u,v_i\rangle\ge a$
for all candidate positions. Then
\begin{equation}
a[1-M_L(N)]\le\|e_{L,N}\|_2\le B[1-M_L(N)].
\label{eq:context-two-sided-error}
\end{equation}
For $0<\xi<a$, define
$N_{\xi}^{\rm del}(L)=\min\{N\in\{1,\ldots,L\}:\|e_{L,N}\|_2\le\xi\}$.
Then
\begin{equation}
N_{\xi/a}^{\rm mass}(L)\le N_{\xi}^{\rm del}(L)
\le N_{\xi/B}^{\rm mass}(L).
\label{eq:context-width-sandwich}
\end{equation}
Under the assumptions of Proposition~\ref{prop:context-mass-profile},
almost surely,
\begin{equation}
0<\rho_{\xi/a}
\le\liminf_{L\to\infty}\frac{N_{\xi}^{\rm del}(L)}L
\le\limsup_{L\to\infty}\frac{N_{\xi}^{\rm del}(L)}L
\le\rho_{\xi/B}<1.
\label{eq:context-linear-order}
\end{equation}
\end{proposition}

\begin{proof}
Projection onto $u$ yields
$\|e_{L,N}\|_2\ge\langle u,e_{L,N}\rangle
\ge a\sum_{i\notin S_{L,N}}\alpha_i$, and the triangle inequality gives
the upper bound. Error at most $\xi$ requires discarded mass at most
$\xi/a$. Discarded mass at most $\xi/B$ is sufficient. These implications
give Equation~\ref{eq:context-width-sandwich}, even when the error norm
is nonmonotone in $N$. Applying the retained-mass limits proves
Equation~\ref{eq:context-linear-order}.
\end{proof}

The proposition establishes linear-order growth under a uniform coherence
condition. It does not specify a unique slope. The upper norm bound alone
gives no necessary set size because cancellation can make the discarded
sum small. The argument also applies to output-projected values
$W_O^{(h)}v_i$ when its assumptions hold in residual-stream space.
Connecting this local result to downstream loss requires an additional
sensitivity assumption.

\subsection{Renormalization and an explicit task-loss model}
\label{app:theory-renormalization}

Section~\ref{sec:aggregation} shows that renormalization can substantially
change the measured set size. Equation~\ref{eq:renorm-centroids} explains
the local effect through retained and discarded weighted means. For
$1\le N<L$, write $M=M_L(N)$, $\mu_S=z_{L,N}/M$, and
$\mu_T=(z_L-z_{L,N})/(1-M)$. Then
\begin{equation}
z_L-z_{L,N}=(1-M)\mu_T,\qquad
z_L-\frac{z_{L,N}}M=(1-M)(\mu_T-\mu_S).
\label{eq:context-renorm-errors}
\end{equation}
If $\mu_T\ne0$, the ratio of renormalization error to deletion error
is exactly $\|\mu_T-\mu_S\|_2/\|\mu_T\|_2$. A component shared
by both means contributes to deletion error and cancels from their
difference. The following example connects this possibility directly to
a loss tolerance.

\begin{proposition}[Growing loss-based set size with identical values]
\label{prop:context-toy-loss}
Let every value equal a fixed nonzero vector $\mu$. A binary readout
assigns the correct-class logit
$\beta\langle\mu,z\rangle/\|\mu\|_2^2$ and the other-class logit
zero, where $\beta>0$. Its loss at $z=M\mu$ is
$\ell(M)=\log(1+\exp(-\beta M))$. For
$0<\tau<\log2-\ell(1)$, define
\begin{equation}
c_\tau=-\frac1\beta
\log\!\left(\exp(\ell(1)+\tau)-1\right)\in(0,1).
\label{eq:context-toy-mass-threshold}
\end{equation}
The smallest deletion set satisfying
$\ell(M_L(N))-\ell(1)\le\tau$ has size
$N_{1-c_\tau}^{\rm mass}(L)$. Under
Proposition~\ref{prop:context-mass-profile}, its ratio to $L$ converges
to the positive solution of $m(\rho)=c_\tau$. Renormalized retention
has zero loss change for every $N\ge1$.
\end{proposition}

\begin{proof}
The full output is $\mu$, deletion gives $M_L(N)\mu$, and renormalization
restores $\mu$. Solving
$\log(1+e^{-\beta M})\le\ell(1)+\tau$ gives $M\ge c_\tau$.
The retained-mass result then establishes the limit.
\end{proof}

For uniform weights, a direct calculation gives
$N_{\rm del}(L)=\lceil c_\tau L\rceil$ and $N_{\rm ren}(L)=1$.
Growing loss-based set size can therefore arise from preserving output
amplitude even when every value carries the same representation. This
construction demonstrates why an increasing deletion requirement alone
cannot identify growth in distinct useful information.

For a general downstream loss that is locally $C$-Lipschitz,
$|\ell(z)-\ell(\widetilde z)|\le C\|z-\widetilde z\|_2$ provides
a sufficient local guarantee. The full interventions also change earlier
activations, later scores, and selected sets. The local comparison in
Equation~\ref{eq:context-renorm-errors} holds for fixed activations and
the same subset, so extending it to the full-model quantity
$N_\tau^{*,(r)}$ requires additional assumptions.

\subsection{Dependence on the scored query positions}
\label{app:theory-target-positions}

Appendix~\ref{app:context-length-results} reports smaller effective set
sizes under all-token scoring than under matched final-target scoring.
To isolate one reason for this difference, index queries by visible-source
count $t\in\{1,\ldots,L\}$. The one-token indexing difference from
causal language-model scoring does not affect the following limits.
Assume that expected absolute loss degradation after the full intervention
has the form
\begin{equation}
d_t(N)=g(N/t),
\label{eq:context-query-profile}
\end{equation}
where $g:[0,\infty)\to[0,\infty)$ is bounded, continuous,
nonincreasing, and zero on $[1,\infty)$. This is an assumed common
degradation profile across query positions. It does not follow from the
local bounds above or require independent interventions across queries.

\begin{proposition}[Effect of scoring positions]
\label{prop:context-query-averaging}
For $N=\lfloor\rho L\rfloor$, $0<\rho<1$, and integers
$1\le w_L\le L$, define
\[
D_{\rm all}(N,L)=\frac1L\sum_{t=1}^L g(N/t),\qquad
D_{\rm tail}(N,L)=\frac1{w_L}
\sum_{t=L-w_L+1}^L g(N/t).
\]
If $w_L/L\to0$, then
\begin{equation}
D_{\rm tail}(N,L)\to g(\rho),\qquad
D_{\rm all}(N,L)\to A(\rho):=
\int_0^1g(\rho/u)\,du\le g(\rho).
\label{eq:context-query-limits}
\end{equation}
For an absolute tolerance $\tau$ with unique interior crossings of
$g(\rho)=\tau$ and $A(\rho)=\tau$, the smallest passing set sizes
divided by $L$ converge to the corresponding solutions
$\rho_{\rm tail}$ and $\rho_{\rm all}$. These satisfy
$\rho_{\rm all}\le\rho_{\rm tail}$.
\end{proposition}

\begin{proof}
Over the final $w_L=o(L)$ queries, $N/t\to\rho$ uniformly, so
continuity gives the tail limit. The all-query average is a Riemann sum.
Its limiting integrand is zero for $u\le\rho$ and has a continuous
extension at zero. Since $\rho/u\ge\rho$ for $0<u\le1$,
monotonicity gives $g(\rho/u)\le g(\rho)$. Convergence on either
side of each unique crossing and monotonicity in $N$ give the set-size
limits and their ordering.
\end{proof}

If $w_L/L\to\kappa\in(0,1]$, the tail limit becomes
$\kappa^{-1}\int_{1-\kappa}^1g(\rho/u)\,du$. Although 128 targets
form a fixed-size tail asymptotically, they occupy approximately half
of a 256-position window. Finite-query averages are therefore more
appropriate for the shortest tested contexts. Relative NLL also uses
protocol- and length-specific full-model baselines, which are outside
the absolute-loss comparison above. Negative or nonmonotone empirical
degradation can violate the assumed profile.

\subsection{Implications for the empirical findings}
\label{app:theory-empirical-scope}

The conditional models organize the evidence in the main text without
identifying a single mechanism. In natural-context extension, the same
final targets require larger selected sets, including under a fixed
absolute tolerance. This limits explanations based only on changing
target populations or tighter relative thresholds. Improved full-model
NLL leaves additional useful information as a plausible contributor,
alongside competition and preservation of the weighted sum.

The controlled-background experiment provides more direct evidence of
competition. Annotated support remains fixed while displacement increases,
mass decreases, and recall at a fixed set size falls. Relating these
measurements to Proposition~\ref{prop:context-recall-width} would require
the distribution of the weakest support rank in addition to mean
displacement. The annotated set would also need to be distinguished
from the full useful set for the prediction.

The normalization controls establish sensitivity to how selected values
are combined. Testing the mass-retention explanation more directly would
require comparing $M_L(\lfloor\rho L\rfloor)$ across lengths and
measuring discarded-vector norms and retained/discarded means, including
after output projection. Applying deletion and renormalization to the
same matched-target context-length sweep would further clarify this
relationship, while accounting for their effects on later rankings.
The existing normalization evaluations do not supply that matched sweep.

Finally, the scoring-position model explains how averaging earlier,
shorter-prefix queries can reduce an estimated set size under explicit
assumptions. It supports treating the target population as part of the
measurement definition. Together, these results motivate interpreting
effective attention set size through competition, aggregation, and the
evaluation criterion, with the latent useful-set size remaining unobserved.

\clearpage
\section*{Ethics Statement}
This work analyzes publicly available pretrained language models using
established text and question-answering benchmarks. It does not collect
new personal data, involve human participants, or infer properties of
individual users. The geometric and effective-set-size measurements
characterize model behavior under specified diagnostics and interventions.
Applications to attention sparsification or cache reduction should evaluate
prediction degradation on the intended downstream distribution.

\section*{Reproducibility Statement}
Appendix~\ref{app:protocol} specifies sampling, tokenization, candidate
scoring, support mapping, interventions, aggregation, threshold searches,
and numerical checks. The supplementary tables report corpus-specific
estimates and distinguish incomplete searches from unmeasured conditions.
For matched-target context extension, the protocol specifies paired
suffixes and 5000 paired-document bootstrap replicates, with uncertainty
conditional on the evaluated set sizes. The scope of numerical parity
checks, unverified historical pairing of the normalization evaluations,
and the 24B precision configuration are stated explicitly. The mathematical
results include their assumptions and proofs. The accompanying source
archive contains the manuscript and figure assets. Raw per-document
measurements, experiment code, and environment manifests are needed to
independently reproduce the empirical summaries.

\section*{AI Use Statement}
Language-model tools were used for language editing, restructuring, and
debugging. The authors are responsible for the mathematical statements,
experimental design, data analysis, citations, and final text.

\end{document}